\providecommand{\mainonlyflag}{0}
\providecommand{\supplementonlyflag}{0}
\providecommand{\buildvariant}{FULL}
\ifnum\supplementonlyflag=1\relax
  \documentclass[final,supplement]{siamart251216}
\else
  \documentclass[final]{siamart251216}
\fi

\usepackage{booktabs}
\usepackage{mathtools}
\usepackage{amssymb,amsfonts}
\usepackage{enumitem}
\usepackage{ulem}
\usepackage{caption}
\usepackage{subcaption}
\usepackage{algpseudocode}
\usepackage{pifont}
\usepackage{etoolbox}
\usepackage{float}

\newsiamremark{remark}{Remark}
\newsiamthm{assumption}{Assumption}

\DeclarePairedDelimiterX\Basics[1](){ #1}
\newcommand{\innermid}{\nonscript\;\delimsize\vert\nonscript\;}
\newcommand{\activatebar}{%
  \begingroup\lccode`\~=`\|
  \lowercase{\endgroup\let~}\innermid 
  \mathcode`|=\string"8000
}

\makeatletter
\hypersetup{pdfsubject={CODEX-BUILD-ID: SPLIT-R3-20260819-\buildvariant}}
\let\siamorigprinteqnum\print@eqnum
\def\print@eqnum{\kern-28pt\siamorigprinteqnum\kern28pt}
\let\siamorigateqnnum\@eqnnum
\def\@eqnnum{\kern-28pt\siamorigateqnnum\kern28pt}
\DeclareMathOperator*{\argmin}{\arg\min}

\newlist{steps}{enumerate}{1}
\setlist[steps, 1]{label = Step \arabic*:}

\newcommand{\interior}[1]{%
  {\kern0pt#1}^{\mathrm{o}}%
}

\newcommand{\NN}{\mathcal{N}}

\newcommand{\originalreferencekeys}{almgren2001optimal,anugu2025ergodic,arapostathis2016ergodic,auer2008near,azar2017minimax,bayraktar2023approximate,bayraktar2026reinforcement,bercu2008exponential,bertsekas1996neuro,Bhatia2017OnTB,black1992global,blommestein2012interactions,bubeck2011x,carr2001pricing,cartea2015algorithmic,dai2025data,dann2017unifying,dayan1992q,domingues2021kernel,dong2019provably,du2020sovereign,duffie1997hedging,fan2020theoretical,fazel2018global,fu2020single,Givens1984ACO,guo2023fast,hambly2021policy,hambly2023recent,han2023choquet,he2015dynamic,Hsu2011ATI,huang2025sublinear,jaakkola1993convergence,jia2022policy,jia2023q,jin2020provably,kakade2003sample,kar2024adaptive,kara2023qlearning,kara2023q,kiran2021deep,kleinberg2008multi,kleinberg2019bandits,kloeden1992numerical,kober2013reinforcement,kushner2001numerical,lazaric2012finite,li2024discrete,mcCallum1996utree,munos2002variable,munos2008finite,ortner2014regret,pazis2013pac,pradhan2025discrete,puterman2014markov,SCHMITT1992215,shalev2016safe,sinclair2023adaptive,slivkins2011contextual,strehl2006pac,tan2014discrete,tektas2005asset,tsitsiklis1996analysis,vershynin2018high,wainwright2019high,wang2020reinforcement,wang2019neural,winkelbauer2012moments,yong1999stochastic,yu2021reinforcement,zhang2025state,zhao2020sim,zhou2000continuous}

\newcommand{\bCh}{\overline{C}_h} 
\newcommand{\wtCh}{\widetilde{C}_h} 
\newcommand{\hCh}{\widehat{C}_h} 

\mathchardef\mhyphen="2D

\newlist{subassumption}{enumerate}{1}
\setlist[subassumption,1]{label=\mbox{(\alph*)},ref=\theassumption\mbox{(\alph*)}}

\ifnum\mainonlyflag=1\relax
\fi
\ifnum\supplementonlyflag=1\relax
\fi

\begin{document}

\title{Adaptive Partitioning and Learning for Stochastic Control of Diffusion Processes}
\author{Hanqing Jin\thanks{Mathematical Institute, University of Oxford
(\email{jinh@maths.ox.ac.uk}, \email{yanzhao.yang@merton.ox.ac.uk}).}
\and Renyuan Xu\thanks{Department of Management Science and Engineering,
Stanford University (\email{renyuanxu@stanford.edu}). R.X. is supported in
part by the NSF CAREER Award DMS-2524465 and a gift fund from Point72.}
\and Yanzhao Yang\footnotemark[1]}


\date{\today}
\maketitle
\headers{Adaptive Partitioning for Diffusions}{H. Jin, R. Xu, and Y. Yang}

\ifnum\supplementonlyflag=0\relax
\begin{abstract}
    We study reinforcement learning for controlled diffusion processes with unbounded continuous state spaces, bounded continuous actions, and polynomially growing rewards—settings that arise naturally in finance, economics, and operations research. To overcome the challenges of continuous and high-dimensional domains, we introduce a model-based algorithm that adaptively partitions the joint state–action space. The algorithm maintains estimators of drift, volatility, and rewards within each partition, refining the discretization whenever estimation bias exceeds statistical confidence. This adaptive scheme balances exploration and approximation, enabling efficient learning in unbounded domains. Our analysis establishes regret bounds that depend on the problem horizon, state dimension, reward growth order, and a newly defined notion of zooming dimension tailored to unbounded diffusion processes. The bounds recover existing results for bounded settings as a special case, while extending theoretical guarantees to a broader class of diffusion-type problems.
\end{abstract}



\section{Introduction}

Data-driven decision-making has emerged as a foundational paradigm in modern scientific and engineering disciplines, enabling systems to adapt and optimize behavior in complex, uncertain environments by learning from empirical evidence. In particular, reinforcement learning (RL)  formalizes {\it sequential} decision-making under uncertainty as a mathematical framework involving agents interacting with unknown environments to maximize long-term cumulative reward. Applications range from robotics \cite{kober2013reinforcement,zhao2020sim} and autonomous systems \cite{kiran2021deep,shalev2016safe} to finance \cite{hambly2023recent} and healthcare \cite{yu2021reinforcement}, especially in settings where traditional model-based methods may fail due to restrictive structural assumptions or limited flexibility.

The literature on RL theory has progressed through a structured hierarchy of assumptions on state–action spaces, beginning with finite (tabular) settings and extending toward infinite states/actions or continuous domains. Earlier seminal works focused on tabular MDPs with finite state-action spaces, where convergence and sample efficiency of model-free algorithms, such as Q-learning, are studied under exact representations \cite{auer2008near,dayan1992q,jaakkola1993convergence,kakade2003sample}. These settings allow for strong performance guarantees using regret and PAC frameworks \cite{azar2017minimax,dann2017unifying}. As attention shifted to large or continuous state spaces,  linear function approximation has been introduced, preserving tractability while enabling generalization \cite{tsitsiklis1996analysis,bertsekas1996neuro,lazaric2012finite}. These frameworks often retain finite action spaces and require bounded features or realizability assumptions. More recent work explores continuous or unbounded state spaces using either nonparametric techniques (e.g., nearest-neighbor methods \cite{jin2020provably}) or neural network approximations \cite{fan2020theoretical,fu2020single,wang2019neural}, though theoretical guarantees remain limited in the latter (in terms of the choice of network architectures). Finite action spaces remain the standard setting in theoretical RL studies, largely due to the combinatorial challenges posed by continuous action spaces, namely the interrelated difficulties in optimization, exploration, and representation. Only a few exceptions exist, such as studies focusing on problems with special structure (e.g., linear-quadratic regulators \cite{fazel2018global,hambly2021policy,guo2023fast}) or those exploring discretization-based nonparametric methods, which include both uniform partitioning \cite{bayraktar2023approximate,kara2023q} and adaptive partitioning approaches \cite{dong2019provably,pazis2013pac, sinclair2023adaptive}.  This progression of the theoretical RL literature reflects a {\it trade-off} between tractability and expressive power: tabular and linearly parameterized settings are more tractable for analysis, whereas generic continuous state–action spaces, though more general and practically important, remain less understood and less theoretically developed due to their complexity.

Many critical decision-making problems in finance, economics, and operations research involve unbounded, continuous state spaces, as well as continuous (often high-dimensional) action spaces, and unbounded reward functions. A central class of such problems arises in portfolio optimization, where agents take continuous actions by dynamically adjusting their wealth allocations across risky assets in response to evolving market conditions. These problems typically involve an unbounded and continuous state space, representing asset prices and wealth levels, and often feature unbounded reward (utility functions) subject to suitable growth conditions  \cite{black1992global,zhou2000continuous,he2015dynamic}. Optimal execution and intraday trading problems are often formulated within a continuous state-action framework, as traders must balance market impact, adverse selection risk, and order flow dynamics in a tractable manner \cite{almgren2001optimal,cartea2015algorithmic}. In dynamic hedging, particularly in incomplete markets or under stochastic volatility, agents must continuously adjust their positions to manage risk exposure \cite{carr2001pricing,duffie1997hedging}. Credit risk and asset-liability management problems faced by banks, insurers, and pension funds also fall within this framework, as they involve dynamic decision-making under uncertainty, often with evolving and potentially unbounded risk profiles \cite{tektas2005asset}. At a broader scale, macro-financial decisions (such as sovereign debt issuance and monetary policy under uncertainty) rely on models with unbounded, continuous spaces to capture long-term dynamics and structural risks \cite{blommestein2012interactions,du2020sovereign}. Beyond finance and economics, structurally identical problems arise in operations research more broadly: heavy-traffic control of queueing and service networks yields controlled diffusions on $\mathbb{R}^{d}$ with scheduling or routing actions \cite{arapostathis2016ergodic,anugu2025ergodic}; large-scale epidemic, inventory, and energy management problems admit similar diffusion formulations with unbounded states and polynomially growing costs.   Despite their importance, such settings remain less understood in the RL literature, particularly regarding algorithmic development and theoretical guarantees.

Motivated by these challenges, this work seeks to address the following open question:
\begin{itemize}
\item[$\quad$] {\it Can we design an adaptive partition scheme tailored to (unknown) high-dimensional diffusion processes and simultaneously learn the optimal policy efficiently within the RL framework?}
\end{itemize}

\subsection{Our work and contributions}
We investigate the above-mentioned open question in a setting governed by diffusion-type dynamics over a finite time horizon, with an unbounded state space and a continuous action space. 
The motivating control model is a controlled diffusion in continuous time. However, a sample-based RL algorithm interacts with the environment through observations, actions, and rewards collected at discrete decision times. Therefore, the implementable RL problem must be formulated on a time-discretized system. In this paper, we formulate and analyze this implementable problem as the discrete-time MDP induced by an Euler--Maruyama discretization of the underlying controlled stochastic differential equation.
To facilitate learning, we consider a discrete-time Markov decision process (MDP) with Gaussian increments, serving as an approximation of continuous-time diffusion processes. 
More explicitly, the one-step transition takes the form
\[
X_{h+1}-X_h=\mu_h(X_h,A_h)\Delta+\sigma_h(X_h,A_h)B_h\sqrt{\Delta},
\]
so the transition kernel is Gaussian only conditionally on the current state--action pair, with both its mean and covariance depending on $(X_h,A_h)$. Thus, the model is not a fixed Gaussian-increment MDP; rather, it is the discrete-time RL formulation of a controlled diffusion with unknown, state--action-dependent drift and volatility.
Crucially, we allow the expected reward to exhibit polynomial growth, going beyond the standard bounded reward assumptions, which enables our framework to capture a broader class of real-world applications.

To address the challenges of unbounded state space, we localize the state space by restricting the learning to a bounded ball, whose radius is carefully chosen to control the ultimate regret. The learning algorithm operates in an episodic setting. The framework is end-to-end in the following concrete sense: the algorithm does not first estimate a continuous-time diffusion model and then solve a separate control problem offline. Instead, from the same observed samples $(X_h^k,A_h^k,r_h^k,X_{h+1}^k)$, it updates the adaptive state--action partition, the local reward, drift, and volatility estimators, the optimistic $Q$- and value-function estimates, and consequently the policy used in future episodes. Throughout the learning process, we maintain representative estimators of both the drift and volatility within each partition of the joint state-action space. These partitions are refined adaptively: when the estimated bias exceeds the statistical confidence of the representative estimators, the partition is subdivided. Using the estimated drift and volatility, we construct a Q-function and select actions based on the upper confidence bound of this function. Mathematically, we show that the proposed algorithm achieves a regret of order $$\tilde{\mathcal{O}}({H}K^{1-\frac{p^{2}-(m+1)^{2}(z_{\max,c}+2)-(m+1)(2d_{\mathcal{S}}+2m+4)}{p(p+m+1)(z_{\max,c}+2)+p(2d_{\mathcal{S}}+2m+4)} }),$$ with $H$ the horizon of the problem, $K$ the number of episodes, $p$ the highest bounded moments for the initial state distribution, $m+1$ the order of reward polynomial growth, $d_{\mathcal{S}}$ the dimension of state space and $z_{\max,c}$ the worst-case zooming dimension over the entire horizon. {Here, the zooming dimension quantifies problem benignness, with $z_{\max,c}$ often much smaller than the joint state-action space dimension $d_{\mathcal{A}}+d_{\mathcal{S}}$ for benign instances \cite{kleinberg2019bandits}.}
The idea of adaptive partitioning is largely inspired by \cite{sinclair2023adaptive}, which considers a markedly different setting—namely, an MDP with a bounded state space and bounded rewards. Our contribution is therefore not merely to apply their adaptive partitioning idea to another MDP: the unbounded diffusion-type setting requires localization of the state process, estimation of state--action-dependent drift and volatility rather than a generic transition kernel, covariance concentration under Lipschitz volatility, polynomial-growth value estimates, and concentration arguments for unbounded martingale terms. Nevertheless, as $p$ tends to infinity, our regret  order asymptotically approaches $\tilde{\mathcal{O}}({H}K^{\frac{z_{\max,c}+1}{z_{\max,c}+2}})$, which is consistent with the order established in \cite{sinclair2023adaptive} in terms of the episodes number $K$, despite the substantial differences in both the problem setting and the underlying technical analysis.

From a technical perspective, a key challenge lies in defining an appropriate notion of zooming dimension which affects the algorithm design and hyperparameter set-up. Unlike the classical zooming dimension defined for bounded state-action spaces, our setting requires a new formulation suited to unbounded state spaces, one that can be meaningfully linked to the regret analysis (see Definition~\ref{def:zooming dimension} and the proof of Lemma~\ref{lemma:theorem F.3 conclusion}). Furthermore, as we aim to analyze regret in diffusion-type settings, our approach differs from that of \cite{sinclair2023adaptive}, which characterizes the concentration of Markov transition kernels. Instead, we fully leverage the structure of the dynamics and derive concentration inequalities for the drift and volatility terms (see the proof of Theorem~\ref{thm:transition kernel wasserstein local lipschitz all together}). In particular, deriving concentration inequalities for covariance matrices under only Lipschitz regularity of the volatility is challenging. To address this, we introduce and carefully analyze two intermediate terms (see more details in Appendix \ref{app:concentration}).  Moreover, to accommodate practical applications, we allow general reward functions with polynomial growth. This introduces additional challenges in estimator construction when the domain is unbounded (see \eqref{eq:Q,V estimation for h=H}–\eqref{V-estimates II} and the proof of Theorem~\ref{thm:barV local lipschitz property}). Finally, our regret analysis must also accommodate martingale difference terms that are unbounded, requiring concentration tools more sophisticated than the standard Azuma–Hoeffding inequality (see the proof of Theorem~\ref{thm:first stage regret decomposition}).

\subsection{Closely related literature} 

\paragraph{Uniform partition and adaptive partition.} Uniform partitioning or discretization is a straightforward nonparametric approach for continuous-state problems \cite{bayraktar2023approximate,kara2023q}. However, these methods may suffer from the curse of dimensionality: fine grids are computationally intensive, whereas coarse grids produce inaccurate results and numerical instability \cite{zhang2025state}, limiting their effectiveness in higher dimensions. For example, value iteration has a per-iteration complexity of $\mathcal{O}(|\mathcal{S}|^2|\mathcal{A}|)$, and policy iteration requires $\mathcal{O}(|\mathcal{S}|^3 + |\mathcal{S}|^2|\mathcal{A}|)$ per iteration \cite{puterman2014markov}, with $|\mathcal{S}|$ and $|\mathcal{A}|$ denoting the size of discretized state and action spaces respectively. Moreover, uniform schemes are often suboptimal due to heterogeneous state visit frequencies—leading to wasted resolution on rarely visited states and insufficient resolution where it matters most. The challenge intensifies in {\it unbounded state spaces}, such as those arising in diffusion processes with applications in finance, physics, and engineering. In these settings, discretization typically requires domain truncation, which introduces bias, while extending grids to the full space is computationally prohibitive. Scalable and principled methods for such domains remain largely {\it unresolved}.

Adaptive partition in RL addresses the inefficiency of uniform grids by refining the state-action spaces {\it only where needed}. Early methods, such as U-Tree \cite{mcCallum1996utree} and variable-resolution discretization \cite{munos2002variable}, focused on adaptively partitioning the state space based on visitation frequency or value approximation error. Subsequent works incorporated function approximation and confidence bounds to guide refinement more systematically \cite{strehl2006pac, munos2008finite, ortner2014regret}. While some algorithms extend adaptivity to continuous action spaces under smoothness assumptions \cite{pazis2013pac, dong2019provably}, jointly handling continuous, high-dimensional state-action spaces, especially under complex dynamics, remains a major challenge. A notable recent exception is \cite{sinclair2023adaptive}, which proposes an adaptive partitioning method for MDPs with bounded, continuous state-action spaces.

\paragraph{Zooming algorithms.}
The use of zooming algorithms for adaptive partitioning was initially developed in the contextual multi-armed bandits (MAB) literature, particularly for problems with Lipschitz structure. \cite{kleinberg2008multi} introduced a zooming algorithm for adaptive exploration and defined the {zooming dimension} to quantify the complexity of such problems. Building on this, \cite{slivkins2011contextual} extended the approach to contextual bandits, proposing a zooming algorithm for adaptive partitioning of the context-action space and analyzing its regret. 

These ideas were later generalized to RL by \cite{sinclair2023adaptive}, who studied adaptive partition in finite-horizon RL with {\it bounded} state-action spaces, assuming a bounded reward function. 
They proposed both model-free and model-based algorithms and provided unified regret bounds
in terms of a worst-case zooming dimension $z'_{\max,c}$ defined under bounded state
assumptions; see Remark~\ref{re:comparison} for the precise rates. These algorithms inspire
the design of our framework, though our setting departs from theirs in several important ways.


More recently, \cite{kar2024adaptive} proposed adaptive partitioning algorithms for non-episodic RL with infinite time horizons, under an ergodicity assumption. Their model-based algorithm attains a regret bound of order \(\tilde{\mathcal{O}}\left(T^{1 - \frac{1}{2d_{\mathcal{S}} + z + 3}}\right)\), where \(T\) denotes the total number of decision steps, \(d_{\mathcal{S}}\) is the dimension of the state space, and \(z\) is the zooming dimension tailored to their setting.

\paragraph{Discrete-time approximation of controlled diffusions.} Our MDP is an Euler--Maruyama discretization of a controlled stochastic differential equation, which naturally connects our formulation with the literature on discrete-time approximations of continuous-time diffusion control problems. The convergence theory for time discretizations of SDEs is classical \cite{kloeden1992numerical}, and the weak-convergence framework of \cite{kushner2001numerical} provides a general methodology for approximating controlled diffusions by discrete-time controlled Markov chains. Related work studies probabilistic interpretations and convergence rates of such schemes \cite{tan2014discrete}, partial-observation extensions \cite{li2024discrete}, and learning-based discretization methods for controlled diffusions \cite{kara2023qlearning,pradhan2025discrete,bayraktar2026reinforcement}. Our focus is complementary: for a fixed discretization level $\Delta$, we study finite-sample regret of an adaptive, end-to-end RL algorithm for the induced discrete-time MDP.

\paragraph{Continuous-time RL under diffusion processes.} As our study concerns diffusion-type dynamics in discrete time, it naturally relates to the literature on RL with system dynamics governed by continuous-time diffusion processes. Recent contributions in this area, such as \cite{wang2020reinforcement,jia2023q,dai2025data,jia2022policy,huang2025sublinear,han2023choquet}, provide elegant mathematical frameworks and demonstrate substantial algorithmic progress. However, aspects such as sample complexity or regret guarantees at the implementation level—particularly in terms of the number of observations collected from the environment—are typically not the primary focus of these works. Moreover, theoretical treatments of unbounded state–action spaces and of implementable sampling schemes for general (non-Gaussian) policies over such spaces remain limited. Addressing these issues constitutes one of the main focuses of our work.

\vspace{10pt}
This paper is organized as follows. Section \ref{sec:set-up} introduces the mathematical formulation of the problem, and Section \ref{sec:algorithm_design} presents the design of our algorithm. We then turn to the technical developments: Section \ref{sec:concentration} establishes concentration inequalities for the estimators used in the algorithm, while Section \ref{sec:regret_analysis} provides the regret analysis.


\section{Mathematical set-up}\label{sec:set-up}
We study a time-discretized controlled diffusion driven by Gaussian increments. Let \([H]:=\{1,\ldots,H\}\). Equip \(\mathbb{R}^{d_{\mathcal S}}\) and \(\mathcal A\) with the Euclidean metrics \(\mathcal D_{\mathcal S}\) and \(\mathcal D_{\mathcal A}\), respectively, where \(\mathcal A\subset\mathbb{R}^{d_{\mathcal A}}\) is a closed hypercube centered at zero with diameter \(2\bar a\). On \(\mathbb{R}^{d_{\mathcal S}}\times\mathcal A\), define $\mathcal D((x,a),(x',a')):=\sqrt{\mathcal D_{\mathcal S}(x,x')^2+\mathcal D_{\mathcal A}(a,a')^2}$. Unless otherwise stated, \(\|\cdot\|\) denotes the \(\ell_2\)-norm.

For $h\in[H-1]$, let
$\mu_h:\mathbb R^{d_{\mathcal S}}\times\mathcal A\to\mathbb R^{d_{\mathcal S}}$
and
$\sigma_h:\mathbb R^{d_{\mathcal S}}\times\mathcal A
\to\mathbb R^{d_{\mathcal S}\times d_{\mathcal S}}$
denote the drift and volatility, respectively. The state evolves according to
\begin{eqnarray}\label{eq:state_process}
    X_{h+1}-X_{h}&=&\mu_h(X_{h},A_{h})\Delta+\sigma_h(X_{h},A_{h})B_{h}\sqrt{\Delta},\quad \textrm{with}\,\, X_{1}=\xi.
\end{eqnarray}
Here $\Delta>0$, $B_h\stackrel{\rm iid}{\sim}\mathcal N(0,I_{d_{\mathcal S}})$, and $\xi\sim\Xi$ is independent. Equation \eqref{eq:state_process} defines a time-discretized  a time-discretized controlled diffusion, with
$T_h(\cdot\mid x,a)=\NN(\mu_h(x,a)\Delta,\Sigma_h(x,a)\Delta)$ for nondegenerate $\sigma_h(x,a)$, where $\Sigma_h=\sigma_h\sigma_h^\top$.

At time $h$, the conditional law of the reward $r_h$ given
$(X_h,A_h)=(x,a)$ is $R_h(x,a)$, where
$R_h:\mathbb R^{d_{\mathcal S}}\times\mathcal A\to\mathcal P(\mathbb R)$.
Write $R:=\{R_h\}_{h\in[H]}$ and
$\bar R_h(x,a):=\int_{\mathbb R}r\,R_h(x,a)(dr)$.

In a stochastic system
$(\mathbb R^{d_{\mathcal S}},\mathcal A,H,\mu,\sigma,R)$, a randomized
policy $\pi:=(\pi_h)_{h\in[H]}$ consists of stochastic kernels
$\pi_h(\cdot\mid x)\in\mathcal P(\mathcal A)$, with
$A_h\sim\pi_h(\cdot\mid X_h)$. It is also called a mixed control strategy
\cite{yong1999stochastic}.

\subsection{Value function, Bellman equations and evaluation criterion}

\paragraph{Bellman equation for generic policy} 
For a randomized policy $\pi$, define the associated value function and state--action value function as
\begin{eqnarray*}
    V_h^\pi (x) := \mathbb{E} \Bigg[ \sum_{h'=h}^H r_{h'}
    \Bigg|X_h=x\Bigg], \,\,
    Q_h^\pi(x,a) := \bar{R}_h(x,a) + \mathbb{E}\left[ \sum_{h'=h+1}^H r_{h'}\,\Bigg|\,X_{h+1}\sim T_h(\cdot|x,a)\right], 
\end{eqnarray*}
subject to $r_{h'}\sim R_{h'}(X_{h'},A_{h'}^\pi)$ and $A_{h'}^\pi\sim\pi_{h'}(X_{h'})$. They satisfy the Bellman equations \cite{puterman2014markov}:
\begin{eqnarray}
     V_{h}^{\pi}(x) &=&\mathbb{E}_{a\sim \pi_h(x)} \Big[Q_{h}^{\pi}(x,a)\Big], \nonumber\\
        Q_{h}^{\pi}(x,a) & =&\bar{R}_{h}(x,a)+\mathbb{E}_{X_{h+1}\sim T_h(\cdot|x,a), a'\sim\pi_{h+1}(X_{h+1})}\Big[Q_{h+1}^{\pi}(X_{h+1},a')\Big], \label{eq:bellman}
\end{eqnarray}
with terminl condition $V_{H+1}^{\pi}=Q_{H+1}^{\pi}=0$; hence
\begin{eqnarray*}
       V_{h}^{\pi}(x) & = \mathbb{E}_{a\sim \pi_{h}(x)}\Big[\bar{R}_{h}(x,a)\Big]+\mathbb{E}_{X_{h+1}\sim T_h(\cdot|x,a), a\sim \pi_h(x)}\Big[V_{h+1}^{\pi}(X_{h+1}) \Big]. 
\end{eqnarray*}

\paragraph{Bellman equation for optimal policy} 
Define the optimal value function as
\begin{eqnarray}\label{eq:optimal_value_def}
    V_h^*(x) = \sup_{\pi} V_h^{\pi}(x).
\end{eqnarray}
Its Bellman equation is
\begin{eqnarray}\label{eq:bellman V star}
 V_h^*(x) = \sup_{a\in \mathcal{A}} \Big\{ \bar{R}_h(x,a) + \mathbb{E}_{X'\sim T_h(\cdot|x,a)}\Big[V_{h+1}^*(X')\Big] \Big\},   
\end{eqnarray}
with terminal condition $V^*_{H+1}=0$. We write
\begin{eqnarray*}
    V_h^*(x) = \sup_{a\in \mathcal{A}} Q_h^*(x,a)
\end{eqnarray*}
where
\begin{eqnarray}\label{eq:bellman Q star}
     Q^*_h(x,a) = \bar{R}_h(x,a)+ \mathbb{E}_{X'\sim T_h(\cdot|x,a)}\Big[V_{h+1}^*(X')\Big].
\end{eqnarray}
Equivalently,
\begin{eqnarray*}
    Q^*_h(x,a) = \bar{R}_h(x,a) + \mathbb{E}_{X'\sim T_h(\cdot|x,a)}\Big[\sup_{a'\in \mathcal{A}}Q_{h+1}^*(X',a')\Big].
\end{eqnarray*}

\paragraph{Objective and evaluation criterion}
Under our standing assumptions, compactness of $\mathcal A$ ensures the
existence of an optimal deterministic policy \cite{puterman2014markov},
represented by $\pi_h^*(\cdot\mid x)=\delta_{a_h^*(x)}$, or simply
$\pi_h^*(x)=a_h^*(x)$. Such a policy need not be unique; throughout,
``optimal policy'' means a deterministic optimal policy.

We learn policies through interaction and quantify their suboptimality relative to $V^*$ by regret.
\begin{definition}
For an algorithm  deploying a sequence of  policies $\{\pi_{k}\}_{k\in[K]}$ with a given sequences of initial states $\{X_1^k\}_{k\in[K]}$, define the regret as
\begin{eqnarray*}
    {\rm Regret}(K):=\sum_{k=1}^K \Big(V_1^*(X_1^k)-V_1^{\pi_{k}}(X_1^k)\Big).
\end{eqnarray*}
\end{definition}

\subsection{Outstanding assumptions}
\label{sec:assumptions}

We impose the following standard assumptions; see, e.g.,
\cite{yong1999stochastic,bubeck2011x}.

\begin{assumption}[Regularity of the dynamics]\label{ass:lipschitz}
Assume there exists constants $\ell_{\mu},\ell_{\sigma}>0$,  $m\in \mathbb{N}$ and $L_{0}> 0$ such that for all $h\in [H-1]$, $x_1,x_2 \in \mathbb{R}^{d_{\mathcal{S}}}$, and $a_1,a_2\in \mathcal{A}$, it holds that:
    \begin{eqnarray*}
        \|\mu_h(x_1,a_1)-\mu_h(x_2,a_2)\|&\leq&\ell_{\mu}\Big(\|x_1-x_2\|+\|a_1-a_2\|\Big),\nonumber\\
        \|\sigma_h(x_1,a_1)-\sigma_h(x_1,a_2)\|&\leq& \ell_{\sigma}\Big(\|x_1-x_2\|+\|a_1-a_2\|\Big),\nonumber\\
          \max_{h\in[H-1]}\{\|\mu_{h}(0,0)\|,\|\sigma_{h}(0,0)\|\}&\leq& L_{0}. 
    \end{eqnarray*} 

In addition, assume the following elliptic condition, i.e., there exists a  constant $\lambda>0$ such that $\forall x \in \mathbb{R}^{d_{\mathcal{S}}}, a \in \mathcal{A},$ and $  h \in [H-1]$:
\begin{eqnarray}\label{ass-eqn:volatility}
    \sigma_{h}(x,a)\sigma_{h}(x,a)^{\top}\succ \lambda I_{d_{\mathcal{S}}}.
\end{eqnarray}
\end{assumption}

\begin{assumption}[Regularity of the reward]\label{ass:expected reward local lipschitz}
Assume the expected reward is local Lipschitz, namely,  there exists  constants $\ell_{r}>0$, $m\in \mathbb{N}$ and $L_{0}> 0$ such that for all $h\in [H]$, $ x_1,x_2 \in \mathbb{R}^{d_{\mathcal{S}}}$, and $ a_1,a_2\in \mathcal{A}$, it holds that:
    \begin{eqnarray*}
    |\bar{R}_{h}(x_1,a_1)-\bar{R}_{h}(x_2,a_2)|&\leq&
{\ell_{r}\Big(\|x_1\|^m+\|x_2\|^m{+1}\Big)\,\Big(\|x_1-x_2\|+\|a_1-a_2\|\Big)},\nonumber\\
     \max_{h\in[H]}|\bar{R}_{h}(0,0)|&\leq&L_{0}.\label{ass-eqn:expected reward local lipschitz}
     \end{eqnarray*}
 In addition, assume that the reward distribution has sub-Gaussian tail decay, i.e., there exists a known constant $\theta>0$ such that $\forall x \in \mathbb{R}^{d_{\mathcal{S}}}, a \in \mathcal{A}, \lambda_{1} \in \mathbb{R}$, and $h \in [H]$:
 \begin{eqnarray}\label{ass-eqn:subGaussian reward}
    \mathbb{E}_{_{r_h\sim R_h(x,a)}}\Big[\exp\Big(\lambda_{1}(r_{h}-\bar{R}_{h}(x,a))\Big)\Big]\leq e^{\frac{\theta \lambda_{1}^2}{2}}.
\end{eqnarray}
\end{assumption}

\begin{assumption}[Regularity of the initial distribution]\label{ass:initial} 
Assume that there exists $p\in \mathbb{N}$ with $p^{2}>(m+1)^{2}(d_{\mathcal{S}}+d_{\mathcal{A}}+2)+(m+1)(2d_{\mathcal{S}}+2m+4)$, such that the initial state $X_{1}=\xi$ of the diffusion process in \eqref{eq:state_process}  satisfies: 
\begin{eqnarray*}
    \mathbb{E}_{\xi\sim \Xi}[\|\xi\|^{p}]<+\infty.
\end{eqnarray*}
\end{assumption}
This moment condition holds for sub-Gaussian initial laws, including Gaussian ones.

\subsection{Properties of the dynamics and value functions}
\label{sec:properties}
We establish several preliminary results under the assumptions of
Section~\ref{sec:assumptions}.

\begin{proposition}\label{thm:Mp estimation}
Given Assumptions \ref{ass:lipschitz}, \ref{ass:expected reward local lipschitz} and \ref{ass:initial},  there exists a constant $M$ such that 
\begin{eqnarray*}
\mathbb{E}\left[\sup_{h\in[H]}\|X_{h}\|^{p}\right]\leq M\Big(1+\mathbb{E}_{\xi\sim \Xi}[\|\xi\|^{p}]\Big),
\end{eqnarray*}
where $M$ depends only on $H,\ell_{\mu},\ell_{\sigma},p,\bar{a},L_{0}$ and $\Delta$. 
\end{proposition}
The proof of Proposition \ref{thm:Mp estimation} is deferred to Appendix \ref{app:proof-2-2}. Proposition \ref{thm:Mp estimation} immediately implies the following result.

\begin{corollary}\label{Theorem: R-estimate}
Assume Assumptions \ref{ass:lipschitz},\ref{ass:expected reward local lipschitz} and \ref{ass:initial} hold.  For any given  $\rho>0$, there exists a constant $M_{p}$ independent of $\rho$ such that
\begin{eqnarray*}
\mathbb{P}\left(\sup_{h\in[H]}\|X_{h}\|\geq \rho\right) \leq \frac{M_{p}}{\rho^{p}}. 
\end{eqnarray*}
\end{corollary}
Equivalently, with probability at least $1-M_p/\rho^p$,
$\|X_h\|<\rho$ for all $h\in[H]$.

We next establish local Lipschitz continuity of the optimal value 
and polynomial growth of all policy value functions, both used in the
algorithm and regret analysis.

\begin{proposition}[Local Lipschitz property of the value function]\label{thm:value function local Lipschitz}
Suppose Assumptions \ref{ass:lipschitz} and \ref{ass:expected reward local lipschitz} hold. Then for each  $h \in [H]$, it holds that
\begin{eqnarray}\label{eq:value function local Lipschitz}
   | V_{h}^{*}(x_1)- V_{h}^{*}(x_2)|\leq \overline{C}_{h}\Big(1 + \|x_1\|^m +\|x_2\|^m\Big)\|x_1-x_2\|,
\end{eqnarray}
 with $\overline{C}_{h}:=\overline{C}_{h}(\overline{C}_{h+1},L_{0},\ell_{\mu},\ell_{\sigma},\ell_{r},\Delta,m)$.
 \end{proposition}

The proof of Proposition \ref{thm:value function local Lipschitz} is deferred to Appendix \ref{app:proof-2-5}. 

Under Assumption~\ref{ass:lipschitz}, every admissible policy value function
has polynomial growth of order $m+1$.

\begin{proposition} \label{thm: value function with any policy growth rate}
Suppose Assumptions \ref{ass:lipschitz} and \ref{ass:expected reward local lipschitz} hold. Then for all $h \in [H]$ and any policy $\pi$, we have
\begin{eqnarray}\label{eq:value function with any policy growth rate}
    |V_{h}^{\pi}(x)|\leq\widetilde{C}_{h}(\|x\|^{m+1}+1),
\end{eqnarray}
 with constant $\widetilde{C}_{h}:=\widetilde{C}_{h}(\widetilde{C}_{h+1}, L_{0},\ell_{\mu},\ell_{\sigma},\ell_{r},\bar{a},H,h,\Delta,m)$. 
\end{proposition} 

The proof of Proposition \ref{thm: value function with any policy growth rate} is deferred to Appendix \ref{app:proof-2-6}.

\section{Algorithm design}\label{sec:algorithm_design}
We present a value-based algorithm; Sections~\ref{sec:concentration}
and~\ref{sec:regret_analysis} establish its guarantees. The algorithm
maintains blockwise estimates of the $Q$- and value functions, acts greedily
with respect to the $Q$-estimates, and refines blocks using the bias--variance
principle of \cite{sinclair2023adaptive}, extended here to unbounded state
spaces.

\paragraph{Initial state partition} To handle the unbounded state space, we restrict learning and optimization to
$$\mathcal{S}_1 :=\Big\{x\in \mathbb{R}^{d_{\mathcal{S}}}\,\,\Big|\,\, \|x\|\leq \rho\Big\},$$ 
where $\rho$ is chosen in Section~\ref{sec:regret_analysis} so that
trajectories remain in $\mathcal S_1$ with high probability. Coarse estimates
outside $\mathcal S_1$ do not affect the leading regret order.

In contrast to \cite{sinclair2023adaptive}, we partition the entire
state--action space $\mathbb R^{d_{\mathcal S}}\times\mathcal A$ into a
collection $\mathcal Z_D$ of hypercubes of diameter $D$. \footnote{It suffices that $\frac{2\bar{a}\sqrt{d_{\mathcal{S}}+d_{\mathcal{A}}}}{D\sqrt{d_{\mathcal{A}}}}$ is a positive integer.}
Initialize the active block collection as
   \begin{eqnarray*}
    \mathcal{B}_{D}:=\left\{B \,\Big|\,B \in \mathcal{Z}_{D}, B\cap (\mathcal{S}_{1}\times \mathcal{A})\neq \emptyset\right\},
\end{eqnarray*}
This collection is finite, and only its blocks are refined. Define
\begin{eqnarray}
    \bar{Z}:=\cup_{B\in \mathcal{B}_{D}} B. \label{eq:partition_space}
\end{eqnarray}
so that $\mathcal S_1\times\mathcal A\subseteq\bar Z$.

Inspired by \cite{sinclair2023adaptive}, our proposed APL-Diffusion Algorithm (see Algorithm~\ref{alg:ML}) repeats:
\begin{itemize}
    \item Construct the estimators $\overline{Q}_{h}^k(.), \overline{V}_h^k(.)$ for the $Q$-function and the value function;
    \item Select block $B_h^k$ according to the estimated $Q$-function;
    \item Construct the confidence level ${\rm CONF}_h^k(B_h^k)$ for each visited block $B_h^k$;
    \item If ${\rm CONF}_h^k(B_h^k)\leq {\rm diam}(B_h^k)$, split the block $B_h^k$ 
    \begin{itemize}
        \item Each side of the block is divided evenly into two parts along every dimension,

\item As a result, $B_k^h$ is split into smaller (closed) hypercubes with half the diameter of the original block.
    \end{itemize}
\end{itemize}
Our $Q$- and value-function estimators and confidence measure differ from
those in \cite{sinclair2023adaptive} to accommodate the unbounded state
space and diffusion dynamics.


\begin{algorithm}[H]
\caption{Adaptive Partition and Learning for Diffusions (APL-Diffusion)}\label{alg:ML}
\begin{algorithmic}[1]
\State {\bf Initialize:}\label{Initial Partition} 
Initialize the partition $\mathcal{P}_h^0= \mathcal{B}_{D}$ for $h
\in [H]$ and  the counting with
   $ n_h^0(B)=0$ for  $B \in\mathcal{P}_h^0$.
Also, initialize the function estimators 
$\overline{Q}_h^0,\overline{V}_h^0$ according to \eqref{eq:initial value for estimation} and $\overline{Q}_h^k(\bar{Z}^{\complement})$ according to \eqref{eq:initialize outside Q} for $k\in [K]\cup\{0\}$  

\For{each episode $k=1,2,\cdots,K$}
\For{each step $h=1,2,\cdots,H$}
\State Observe $X_h^k$ 
\State \label{block selection} Select $B_h^k$ by BLOCK SELECTION$(X_h^k)$
\State \label{projection operator II} Take action: $A_h^k$ uniformly sampled from the action set $\Gamma_{\mathcal{A}}(B_h^k)$  
\State Receive $r_{h}^k$ and transition to $X_{h+1}^k$ 
\EndFor
\For{each step $h=H,H-1,\cdots,1$}
\State UPDATE COUNTS $(B_h^k)$
\State SPLITTING $(B_h^k)$
\State UPDATE ESTIMATE $(X_h^k,A_h^k, X_{h+1}^k,r_{h}^k,B_h^k)$ 
\EndFor
\EndFor

\end{algorithmic}
\end{algorithm}

\noindent {\bf Projection operators (line \ref{projection operator II} in Algorithm \ref{alg:ML} and line \ref{projection operator I} in Algorithm \ref{alg:ML 2}).} For $B\subset\mathbb R^{d_{\mathcal S}}\times\mathcal A$, let
$\Gamma_{\mathcal S}(B)$ and $\Gamma_{\mathcal A}(B)$ denote its projections
onto the state and action spaces, respectively. 

\vspace{10pt}

The algorithm comprises the following three subroutines.

\begin{algorithm}[H]
    \caption{BLOCK SELECTION$(X_h^k)$}\label{alg:ML 2}
\begin{algorithmic}[1]
    \State \label{projection operator I}
    Determine ${\rm RELEVANT}_h^k(X_h^k) = \{B\in \mathcal{P}_h^{k-1}\cup\{\bar{Z}^{\complement}\}|X_h^k\in \Gamma_{\mathcal{S}}(B)\}$ 
\State \label{Greedy selection rule} 
Greedy selection rule: select $B_h^k \in \arg \max _{B\in {\rm RELEVANT}_h^k(X_h^k)}
\overline{Q}_h^{k-1}(B)$ 
\end{algorithmic}
\end{algorithm}
\vspace{-5pt}
At state $X_h^k$, Algorithm~\ref{alg:ML 2} considers all blocks $B$ with $X_h^k\in\Gamma_{\mathcal S}(B)$ and selects one with the largest $\overline Q_h^{k-1}(B)$.
\vspace{-5pt}
\begin{algorithm}[H]
\caption{UPDATE COUNTS$(B_h^k)$}\label{alg:ML 6}
\begin{algorithmic}[1]
\For{$B\in \mathcal{P}_h^{k-1}$}
\State \label{counts} 
Update $n_h^k(B)$ via \eqref{eq:update counts}
\EndFor
\end{algorithmic}
\end{algorithm}

\noindent {\bf Counts updates (line \ref{counts} in Algorithm \ref{alg:ML 6}).} $n_h^k(B)$ counts visits to $B$ or its ancestors through episode $k$; only the visited block is incremented:
\begin{eqnarray}
    n_h^k(B_h^k)=n_h^{k-1}(B_h^k)+1,
  \quad  n_h^k(B)=n_h^{k-1}(B). \label{eq:update counts}
\end{eqnarray}

\begin{algorithm}[H]
\caption{SPLITTING$(B_h^k)$}\label{alg:ML 4}
\begin{algorithmic}[1]
\State\label{splitting rule}  {\bf If} $B_h^k \in \mathcal{P}_h^{k-1}$, ${\rm CONF}_h^k(B_h^k)\leq {\rm {\rm diam}}(B_h^k)$ 
{\bf then}:
\State $\qquad$ Construct $\mathcal{P}(B_h^k)=\{B_1,\cdots,B_{2^{d_{\mathcal{S}}+d_{\mathcal{A}}}}\}$ as the partition of $B_h^k$ such that each $B_i$ is a (closed) hypercube with ${\rm diam}(B_i)=\frac{{\rm diam}(B)}{2}$ such that $\cup_{i=1}^{2^{d_{\mathcal{S}}+d_{\mathcal{A}}}} B_i = B_h^k$
\State $\qquad$ Update $\mathcal{P}_h^k = \mathcal{P}_h^{k-1}\cup \mathcal{P}(B_h^k)\setminus	 B_h^k$
\For{ $B_1,\cdots,B_{2^{d_{\mathcal{S}}+d_{\mathcal{A}}}}$ } 
\State Initialize $n_h^k(B_i)=n_h^k(B_h^k)$
\EndFor
\State {\bf Else} $B_h^k \in \mathcal{P}_h^{k-1}$ with ${\rm CONF}_h^k(B_h^k)> {\rm {\rm diam}}(B_h^k)$ or $B_h^k=\bar{Z}^{\complement}$
{\bf then}:
\State $\qquad$ Update $\mathcal{P}_h^k = \mathcal{P}_h^{k-1}$
\end{algorithmic}
\end{algorithm}

\noindent {\bf Splitting rule (line \ref{splitting rule} in Algorithm \ref{alg:ML 4}).}
We split the visited block if
\begin{eqnarray}\label{eq:spliting formula}
    {\rm CONF}_h^k(B_h^k)\leq {\rm diam}(B_h^k),
\end{eqnarray}
where ${\rm CONF}_h^k$ defined in  \eqref{eq:CONF} quantifies statistical uncertainty. Thus refinement occurs once block bias, proportional to diameter, dominates this uncertainty.

\begin{algorithm}[H]
\caption{UPDATE ESTIMATE $(X_h^k,A_h^k,X_{h+1}^k,r_{h}^k,B_h^k)$}\label{alg:ML 3}
\begin{algorithmic}[1]
\For{$B\in \mathcal{P}_h^k$}
\State \label{estimators} 
Update the following quantities:
\begin{itemize}
    \item $\widehat{\mu}_h^k(B)$, $\widehat{\Sigma}_h^k(B)$ and $\bar{T}_h^k(\cdot|B)$ via \eqref{eq:drift and volatility function estimator}
    \item $\hat{R}_h^k(B)$ via \eqref{eq:reward function estimator}
\end{itemize}
\State Update $\overline{Q}_h^k$ and $\overline{V}_h^k$ via \eqref{eq:Q,V estimation for h=H}-\eqref{V-estimates II} 
\EndFor
\end{algorithmic}
\end{algorithm}

\noindent {\bf Estimators (line \ref{estimators} in Algorithm \ref{alg:ML 3}).}
For each $B\in\mathcal P_h^k$,
$\hat R_h^k(B)$, $\widehat\mu_h^k(B)$, $\widehat\Sigma_h^k(B)$, and
$\bar T_h^k(\cdot\mid B)$ estimate $\bar R_h(x,a)$, $\mu_h(x,a)$,
$\Sigma_h(x,a)$, and $T_h(\cdot\mid x,a)$, respectively, for $(x,a)\in B$;
$\overline Q_h^k(B)$ estimates $Q_h^*(x,a)$ on $B$, and
$\overline V_h^k(x)$ estimates $V_h^*(x)$.

For further use, we define the APL-Diffusion policy as the sequences of policies described in line \ref{block selection} and \ref{projection operator II} in Algorithm \ref{alg:ML} and denote it by 
\begin{eqnarray}\label{our policy}
    \{\tilde{\pi}^{k}\}_{k\in [K]}.
\end{eqnarray}

\noindent {\bf Demonstration of the algorithm.} Figure~\ref{fig:Illustration} illustrates the construction. In panel (a),
darker shading indicates larger $Q_h^*$ values; the left plot shows finer
cells in high-value regions, while the right plot enlarges the current
partition $\mathcal P_h^{k-1}$. Panel (b) displays the partition history
$\{\mathcal P_h^j\}_{j=0}^{k-1}$ as a tree.

\begin{figure}[H]
    \centering
    \begin{subfigure}[t]{0.66\textwidth}
        \centering
        \includegraphics[width=\linewidth,height=2.2in,keepaspectratio]{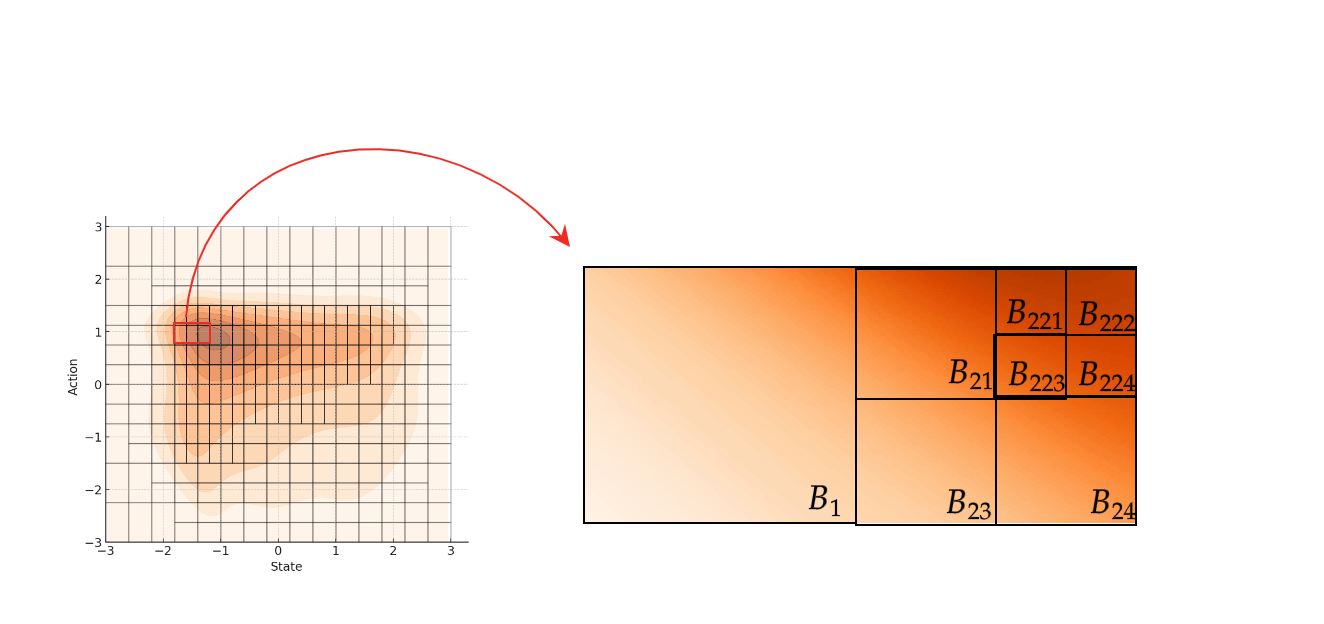}
        \caption{Illustration of the adaptive partition. The right panel zooms \\
        into the current partition $\mathcal{P}_h^{k-1}$.}
    \end{subfigure}\hfill
    \begin{subfigure}[t]{0.30\textwidth}
        \centering
        \includegraphics[width=\linewidth,height=2.0in,keepaspectratio]{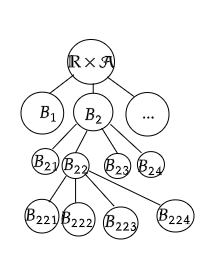}
        \caption{History partitions $\{\mathcal{P}_h^{j}\}_{j=0}^{k-1}$ .}
    \end{subfigure}
  
    \caption{Partitioning scheme for $\mathbb{R}\times\mathcal{A}=(-\infty,+\infty)\times [-3,3]$.}
    \label{fig:Illustration}
\end{figure}

\section{Concentration inequalities of the estimators}\label{sec:concentration}
We derive concentration bounds for the drift and covariance estimates; Appendix~\ref{app:concentration} handles the dependence in covariance estimation induced by the plug-in mean.

Let $k_1\le\cdots\le k_{n_h^k(B)}\le k$ be the episodes in which $B$ or its ancestor was visited.

For $(h,k)\in[H-1]\times[K]$, $B\in\mathcal P_h^k$, and $n_h^k(B)>0$, define
   \begin{eqnarray}
        \widehat{\mu}_{h}^{k}(B)&=&\frac{\sum_{i}(X_{h+1}^{k_{i}}-X_{h}^{k_{i}})}{n_h^k(B)\Delta},\nonumber\\
    \widehat{\Sigma}_{h}^{k}(B)&=&\frac{\sum_{i}\Big((X_{h+1}^{k_{i}}-X_{h}^{k_{i}})-\widehat{\mu}_{h}^{k}(B)\Delta\Big)\Big((X_{h+1}^{k_{i}}-X_{h}^{k_{i}})-\widehat{\mu}_{h}^{k}(B)\Delta\Big)^{\top}}{n_h^k(B)\Delta}.\label{eq:drift and volatility function estimator}
    \end{eqnarray} 
If $n_h^k(B)=0$, set
$\widehat\mu_h^k(B)=\widehat\Sigma_h^k(B)=0$. The estimated transition
kernel is
\begin{eqnarray*}
\bar{T}_h^k(\cdot|B):=\NN\Big(\widehat{\mu}_h^k(B)\Delta, \widehat{\Sigma}_h^k(B)\Delta\Big).
\end{eqnarray*}

For brevity, write
\begin{align*}
 \overline{\mathbb E}[\cdot]&:=\mathbb E[\cdot\mid X_h^{k_1},A_h^{k_1},\ldots,X_h^{k_{n_h^k(B)}},A_h^{k_{n_h^k(B)}}],\\
 \overline{\mathbb V}[\cdot]&:=\mathbb V[\cdot\mid X_h^{k_1},A_h^{k_1},\ldots,X_h^{k_{n_h^k(B)}},A_h^{k_{n_h^k(B)}}].
\end{align*}
Let
\begin{eqnarray}\label{eq:L}
    L:=\max\{\ell_{\mu},\ell_{\sigma}\},
\end{eqnarray}
where $\ell_\mu$ and $\ell_\sigma$ are from Assumption~\ref{ass:lipschitz}. For $B\in\mathcal P_h^k$, let
\begin{eqnarray}\label{eq:center definition}
   \mbox{$\tilde{x}(B),\tilde{a}(B)$ as centers of $\Gamma_{\mathcal{S}}(B),\Gamma_{\mathcal{A}}(B)$ respectively.}
\end{eqnarray}
Finally, $^{o}B$  denote the unique initial block satisfying
\begin{eqnarray}
    \label{eq:oB definition}
    B\subset \, ^{o}B \mbox{  such that } ^{o}B \in \mathcal{B}_{D}.
\end{eqnarray}

With these notations, we have, for any $(x,a)\in B$ and $f=\mu_{h},\sigma_{h}$,
{\small
    \begin{eqnarray}
        \|f(x,a)\|&\leq& \Big\|f(\tilde{x}(^{o}B),\tilde{a}(^{o}B))\Big\|+\Big\|f(x,a)-f(\tilde{x}(^{o}B),\tilde{a}(^{o}B))\Big\|\nonumber\\
        &\leq& L_{0}+L(\|\tilde{x}(^{o}B)\|+\bar{a})+2LD:=\eta (\|\tilde{x}(^{o}B)\|),\label{eq:Lip-bound}\label{eq:eta function}
    \end{eqnarray}
}
    in which we defined $\eta: \mathbb{R}_+\cup\{0\}\mapsto \mathbb{R}_+$.

\subsection{Concentration inequalities for transition kernel estimators}
We first bound the Wasserstein distance between the estimated and true transition kernels; see Appendix~\ref{app:high_prob_w_bound}.
\begin{theorem}\label{thm:high_prob_w_bound} Given Assumption \ref{ass:lipschitz}, it holds with probability at least $1-2\delta$ that, for any $ (h,k) \in [H-1]\times [K]$, $B \in \mathcal{P}_h^k$ with $n_h^k(B)>0$, and any $(x,a) \in B$,  

    \begin{eqnarray}
&&\mathcal{W}_{2}\Big(\NN(\widehat{\mu}_{h}^{k}(B)\Delta,\widehat{\Sigma}_{h}^{k}(B)\Delta),\NN(\mu_{h}(x,a)\Delta,\Sigma_{h}(x,a)\Delta)\Big) \nonumber\\
    &\leq& \Delta\kappa_{\mu}(\delta,\|\tilde{x}(^{o}B)\|,n_h^k(B))+ \frac{\Delta^{\frac{3}{2}}}{\sqrt{\lambda}}\kappa_{\mu}(\delta,\|\tilde{x}(^{o}B)\|,n_h^k(B))^{2}+\frac{\sqrt{d_{\mathcal{S}}}\Delta^{\frac{1}{2}}}{\sqrt{\lambda}}\kappa_{\Sigma}(\delta,\|\tilde{x}(^{o}B)\|,n_h^k(B))\nonumber\\
    &&\hspace{12mm}+
    \Big\|\overline{\mathbb{E}}[\widehat{\mu}_{h}^{k}(B)]-\mu_{h}(x,a)\Big\|\Delta+\Big\|\overline{\mathbb{E}}[\widetilde{\Sigma}_{h}^{k}(B)]-\Sigma_h(x,a)\Big\|\frac{\sqrt{\Delta}}{\sqrt{\lambda}},\label{eq:high_prob_w_bound}
   \end{eqnarray}

where  $\kappa_{\mu}:(0,1]\times (\mathbb{R}_+\cup\{0\})\times \mathbb{N}_+ \mapsto \mathbb{R}_+$ is defined as 
\begin{eqnarray*}
    \kappa_{\mu}(\delta,y,n) :=\frac{\eta(y)}{\sqrt{\Delta}}\Big(\sqrt{\frac{d_{\mathcal{S}}}{n}} + \sqrt{\frac{2\log(\frac{HK^2}{\delta})}{n}}\Big);
\end{eqnarray*}
$\kappa_{\Sigma}:(0,1]\times (\mathbb{R}_+\cup\{0\})\times \mathbb{N}_+ \mapsto \mathbb{R}_+$ is defined as 
\begin{eqnarray*}
    \kappa_{\Sigma}(\delta,y,n) :=\eta(y)^{2}\Bigg( D_1\Big( \sqrt{\frac{d_{\mathcal{S}}}{n}}+\frac{d_{\mathcal{S}}}{\sqrt{n}}\Big) +\Bigg( \sqrt{\frac{\log(\frac{D_2}{d_{\mathcal{S}}})}{D_3n}}+\frac{\log(\frac{D_2 HK^2}{\delta})}{D_3\sqrt{n}}\Bigg)\Bigg).
\end{eqnarray*}
\end{theorem}

We next bound
\begin{eqnarray*}
    \left|\mathbb{E}_{X \sim \bar{T}_{h}^{k}(\cdot|B)}[V_{h+1}^{*}(X)]- \mathbb{E}_{Y \sim {T}_{h}(\cdot|x,a)}[V_{h+1}^{*}(Y)]\right|.
\end{eqnarray*}
For $(h,k)\in[H-1]\times[K]$ and $B\in\mathcal P_h^k$ with
$n_h^k(B)>0$, define the transition uncertainty measure by
\begin{eqnarray}\label{eq:TUCB}
    \mbox{\rm T-UCB}_{h}^{k}(B)&:=&L_{V}(\delta, \|\tilde{x}(^{o}B)\|)\times \Bigg(\kappa_{\mu}(\delta,\|\tilde{x}(^{o}B)\|,n_h^k(B))\Delta+ \frac{\Delta^{\frac{3}{2}}}{\sqrt{\lambda}}\kappa_{\mu}(\delta,\|\tilde{x}(^{o}B)\|,n_h^k(B))^{2}\nonumber\\   &&+\frac{\sqrt{d_{\mathcal{S}}}\Delta^{\frac{1}{2}}}{\sqrt{\lambda}}\kappa_{\Sigma}(\delta,\|\tilde{x}(^{o}B)\|,n_h^k(B))\Bigg), \quad h<H, \nonumber\\
     &&\hspace{12mm}\mbox{\rm T-UCB}_{H}^{k}(B):=0,
\end{eqnarray}
where
\begin{eqnarray}
   L_{V}(\delta,y)
   &:=&\sqrt{3}\,\overline{C}_{\max}\Bigg(1+{\widetilde{C}(m,d_{\mathcal{S}})}\Bigg( 2^{m}\Big((\sqrt{n}\,\,\kappa_{\mu}(\delta,y,n))^{m}+\eta(y)^m\Big)\Delta^{m} \nonumber\\
        &&\quad+3^{\frac{m}{2}}\Big((\sqrt{n}\,\,\kappa_{\mu}(\delta,y,n))^{m}\Delta^{\frac{m}{2}}+(\sqrt{n}\,\,\kappa_{\Sigma}(\delta,y,n))^{\frac{m}{2}}+\Big(\eta(y)^2+L^2D^2\Delta\Big)^{\frac{m}{2}}\Big)\Delta^{\frac{m}{2}}\nonumber\\
&&\quad+\eta(y)^m\Delta^{m}+\eta(y)^m\Delta^{\frac{m}{2}}\Bigg)\Bigg)\label{eq:Lv definition}
\end{eqnarray}
with the constants $\overline{C}_{max}$ and $\widetilde{C}(m,d_{\mathcal{S}})$ defined by
 \begin{eqnarray}
     \overline{C}_{\max}&:=&\max_{h \in [H]}\bCh,\label{eq:barCmax}\\
    \widetilde{C}(m,d_{\mathcal{S}})&:=&d_{\mathcal{S}}^{\frac{3}{4}m+1}2^{\frac{3m-1}{2}}\frac{\Gamma(m+\frac{1}{2})^{\frac{1}{2}}}{\pi^{\frac{1}{4}}}.\label{eq:C(q,d)}
\end{eqnarray}

Hence, we can bound the difference of expected value functions using the T-UCB function.
\begin{theorem}\label{thm:transition kernel wasserstein local lipschitz all together}
Assume Assumption \ref{ass:lipschitz} 
holds. With probability at least $1-2\delta$, we have that, for any $(h,k)  \in [H-1] \times [K]$, $B \in \mathcal{P}_h^k$ with $n_h^k(B)>0$, and $(x,a) \in B$:
\begin{eqnarray}
     &&\hspace{12mm}\left|\mathbb{E}_{X \sim \bar{T}_{h}^{k}(\cdot|B)}[V_{h+1}^{*}(X)]- \mathbb{E}_{Y \sim {T}_{h}(\cdot|x,a)}[V_{h+1}^{*}(Y)]\right|\label{eq:transition kernel wasserstein local lipschitz all together}\\
     &\leq&  \mbox{\rm T-UCB}_{h}^{k}(B)
     +L_{V}(\delta, \|\tilde{x}(^{o}B)\|)\,\,\Big(\Big\|\overline{\mathbb{E}}[\widehat{\mu}_{h}^{k}(B)]-\mu_{h}(x,a)\Big\|\Delta+\Big\|\overline{\mathbb{E}}[\widetilde{\Sigma}_{h}^{k}(B)]-\Sigma_h(x,a)\Big\|\frac{\sqrt{\Delta}}{\sqrt{\lambda}}\Big).\nonumber
\end{eqnarray}
\end{theorem}

The proof of Theorem \ref{thm:transition kernel wasserstein local lipschitz all together} is deferred to Appendix \ref{app:proof-4-8}.

\subsection{Concentration on reward estimators and properties of adaptive partition}
For $(h,k)\in[H]\times[K]$ and $B\in\mathcal P_h^k$, define
\begin{eqnarray}\label{eq:reward function estimator}
    \widehat{R}_{h}^{k} (B)&=&\frac{\sum_{i=1}^{n_h^k(B)}r_{h}^{k_{i}}}{n_h^k(B)},\quad \mbox{if}\,\,  n_h^k(B)>0,\nonumber\\
 \widehat{R}_{h}^{k} (B)&=&0,\quad \mbox{if} \,\, n_h^k(B)=0,
\end{eqnarray}
where $r_h^{k_i}$ is the step-$h$ reward in episode $k_i$.

For $n_h^k(B)>0$, define the reward uncertainty measure by
\begin{eqnarray}\label{eq:RUCB definition}
    \mbox{\rm R-UCB}_{h}^{k}(B):=\sqrt{\frac{\log(\frac{2HK^{2}}{\delta})\,\theta}{n_h^k(B)}},
\end{eqnarray}
with $\theta$ defined in \eqref{ass-eqn:subGaussian reward}.

\begin{theorem}\label{thm:reward high prob bound}
Under Assumption \ref{ass:expected reward local lipschitz}, It holds with probability at least $1-\delta$ that, for any  $(h,k) \in [H]\times [K], B\in \mathcal{P}_h^k$ with $n_h^k(B)>0$, and any $(x,a) \in B$, 
\begin{eqnarray}\label{eq:reward high prob bound}
    |\widehat{R}_{h}^{k}(B)-\bar{R}_h(x,a)| \leq \mbox{\rm R-UCB}_{h}^{k}(B)+\left|\frac{\sum_{i=1}^{n_h^k(B)}\bar{R}_{h}(X_{h}^{k_{i}},A_{h}^{k_{i}})}{n_h^k(B)}-\bar{R}_h(x,a)\right|.
\end{eqnarray}
\end{theorem}

The proof of Theorem \ref{thm:reward high prob bound} is deferred to Appendix \ref{app:reward high prob bound 1}.

To upper bound $\mbox{\rm R-UCB}_{h}^{k}(B)+ \mbox{\rm T-UCB}_{h}^{k}(B)$, we construct the confidence of a block  by 
\begin{eqnarray}\label{eq:CONF}
    {\rm CONF}_{h}^{k}(B)=\frac{g_{1}(\delta, \|\tilde{x}(^{o}B)\|)}{\sqrt{n_h^k(B)}},
\end{eqnarray}
for all $ (h,k) \in [H]\times [K], B \in \mathcal{P}_h^k $ with $n_h^k(B)>0$. Here
$g_{1}: (0,1]\times (\mathbb{R}_+\cup\{0\})\mapsto \mathbb{R}_+$ is defined as 
\begin{equation}\label{eq:g1 definition}
\begin{aligned}
g_{1}(\delta,y):=\sqrt{n}\Bigg(&L_{V}(\delta,y)\Bigg(\kappa_{\mu}(\delta,y,n)\Delta
 +\sqrt{n}\frac{\Delta^{\frac{3}{2}}}{\sqrt{\lambda}}\kappa_{\mu}(\delta,y,n)^{2}\\
&+\frac{\sqrt{d_{\mathcal{S}}}\Delta^{\frac{1}{2}}}{\sqrt{\lambda}}\kappa_{\Sigma}(\delta,y,n)\Bigg)
 +\sqrt{\frac{\log(\frac{2HK^{2}}{\delta})\,\theta}{n}}\Bigg).
\end{aligned}
\end{equation}
Here $g_1$ aggregates local reward, drift, and covariance uncertainty, and hence
${\rm CONF}_h^k(B)$ in \eqref{eq:CONF} is the sampling-error bound used for splitting in Algorithm~\ref{alg:ML 4}.

The next lemma bounds the first two empirical moments of the block
diameters.

\begin{lemma}\label{thm:quoted cumulative bound lemma}
 For all $(h,k) \in [H] \times [K]$ and $B \in \mathcal{P}_h^k$ with $n_h^k(B)>0$, we have
\begin{eqnarray}\label{eq:cumulative bound I}
    \frac{\sum_{i=1}^{n_h^k(B)} {\rm diam}(B_{h}^{k_{i}})}{n_h^k(B)} \leq 4 \,{\rm diam}(B)\,\, \mbox{ and }\,\, \frac{\sum_{i=1}^{n_h^k(B)} {\rm diam}(B_{h}^{k_{i}})^{2}}{n_h^k(B)} \leq 4 D\, {\rm diam}(B),
\end{eqnarray}
where $k_{1}\leq ,\cdots, \leq k_{n_h^k(B)}$ are the corresponding episode indices such that $B$ or its ancestors have been visited by Algorithm \ref{alg:ML}. 
\end{lemma}
 We defer proof to Appendix \ref{app:proof-4-10}, which relies on the ${\rm CONF}_{h}^{k}(B)$ specified in \eqref{eq:CONF}.

\subsection{Bias of the estimators}

Next, we provide upper bounds for $\Big\|\overline{\mathbb{E}}[\widehat{\mu}_{h}^{k}(B)]-\mu_{h}(x,a)\Big\|\Delta+\Big\|\overline{\mathbb{E}}[\widetilde{\Sigma}_{h}^{k}(B)]-\Sigma_h(x,a)\Big\|\frac{\sqrt{\Delta}}{\sqrt{\lambda}}$, for which we introduce $\mbox{\rm T-BIAS}(B)$ to represent the block-wise bias in estimating the transition kernel
\begin{eqnarray*}
    \mbox{\rm T-BIAS}(B) =\Bigg(8L\Delta+16L\,\,\eta(\|\tilde{x}(^{o}B)\|)\frac{\sqrt{\Delta}}{\sqrt{\lambda}} +32L^{2}D\frac{\Delta^{\frac{3}{2}}}{\sqrt{\lambda}}+128L^{2}D\frac{\Delta^{\frac{3}{2}}}{\sqrt{\lambda}}\Bigg){\rm diam}(B).
\end{eqnarray*}

\begin{theorem}\label{thm:BIAS Bound}
 With the same assumptions as in Theorem \ref{thm:high_prob_w_bound}, the following inequality holds for all $h\in [H-1],k\in [K]$, $B \in \mathcal{P}_h^k$ with $n_h^k(B)>0$, and any $(x,a) \in B$:
\begin{eqnarray}\label{eq:BIAS BOUND}
     \Big\|\overline{\mathbb{E}}[\widehat{\mu}_{h}^{k}(B)]-\mu_{h}(x,a)\Big\|\Delta+\Big\|\overline{\mathbb{E}}[\widetilde{\Sigma}_{h}^{k}(B)]-\Sigma_h(x,a)\Big\|\frac{\sqrt{\Delta}}{\sqrt{\lambda}} \leq \mbox{\rm T-BIAS}(B). 
\end{eqnarray}
\end{theorem}
The proof of Theorem \ref{thm:BIAS Bound} is deferred to Appendix \ref{app:BIAS Bound}.

Then we derive upper bounds for $\big|\frac{\sum_{i=1}^{n_h^k(B)}\bar{R}_{h}(X_{h}^{k_{i}},A_{h}^{k_{i}})}{n_h^k(B)}-\bar{R}_h(x,a)\big|$, for which we define $\mbox{\rm R-BIAS}(B)$ to represent the block-wise bias in reward estimator
\begin{eqnarray*}
    \mbox{\rm R-BIAS}(B)=4L_m(\|\tilde{x}(^{o}B)\|){\rm diam}(B),
\end{eqnarray*}
with $L_m: \mathbb{R}_+\cup\{0\}\mapsto \mathbb{R}_+$  defined by
\begin{eqnarray*}
L_m(y):=4L\Big(1+2(y+D)^m\Big).
\end{eqnarray*}

\begin{theorem}\label{thm:RBIAS bound}
    With the same assumptions as in Theorem \ref{thm:reward high prob bound}, the following inequality holds for all $(h,k)\in [H]\times [K]$, $B \in \mathcal{P}_h^k$ with $n_{h}^{k}(B)>0$, and any $(x,a) \in B$:
    \begin{eqnarray}\label{eq:RBIAS bound}
\Big|\frac{\sum_{i=1}^{n_h^k(B)}\bar{R}_{h}(X_{h}^{k_{i}},A_{h}^{k_{i}})}{n_h^k(B)}-\bar{R}_h(x,a)\Big|\leq \mbox{\rm R-BIAS}(B). 
    \end{eqnarray}
\end{theorem}
The proof of Theorem \ref{thm:RBIAS bound} is deferred to Appendix \ref{app:RBIAS bound}.


\section{Regret analysis}\label{sec:regret_analysis}\label{sec:5}
We analyze the regret for our adaptive-partition framework.

\subsection{Construction of value estimators}

In this subsection, we construct estimators of the value functions.

To proceed, we first introduce a few notations. Define the block-wise bias consisting both the bias of transition estimator and the  bias of reward estimator:
\begin{eqnarray}\label{eq:BIAS}
    {\rm BIAS}(B)&:=&\mbox{\rm R-BIAS}(B)+L_{V}(\delta, \|\tilde{x}(^{o}B)\|)\mbox{\rm T-BIAS}(B)\nonumber\\
    &:=&g_{2}(\delta, \|\tilde{x}(^{o}B)\|){\rm diam}(B),
\end{eqnarray}
with $g_{2}:(0,1]\times (\mathbb{R}_+\cup\{0\}) \mapsto \mathbb{R}$ defined as 
\begin{eqnarray}\label{eq:g2 definition}
   g_{2}(\delta,y)&:=&4L_m(y)+L_{V}(\delta,y)\Bigg(8L\Delta+16L\eta(y)\frac{\sqrt{\Delta}}{\sqrt{\lambda}}\nonumber\\
   &&+32L^{2}D\frac{\Delta^{\frac{3}{2}}}{\sqrt{\lambda}}+128L^{2}D\frac{\Delta^{\frac{3}{2}}}{\sqrt{\lambda}}\Bigg).
\end{eqnarray}
Here $\eta,\tilde x,{}^oB,L_V$ are defined in \eqref{eq:eta function}, \eqref{eq:center definition}, \eqref{eq:oB definition}, and \eqref{eq:Lv definition}. Thus $g_2$ collects local approximation constants. In addition, ${\rm BIAS}(B)$ is linear in diameter.

In addition, for each \((h,k)\in[H]\times[K]\), define the state partition induced by
\(\mathcal P_h^k\) as
\[
\Gamma_{\mathcal S}(\mathcal P_h^k)
:=
\left\{
\Gamma_{\mathcal S}(B):
B\in \mathcal P_h^k,\ 
\nexists\, B'\in \mathcal P_h^k
\ \text{such that}\
\Gamma_{\mathcal S}(B')\subsetneq \Gamma_{\mathcal S}(B)
\right\}.
\]

For $S\in\Gamma_{\mathcal S}(\mathcal P_h^k)$, let $\tilde x(S)$ denote its center. Set the estimator's local Lipschitz constant as
\begin{eqnarray}\label{eq:local lipschitz constant}
    C_{h}:=\max\Big\{\overline{C}_h, 2^{m+1}\widetilde{C}_h\Big\}, \quad \forall h\in [H],
\end{eqnarray}
 where $\overline C_h$ and $\widetilde C_h$ are defined  in \eqref{eq:value function local Lipschitz} and \eqref{eq:value function with any policy growth rate}. With this choice, the estimator satisfies the local Lipschitz bound in Theorem~\ref{thm:barV local lipschitz property}.

\paragraph{Design of value estimators} 

For $k=0, h\in [H],B\in \mathcal{P}_{h}^{0}, S=\Gamma_{\mathcal{S}}(B)$, and $x\in \mathbb{R}^{d_{\mathcal{S}}}$, the function estimators are initialized with
\begin{eqnarray}\label{eq:initial value for estimation}
        \overline{Q}_h^0(B)&:=& \widetilde{C}_{h}(1+(\|\tilde{x}(^{o}B)\|+D)^{m+1}),\nonumber\\
    \widetilde{V}_h^0(S)&:=&\widetilde{C}_{h}(1+(\|\tilde{x}(S)\|+D)^{m+1}),\nonumber\\
        \overline{V}_h^0(x) &:=& \widetilde{C}_h(1+\|x\|^{m+1}).
\end{eqnarray}
In addition, for $(h,k)\in [H]\times ([K]\cup\{0\})$, we set $\overline{Q}_h^k(\bar{Z}^{\complement})$ as:
\begin{eqnarray}\label{eq:initialize outside Q}
     \overline{Q}_h^k(\bar{Z}^{\complement}):=-\widetilde{C}_{h}(1+\rho^{m+1}).
\end{eqnarray}

For $k\ge1$, $B\in\mathcal P_H^k$, $S\in\Gamma_{\mathcal S}(\mathcal P_H^k)$, and $x\in\mathcal S_1$, let us define
\begin{eqnarray}\label{eq:Q,V estimation for h=H}
     \overline{Q}_{H}^{k}(B) &:=&\left\{\begin{array}{lll}\widehat{R}_{H}^{k}(B)+\mbox{\rm R-UCB}_{H}^{k}(B)+\mbox{\rm R-BIAS}(B) && \mbox{ if } n_H^k(B)>0\\
     \overline{Q}_{H}^{0}(B) &&
       \mbox{ if } n_H^k(B)=0, 
       \end{array}\right. \nonumber\\ 
        \widetilde{V}_{H}^{k}(S)&:=&\min\Big\{\widetilde{V}_{H}^{k-1}(S), \max_{B \in \mathcal{P}_H^k,S \subset\Gamma_{\mathcal{S}}(B) }\overline{Q}_{H}^{k}(B)\Big\} ,\nonumber\\ 
         V_{H,k}^{\rm local}(x, S)&:=& \widetilde{V}_{H}^{k}(S)+C_{H}\Big(1+\|x\|^m+\|\tilde{x}(S)\|^m\Big)\|x-\tilde{x}(S)\|, \nonumber\\ 
        \overline{V}_{H}^{k}(x)&:=&\min_{S \in \Gamma_{\mathcal{S}}(\mathcal{P}_H^k)} V_{H,k}^{\rm local}(x, S).
\end{eqnarray}
For $x \in \mathbb{R}^{d_{\mathcal{S}}}\setminus\mathcal{S}_{1}$, we define 
\begin{eqnarray}\label{eq:Q,V estimation for h=H outside}
    \overline{V}_{H}^{k}(x):=\overline{V}_{H}^{k}\left(\frac{\rho}{\|x\|}x\right)+C_{H}(1+\|x\|^{m}+\rho^{m})\left\|\left(1-\frac{\rho}{\|x\|}\right)x\right\|. 
\end{eqnarray}
This extrapolation ensures that $\overline V_H^k$ is globally continuous.

For $h<H$, $B\in\mathcal P_h^k$, $S\in\Gamma_{\mathcal S}(\mathcal P_h^k)$, and $x\in\mathcal S_1$, let us define
\begin{eqnarray}\label{V-estimate}
    \overline{Q}_{h}^{k}(B) &:=&\left\{\begin{array}{lll}\widehat{R}_{h}^{k}(B)+\mbox{\rm R-UCB}_{h}^{k}(B)+\mathbb{E}_{X \sim \bar{T}_{h}^{k}(\cdot|B)}[\overline{V}_{h+1}^{k}(X)]\nonumber\\
    +\mbox{\rm T-UCB}_{h}^{k}(B)+{\rm BIAS}(B) &&\mbox{if } n_h^k(B)>0 \nonumber\\
     \overline{Q}_{h}^{0}(B) &&
       \mbox{if } n_h^k(B)=0,
       \end{array}\right. \nonumber\\ 
    \widetilde{V}_{h}^{k}(S)&:=&\min\Big\{\widetilde{V}_{h}^{k-1}(S), \max_{B \in \mathcal{P}_h^k,S \subset\Gamma_{\mathcal{S}}(B)}\overline{Q}_{h}^{k}(B)\Big\},\nonumber\\
     V_{h,k}^{\rm local}(x, S)&:=& \widetilde{V}_{h}^{k}(S)+C_{h}\Big(1+\|x\|^m+\|\tilde{x}(S)\|^m\Big)\|x-\tilde{x}(S)\|  , \nonumber\\
  \overline{V}_{h}^{k}(x)&:=&\min_{S \in \Gamma_{\mathcal{S}}(\mathcal{P}_h^k)}
     V_{h,k}^{\rm local}(x,S).
\end{eqnarray}

Finally, for $x \in \mathbb{R}^{d_{\mathcal{S}}}\setminus\mathcal{S}_{1}$, we define 
\begin{eqnarray}\label{V-estimates II}
    \overline{V}_{h}^{k}(x):=\overline{V}_{h}^{k}\left(\frac{\rho}{\|x\|}x\right)+C_{h}(1+\|x\|^{m}+\rho^{m})\left\|\left(1-\frac{\rho}{\|x\|}\right)x\right\|.
\end{eqnarray}

\begin{remark}[Role of $V^{\rm local}_{h,k}$]
   We design 
$V^{\rm local}_{h,k}(.,S)$ as a locally Lipschitz extension of the estimate for $S$ across the entire state space. The local Lipschitz property plays a key role in establishing concentration bounds associated with $\overline{V}_{h}^{k}$. This is formalized in Appendix \ref{app:value-estimators}.
\end{remark}

Equations~\eqref{eq:Q,V estimation for h=H}--\eqref{V-estimates II}
implement value iteration using estimated rewards and transition kernels.
The bonuses $\mbox{\rm R-UCB}_h^k(B)$, $\mbox{\rm T-UCB}_h^k(B)$, and
${\rm BIAS}(B)$ bound reward-estimation error, transition-estimation error,
and partition bias, respectively.

We next show that $\overline Q_h^k$, $\widetilde V_h^k$, and
$\overline V_h^k$ upper-bound their corresponding optimal value functions.

\begin{theorem} \label{thm: True Upper Bounded By Estimates}
Under Assumptions \ref{ass:lipschitz}-\ref{ass:expected reward local lipschitz}, with probability at least $1-3\delta$, it holds that for all $(h,k) \in [H]\times[K]$, 
\begin{eqnarray}\label{eq:upperbound}
     \overline{Q}_{h}^{k}(B)&\geq& Q_{h}^{*}(x,a),\,\, \mbox{for all $B \in \mathcal{P}_h^k $ and  $ (x,a) \in B$}, \nonumber\\
     \widetilde{V}_{h}^{k}(S)&\geq& V_{h}^{*}(x), \,\,\mbox{for all $S \in \Gamma_{\mathcal{S}}(\mathcal{P}_h^k)$ and $x \in S$}, \nonumber\\
     \overline{V}_{h}^{k}(x)&\geq& V_{h}^{*}(x), \,\,\mbox{for all $x \in \mathbb{R}^{d_{\mathcal{S}}}$}.    
\end{eqnarray} 
\end{theorem}
The proof of Theorem \ref{thm: True Upper Bounded By Estimates} is deferred to Appendix \ref{app:proof-5-3}. \\

Next, we show that the estimated value functions satisfy a local Lipschitz property.

\begin{theorem}\label{thm:barV local lipschitz property}
   Under Assumptions \ref{ass:lipschitz}-\ref{ass:expected reward local lipschitz}, with probability at least $1-3\delta$, for all $(h,k) \in [H]\times [K]$, $x_1,x_2\in \mathbb{R}^{d_{\mathcal{S}}}$, 
    \begin{eqnarray*}
\Big|\overline{V}_{h}^{k}(x_1)-\overline{V}_{h}^{k}(x_2)\Big|\leq \hCh \Big(1+\|x_1\|^{m}+\|x_2\|^{m}\Big)\|x_1-x_2\|,
    \end{eqnarray*}
    where 
    \begin{eqnarray}\label{eq:C_hat}
        \hCh:=\hCh(m,C_h,\widetilde{C}_h, D),
    \end{eqnarray}
with $C_h$ defined in \eqref{eq:local lipschitz constant} and $\widetilde{C}_h$ defined in \eqref{eq:value function with any policy growth rate}. 
\end{theorem}


The initialization in \eqref{eq:initial value for estimation} and the value estimates in \eqref{eq:Q,V estimation for h=H}--\eqref{V-estimates II} are crucial for local Lipschitz continuity. Polynomial growth in our setting makes the proof of Theorem~\ref{thm:barV local lipschitz property} more involved than in \cite{sinclair2023adaptive}; see Appendix~\ref{app:proof-5-1}.

\subsection{Upper bound via clipping}
Using the clipping method of \cite[Section E]{sinclair2023adaptive}, we bound
\begin{eqnarray}\label{eq:delta-hk}
    \Delta_{h}^{(k)}:=\overline{V}_h^{k-1}(X_h^k)-V_h^{\tilde{\pi}^{k}}(X_h^k),
\end{eqnarray}
with $\Delta_{H+1}^{(k)}=0$ and $\tilde\pi^k$ from \eqref{our policy}.The upper bound of $\Delta_{h}^{(k)}$ will play an important role in controlling the final regret bound.

The clip function is defined as
\begin{eqnarray}\label{eq:clip}
       {\rm CLIP}(\nu_{1}|\nu_{2}):=\nu_{1}\mathbb{I}_{\nu_{1}\geq \nu_{2}}, \forall \nu_{1},\nu_{2} \in \mathbb{R}.
\end{eqnarray}

Before proceeding, we introduce a few useful notations: 
\begin{eqnarray}
     {\rm \widetilde{G}ap}_h(x,a)&:=&V_h^{*}(x)-Q_h^{*}(x,a);\label{eq:gap definition} \\
     {\rm Gap}_h(B)&:=&\min_{(x,a)\in B}{\rm \widetilde{G}ap}_h(x,a);\label{eq:gapB definition}\\
     f_{h+1}^{k-1}(X_h^k, A_h^k)&:=&\mathbb{E}_{Y \sim {T}_{h}(\cdot|X_h^k,A_h^k)}[\overline{V}_{h+1}^{k-1}(Y)]-\mathbb{E}_{Y \sim {T}_{h}(\cdot|X_h^k,A_h^k)}[V_{h+1}^{*}(Y)],h<H\label{eq:f definition},
\end{eqnarray}
with terminal $f_{H+1}^{k-1}(X_H^k,A_H^k)=0$.

Also, to further ease the notation, for $(h,k)\in [H]\times[K], B_h^k \in \mathcal{P}_h^{k-1}$ we denote:
$$G_h^{k}(B_h^k):=2\frac{\widetilde{C}_h}{D}(1+(\|\tilde{x}(^{o}B_h^k)\|+D)^{m+1}){\rm diam}(B_h^k);$$
otherwise 

\begin{eqnarray}\label{eq:G_h^k definition}
G_h^{k}(B_h^k)
&:=&
2\frac{\widehat{C}_{\max}}{\overline{C}_{\max}}
\Big(
\mbox{\rm R-UCB}_{h}^{k-1}(B_h^k)
+
\Big(
\mbox{\rm T-UCB}_{h}^{k-1}(B_h^k)
+
{\rm BIAS}(B_h^k)
\Big)\mathbb{I}_{h<H}
\\
&&\quad+
\mbox{\rm R-BIAS}(B_h^k)\mathbb{I}_{h=H}
\Big)+
C_h\Big(
1+2(\|\tilde{x}(^{o}B_h^k)\|+D)^m
\Big){\rm diam}(B_h^k).
\nonumber
\end{eqnarray}

where $\tilde{x}$ is defined in \eqref{eq:center definition}, $\widetilde{C}_h$  in \eqref{eq:value function with any policy growth rate}, $C_h$  in \eqref{eq:local lipschitz constant}, $\overline{C}_{\max}$  in \eqref{eq:barCmax}, $\widehat{C}_{\max}:=\max_{h\in[H]}\hCh$ with $\hCh$ defined in \eqref{eq:C_hat}. In addition, $B_h^k$ is selected according to Algorithm \ref{alg:ML 2} and $^{o}B_h^k$ is defined in \eqref{eq:oB definition}.

\begin{remark}[Role of $G_h^k(B_h^k)$] 
We remark that $G_h^k(B_h^k)$ represents the overall bonus terms and bias of the estimate w.r.t the selected block $B_h^k$. By "clipping" it with the gap term, it provides a useful upper bound for us to control the final regret; see more analysis in Lemma \ref{lemma:theorem F.3 conclusion}.
\end{remark}

\begin{theorem}\label{thm: upper bound via clipping}
Suppose Assumptions \ref{ass:lipschitz}-\ref{ass:expected reward local lipschitz} hold. With probability at least $1-3\delta$, for all $ (h,k)\in [H]\times [K]$, $B_h^k\in \mathcal{P}_h^{k-1}$, we have that:

\begin{eqnarray*}
 \Delta_{h}^{(k)}\leq {\rm CLIP}\Bigg( G_{h}^{k}(B_h^k)\,\,\Bigg|\,\,\frac{{\rm Gap}_h(B_h^k)}{H+1}\Bigg)+\Big(1+\frac{1}{H}\Big)f_{h+1}^{k-1}(X_h^k, A_h^k)+Q_h^{*}(X_h^k,A_h^k)-V_h^{\tilde{\pi}^{k}}(X_h^k).
\end{eqnarray*}

\end{theorem}

The proof of Theorem \ref{thm: upper bound via clipping} is deferred to Appendix \ref{app:proof-5-5}.

Theorem~\ref{thm: upper bound via clipping} upper-bounds the one-step gap
$\Delta_h^{(k)}$ in terms of
$f_{h+1}^{k-1}(X_h^k,A_h^k)$ and the suboptimality gap of the deployed-policy. The next
lemma expresses these terms through $\Delta_{h+1}^{(k)}$ and
martingale-difference errors, enabling backward iteration along the episode.

\begin{lemma}[Modification of Theorem F.1 in  \cite{sinclair2023adaptive}]\label{thm:induction} It holds that
$$\Big(1+\frac{1}{H}\Big)f_{h+1}^{k-1}(X_h^k, A_h^k)+Q_h^{*}(X_h^k,A_h^k)-V_h^{\tilde{\pi}^{k}}(X_h^k)\leq \Big(1+\frac{1}{H}\Big(\Delta_{h+1}^{(k)}+\xi_{h+1}^{k}\Big),$$
in which for $h<H$

\begin{eqnarray*}
    \xi_{h+1}^{k}&:=&\mathbb{E}_{Y \sim {T}_{h}(\cdot|X_h^k,A_h^k)}[\overline{V}_{h+1}^{k-1}(Y)]-\mathbb{E}_{Y \sim {T}_{h}(\cdot|X_h^k,A_h^k)}[V_{h+1}^{\tilde{\pi}^{k}}(Y)]-(\overline{V}_{h+1}^{k-1}(X_{h+1}^{k})-V_{h+1}^{\tilde{\pi}^{k}}(X_{h+1}^{k}))\\
    \zeta_{h+1}^k&:=&\bar{R}_{h}(X_h^k,A_h^k)-\mathbb{E}_{a\sim \pi_h^k(X_h^k)}[\bar{R}_{h}(X_h^k,a)]\\
    && +\mathbb{E}_{Y \sim {T}_{h}(\cdot|X_h^k,A_h^k)}[V_{h+1}^{\tilde{\pi}^{k}}(Y)]-  \mathbb{E}_{a\sim \pi_h^k(X_h^k), Y^{'} \sim {T}_{h}(\cdot|X_h^k,a)}[V_{h+1}^{\tilde{\pi}^{k}}(Y^{'})],
\end{eqnarray*}

$\xi_{H+1}^{k}:=0$ and $\zeta_{H+1}^{k}:=\bar{R}_{H}(X_H^k,A_H^k)-\mathbb{E}_{a\sim \pi_H^k(X_H^k)}[\bar{R}_{H}(X_H^k,a)]$. In addition, $\Delta_{h+1}^{(k)}$ is defined in \eqref{eq:delta-hk} for $h<H$, $\Delta_{H+1}^{(k)}:=0$,  and $f_{h+1}^{k-1}(X_h^k, A_h^k)$ is defined in \eqref{eq:f definition}.
\end{lemma}
Lemma \ref{thm:induction}  is adapted from Theorem F.1 of \cite{sinclair2023adaptive}  with minor modifications to suit our setting. For completeness, the proof  is deferred to Appendix \ref{app:thm induction}.

\subsection{Concentrations on the size of $J_{\rho}^{K}$ and initial value function }

In this subsection, we provide some useful concentrations before establishing the final regret bound.

We categorize the sample trajectories into two types: those 
remains within $\mathcal{S}_1$ 
for the {\it entire} episode, denoted by 
 \begin{eqnarray}\label{eq:J0}
J_{\rho}^{K}:=\left\{k \in [K]: \max_{h\in[H]}\|X_{h}^{k}\|\leq \rho \right\},
 \end{eqnarray}
and those 
exceeds $\mathcal{S}_1$.  
We also denote  
\begin{eqnarray*}
I_k:=\mathbb{I}_{\{k\in J^K_\rho\}}, \quad p_{\rho}^{k}:=\mathbb{P}\left(k\in J^K_\rho\right):=\mathbb{E} [I_k], 
\quad \mbox{and}\quad K_{0}=\sum_{k=1}^{K}I_{k} = |J_{\rho}^{K}|.
\end{eqnarray*} 
According to Corollary \ref{Theorem: R-estimate}, we know that 
\begin{eqnarray}\label{ineq:alpha_p}
    p_{\rho}^{k} \geq 1-\frac{M_{p}}{\rho^{p}}, \mbox{ and hence } \mathbb{E}[K_{0}]\geq K\left(1-\frac{M_{p}}{\rho^{p}}\right).
\end{eqnarray}
We also have the following concentration bound for $K_0$.
\begin{proposition}\label{thm:concentration on size of J0}
Suppose Assumptions \ref{ass:lipschitz}, \ref{ass:expected reward local lipschitz} and \ref{ass:initial} hold. The following holds with probability at least $1-\delta$, 
\begin{eqnarray*}
     K-K_{0} \leq \frac{KM_{p}}{\rho^{p}}+\sqrt{2K\log\left(\frac{1}{\delta}\right)}.
\end{eqnarray*}
\end{proposition}
The proof of Proposition \ref{thm:concentration on size of J0} is deferred to Appendix \ref{app:proof-5-10-prop}.

Next, we present a concentration result for value functions associated with state processes that exit the ball of radius $\rho$.


\begin{theorem}\label{thm:concentration on value function at initial state}
Suppose Assumptions \ref{ass:lipschitz}, \ref{ass:expected reward local lipschitz} and \ref{ass:initial} hold. For any policy $\pi$, we have the following holds with probability at least $1-\delta$,
\begin{eqnarray}\label{eq:concentration on value function at initial state}
   \sum_{k \in [K]\backslash J_{\rho}^{K}}|V_{1}^{\pi}(X_{1}^{k})|\leq K\kappa_{m+1}(\delta,\rho)+\widetilde{C}_1\Big(1+\rho^{m+1}\Big)(K-K_{0}),
\end{eqnarray}
where $\kappa_{m+1}:(0,1]\times (\mathbb{R}_+\cup\{0\}) \mapsto \mathbb{R}_+$ is defined as
\begin{eqnarray}\label{eq:kappa1 and kappa m+1}
    \kappa_{m+1}(\delta,y):=\frac{1}{\delta}\widetilde{C}_1\Big(\frac{\mathbb{E}_{\xi \sim \Xi}[\|\xi\|^{p}]}{y^p}+ \frac{\mathbb{E}_{\xi \sim \Xi}[\|\xi\|^{p}]}{y^{p-(m+1)}}\Big),
\end{eqnarray} 
and $\widetilde{C}_{1}$ is defined in \eqref{eq:value function with any policy growth rate}.
\end{theorem}

The proof of Theorem \ref{thm:concentration on value function at initial state} is deferred to Appendix \ref{app:proof-5-7}.

\subsection{Regret composition}
In this subsection, we provide the regret analysis of Algorithm \ref{alg:ML}. In Theorem \ref{thm:first stage regret decomposition}, we bound the regret by separating two types of episodes. 

\begin{theorem}\label{thm:first stage regret decomposition}

 Assume Assumptions \ref{ass:lipschitz}-\ref{ass:initial} hold. With probability at least $1-6\delta$, we have:

\begin{eqnarray}\label{eq:regret_second}
{\rm Regret}(K)
&\leq&
e^{2}\sum_{h=1}^{H}\sum_{k\in J_{1}}
{\rm CLIP}\Bigg(
G_{h}^{k}(B_h^k)
\Bigg|
\frac{{\rm Gap}_{h}(B_h^k)}{H+1}
\Bigg)
\\
&&+
2e^{2}\sqrt{
\widetilde{L}_{1}HK
\left(
\left(\frac{M_{p}K}{\delta}\right)^{\frac{2m+2}{p}}
+1
\right)
\log\left(\frac{2}{\delta}\right)}
+
2K\kappa_{m+1}(\delta,\rho)
\nonumber\\
&&+4\widetilde{C}_{1}
\left(
\widetilde{L}_{3}+\rho^{m+1}
+e^{2}\widetilde{L}_{2}H
\left(\frac{M_{p}K}{\delta}\right)^{\frac{m+1}{p}}
\right)
\left(
\frac{M_{p}K}{\rho^{p}}
+\sqrt{2K\log\left(\frac{1}{\delta}\right)}
\right),\nonumber
\end{eqnarray}

where $\widetilde{L}_{1}, \widetilde{L}_{2}$ depends only on $m, D, d_{\mathcal{S}},\widetilde{C}_{\max}, C_{\max} $ and $\widetilde{L}_{3}:=1+e^{2}\widetilde{L}_{2}H$
with 
\begin{eqnarray}\label{eq:widetilde Cmax,Cmax}
\widetilde{C}_{\max}:=\max_{h\in[H]}\widetilde{C}_{h},\quad C_{\max}:=\max_{h\in[H]}C_{h}.
\end{eqnarray}
Here $\widetilde{C}_{h}$ is defined in \eqref{eq:value function with any policy growth rate}, $C_{h}$  in \eqref{eq:local lipschitz constant}, $\widetilde{C}$  in \eqref{eq:C(q,d)}, $\eta$  in \eqref{eq:eta function}, ${\rm Gap}_h$  in \eqref{eq:gapB definition}, $G_{h}^{k}$  in \eqref{eq:G_h^k definition}, and $\kappa_{m+1}$  in \eqref{eq:kappa1 and kappa m+1}. In addition, $B_h^k$ is selected according to Algorithm \ref{alg:ML 2}.
\end{theorem}
\begin{proof}

Combine Theorem \ref{thm: upper bound via clipping} and Lemma \ref{thm:induction}, with probability at least $1-3\delta$, we have:

\begin{eqnarray}\label{bound 4}
\sum_{k\in J_{\rho}^{K}}\Delta_{1}^{(k)}
&\leq&
\sum_{k\in J_{\rho}^{K}}
{\rm CLIP}\Bigg(
G_{1}^{k}(B_1^k)
\Bigg|
\frac{{\rm Gap}_{1}(B_1^k)}{H+1}
\Bigg)
\\
&&+
\left(1+\frac{1}{H}\right)
\left(
\sum_{k\in J_{\rho}^{K}}\Delta_{2}^{(k)}
+\sum_{k\in J_{\rho}^{K}}\xi_{2}^{k}
\right)
+\sum_{k\in J_{\rho}^{K}}\zeta_{2}^{k}
\nonumber\\
&\leq&
\sum_{h=1}^{H}\sum_{k\in J_{\rho}^{K}}
\left(1+\frac{1}{H}\right)^{2(h-1)}{\rm CLIP}\Bigg(
G_{h}^{k}(B_h^k)
\Bigg|
\frac{{\rm Gap}_{h}(B_h^k)}{H+1}
\Bigg)
\nonumber\\
&&+
\sum_{h=1}^{H}\sum_{k\in J_{\rho}^{K}}
\left(1+\frac{1}{H}\right)^{2h}\xi_{h+1}^{k}
+\sum_{h=1}^{H}\sum_{k\in J_{\rho}^{K}}\zeta_{h+1}^{k}.
\nonumber
\end{eqnarray}

We claim the following results in \eqref{eq:xi_1}-\eqref{eq:zeta_2} hold with probability at least $1-5\delta$, for which the proof is deferred to  Appendix \ref{app:proof-5-9}.

\begin{eqnarray}
&&\sum_{h=1}^{H}\sum_{k=1}^{K}\xi_{h+1}^{k}\leq 2e^{2}\sqrt{\widetilde{L}_{1}HK\Big(\Big(\frac{M_{p}K}{\delta}\Big)^{\frac{2m+2}{p}} +1\Big)\log\Big(\frac{2}{\delta}\Big)} , \label{eq:xi_1}\\
    && \Big|\sum_{h=1}^{H}\sum_{k\in [K]\backslash J_{\rho}^{K}}^{K}\xi_{h+1}^{k}\Big|\leq e^{2}\widetilde{L}_{2}H\Big(\frac{KM_{p}}{\rho^{p}}+\sqrt{2K\log(\frac{1}{\delta}})\Big)\Big(\Big(\frac{M_{p}K}{\delta}\Big)^{\frac{m+1}{p}} +1\Big),\label{eq:xi_2}\\  &&\sum_{h=1}^{H}\sum_{k=1}^{K}\zeta_{h+1}^{k}\leq 2e^{2}\sqrt{\widetilde{L}_{1}HK\Big(\Big(\frac{M_{p}K}{\delta}\Big)^{\frac{2m+2}{p}} +1\Big)\log\Big(\frac{2}{\delta}\Big)},\label{eq:zeta_1} \\
   && \Big|\sum_{h=1}^{H}\sum_{k\in [K]\backslash J_{\rho}^{K}}^{K}\zeta_{h+1}^{k}\Big|\leq e^{2}\widetilde{L}_{2}H\Big(\frac{KM_{p}}{\rho^{p}}+\sqrt{2K\log(\frac{1}{\delta}})\Big)\Big(\Big(\frac{M_{p}K}{\delta}\Big)^{\frac{m+1}{p}} +1\Big), \label{eq:zeta_2}
\end{eqnarray}

where $\widetilde{L}_{1}, \widetilde{L}_{2}$ depends only on $m, D, d_{\mathcal{S}},\widetilde{C}_{\max}, C_{\max}$. 


Then decompose the regret, we have the following holds with probability at least $1-2\delta$:

\begin{eqnarray}\label{eq:regret_decompose}
{\rm Regret}(K)
&\leq&
\sum_{k\in J_{\rho}^{K}}
\left(\overline{V}_{1}^{k-1}(X_{1}^{k})
-V_{1}^{\tilde{\pi}^{k}}(X_{1}^{k})\right)
+\sum_{k\in J\backslash J_{\rho}^{K}}
\left(|V_{1}^{*}(X_1^k)|
+|V_{1}^{\tilde{\pi}^{k}}(X_1^k)|\right)
\\
&\leq&
\sum_{k\in J_{\rho}^{K}}\Delta_{1}^{(k)}
+2\left(
K\kappa_{m+1}(\delta,\rho)
+\widetilde{C}_{1}
(1+\rho^{m+1})(K-K_{0}),
\right)\nonumber
\end{eqnarray}
where $J=[K]$ and $J_{\rho}^{K}$ is defined in \eqref{eq:J0}.
The first inequality holds due to Theorem \ref{thm: True Upper Bounded By Estimates}. The second inequality holds due to Theorem \ref{thm:concentration on value function at initial state}. 

By combining Theorems \ref{thm: upper bound via clipping}, \ref{thm:concentration on size of J0}, \eqref{eq:xi_1}-\eqref{eq:zeta_2},   the fact that $\Big(1+\frac{1}{H}\Big)^{H}\leq e$ and \eqref{eq:regret_decompose}, we get that with probability at least $1-6\delta$, we verify \eqref{eq:regret_second}.

The last inequality holds due to \eqref{bound 4}-\eqref{eq:zeta_2}.
\end{proof}

Unlike in the bounded-state setting of \cite{sinclair2023adaptive}, the
martingale difference $\xi_{h+1}^k$ is unbounded. The proof of
Theorem~\ref{thm:first stage regret decomposition} therefore uses the more
general martingale concentration bounds
\eqref{eq:xi_1}--\eqref{eq:zeta_2} instead of the Azuma--Hoeffding inequality.

Before deriving the regret bound, we introduce near-optimal sets and zooming
dimensions, standard tools in contextual-bandit regret analysis
\cite{sinclair2023adaptive}. We modify these notions to accommodate the
unbounded state space and polynomial reward growth in our setting.

\begin{definition}[Near optimal set]\label{def:near optimal set}
 The  near optimal set of $\bar{Z}$ for a given value $r$ is defined as 
\begin{eqnarray}\label{eq:def near optimal set}
      Z_{h}^{r,\rho} = \Big\{(x,a) \in \bar{Z}\,\Big|\, {\rm \widetilde{G}ap}_{h}(x,a) \leq \bar{g}(\delta,x)(H+1)r\Big\},
\end{eqnarray}
where the partition space $\overline{Z}$ is defined in \eqref{eq:partition_space} and $\bar{g}:(0,1]\times \mathbb{R}^{d_{\mathcal{S}}}\mapsto \mathbb{R}_+$ is defined as

\begin{eqnarray} \label{eq: l(x) def}
    \bar{g}(\delta, x):&=& 2g_{3}(\delta,\|x\|+D)+3\overline{C}_{\max}\Big(1+2(\|x\|+2D)^m\Big)
    \\&&+2\frac{\widetilde{C}_{\max}}{D}\Big(1+(\|x\|+D)^{m+1}\Big)\nonumber.
\end{eqnarray}

Also, $g_{3}:(0,1]\times (\mathbb{R}_+\cup\{0\}) \mapsto \mathbb{R}_+$ is defined as
\begin{eqnarray}\label{eq:g3 definition}
    g_{3}(\delta,y):=2\frac{\widehat{C}_{\max}}{\overline{C}_{\max}}+2\frac{\widehat{C}_{\max}}{\overline{C}_{\max}}g_{2}(\delta,y)+C_{\max}\Big(1+2(y+D)^m\Big),
\end{eqnarray} 
where $g_{2}$ is defined in \eqref{eq:g2 definition}, $\overline{C}_{\max}$ is defined in \eqref{eq:barCmax}, $\widetilde{C}_{\max}$ and $C_{\max}$ are defined in \eqref{eq:widetilde Cmax,Cmax}, and $\widehat{C}_{\max}$ is defined below \eqref{eq:G_h^k definition}.
\end{definition}  

While $\widetilde{\mathrm{Gap}}_h(x,a)$ measures action suboptimality,
$\bar g(\delta,x)$ captures polynomial growth and local learning difficulty.
We bound regret by the size of the resulting near-optimal sets, whose
intrinsic dimension can be much smaller than
$d_{\mathcal S}+d_{\mathcal A}$; see \cite{sinclair2023adaptive} for
examples.

To quantify the size of near-optimal sets, we introduce the concepts of packing, packing numbers, and zooming dimension.

\begin{definition}[$r$-packing and $r$-packing number; Definition 4.2.4 in \cite{vershynin2018high}]\,\,
\begin{itemize}
    \item For a given $r>0$ and a compact set $\mathcal{U}$, an r-packing ${\rm P}_{\mathcal{U}}^{r} \subset \mathcal{U}$ is a set such that $\|x-x^{\prime}\|>r$ for any two distinct $x,x^{\prime} \in {\rm P}_{\mathcal{U}}^{r}$. 
\item We define the $r$-packing number of $\mathcal{U}$, denoted $N_{r}(\mathcal{U})$, as the maximum cardinality among all $r$-packings of $\mathcal{U}$.
\end{itemize}
\end{definition}

\begin{definition}[Zooming dimension and maximum zooming dimension]\label{def:zooming dimension}
The  step-$h$ zooming dimension $z_{h,c}$ with a given positive constant $c$ is defined as
\begin{eqnarray*}
    z_{h,c}=\inf\left\{d>0 \,:\, \frac{N_{r}(Z_{h}^{r,\rho})}{\rho^{d_{\mathcal{S}}}}\leq c\,r^{-d}, \forall 0<r\leq D , \forall \rho> D \right\}. 
\end{eqnarray*}
The maximum zooming dimension $z_{max,c}$ is defined as 
\begin{eqnarray*}
    z_{\max,c}=\max_{h \in [H]}z_{h,c}.
\end{eqnarray*}
\end{definition}

We adapt the zooming dimension of \cite{sinclair2023adaptive} to unbounded
state spaces so that it is independent of $\rho$. This allows potentially
sharper regret bounds based on the zooming dimension rather than the ambient
dimension $d_{\mathcal S}+d_{\mathcal A}$.\footnote{In Appendix \ref{app:choice of c}, we show that if we take $c\geq \frac{2^{d_{\mathcal{S}}}\Gamma(\frac{d_{\mathcal{S}}+d_{\mathcal{A}}}{2}+1)\bar{a}^{d_{\mathcal{A}}}}{\Gamma(
\frac{d_{\mathcal{S}}}{2}+1)\Gamma(
\frac{d_{\mathcal{A}}}{2}+1)}$, then the zooming dimension does not exceed ambient dimension $d_{\mathcal{S}}+d_{\mathcal{A}}$.}

Theorem~\ref{thm:first stage regret decomposition} reduces regret to
accumulated clipped bonuses over episodes whose trajectories remain in the
localized state region. Since clipping retains only blocks with near-optimal
centers, the next lemma bounds this sum using the packing numbers of
$\mathcal Z_h^{r,\rho}$.

\begin{lemma}[Modification of Theorem F.3 in \cite{sinclair2023adaptive}]\label{lemma:theorem F.3 conclusion}
    Assume Assumptions \ref{ass:lipschitz}-\ref{ass:expected reward local lipschitz} hold.Then for any given constant $r_{0}>0$ we have the following: 
    
    \begin{eqnarray}\label{eq:CLIP BOUND}
        &&\sum_{h}\sum_{k\in J_{\rho}^{K}}{\rm CLIP}\Bigg(G_h^k(B_h^k)\,\Bigg|\,\frac{{\rm Gap}_{h}(B_h^k)}{H+1}\Bigg) \nonumber\\
       &\leq& \sum_{h=1}^{H}\,\,\left(2g_{3}(\delta,\rho+D)Kr_{0}+g_{4}(\delta,\rho+D)\sum_{r \geq r_{0}, r\in \mathcal{R}}N_{r}(Z_{h}^{r,\rho})\frac{1}{r}\right),
    \end{eqnarray}
where $\mathcal{R}:=\{r\,\,|\,\,\exists h \in [H], k \in J_{\rho}^{K}, {\rm diam}(B_h^k)=r\}$. Here $J_{\rho}^{K}$ is defined in \eqref{eq:J0}, $\overline{C}_{\max}$  in \eqref{eq:barCmax}, ${\rm Gap}_h$  in \eqref{eq:gapB definition}, $G_{h}^{k}$  in \eqref{eq:G_h^k definition}, $g_1$ in \eqref{eq:g1 definition}, $g_{3}$ in \eqref{eq:g3 definition}, and $\bar{g}$  in \eqref{eq: l(x) def}. In addition, $g_{4}:(0,1]\times (\mathbb{R}_+\cup\{0\})\mapsto \mathbb{R}_+ $ is defined  as 
    \begin{eqnarray}\label{eq:g4 definition}
   g_{4}(\delta,y):=g_{3}(\delta,y)g_{1}(\delta,y)^2+\frac{(2\bar{a})^{d_{\mathcal{A}}}}{D^{d_{\mathcal{S}}+d_{\mathcal{A}}-2}}(d_{\mathcal{S}}+d_{\mathcal{A}})^{\frac{d_{\mathcal{S}}+d_{\mathcal{A}}}{2}}y^{d_{\mathcal{S}}}\bar{g}(\delta, y). 
    \end{eqnarray}
    \end{lemma}
This result adapts Theorem~F.3 of \cite{sinclair2023adaptive} to our
setting with an unbounded state space and polynomially growing rewards. Its proof is deferred to Appendix~\ref{app:proof-lemma-5.13}. We emphasize
that the upper bound in \eqref{eq:CLIP BOUND} holds for any choice of
\(r_0>0\); this resolution parameter will be optimized in the final regret
analysis in Theorem~\ref{thm:final high prob regret bound}.

Finally, we are ready to provide the regret bound.

\begin{theorem}\label{thm:final coupled for worst case analysis}
Assume Assumptions \ref{ass:lipschitz}-\ref{ass:initial} hold.  With probability at least $1-6\delta$, we have: 
\begin{eqnarray}
      {\rm Regret}(K)&\leq& e^{2} \sum_{h=1}^{H}\,\,\left(2g_{3}(\delta,\rho+D)Kr_{0}+g_{4}(\delta,\rho+D)\sum_{r \geq r_{0}, r\in \mathcal{R}}N_{r}(Z_{h}^{r,\rho})\frac{1}{r}\right)\nonumber\\
        &&+2e^{2}\sqrt{\widetilde{L}_{1}HK\Big(\Big(\frac{M_{p}K}{\delta}\Big)^{\frac{2m+2}{p}} +1\Big)\log\Big(\frac{2}{\delta}\Big)}+2K\kappa_{m+1}(\delta,\rho)\nonumber\\
&&+4\widetilde{C}_1\Big(\widetilde{L}_{3}+\rho^{m+1}+e^{2}\widetilde{L}_{2}H\Big(\frac{M_{p}K}{\delta}\Big)^{\frac{m+1}{p}} \Big)\left(\frac{M_{p}}{\rho^{p}}K+\sqrt{2K\log\Big(\frac{1}{\delta}\Big)}\right)\label{eq:final regret before order balance},
\end{eqnarray}
where $g_{3}$ is defined in \eqref{eq:g3 definition}, $g_{4}$ is defined in \eqref{eq:g4 definition},  $\widetilde{C}_1$ is defined in \eqref{eq:value function with any policy growth rate}, and $\kappa_{m+1}$ is defined in \eqref{eq:kappa1 and kappa m+1}.
\end{theorem}
The proof of Theorem~\ref{thm:final coupled for worst case analysis} is deferred to Appendix~\ref{app:final coupled for worst case analysis}. The bound separates the regret into two parts: the contribution from near-optimal regions, captured by the packing numbers \(N_r(Z_h^{r,\rho})\), and the additional errors caused by localization and unbounded martingale fluctuations. At this stage, the result should be understood as a pre-optimization bound, since its final order depends on the choice of the localization radius
\(\rho\) and the resolution threshold \(r_0\). We next optimize these parameters by balancing the competing terms in the bound.

\begin{theorem} \label{thm:final high prob regret bound}
 Take the same assumptions in Theorem \ref{thm:final coupled for worst case analysis} and set $\rho=M_{p}^{\frac{1}{p}}K^{\beta}$, $r_{0}=K^{\gamma}$. The optimal regret order on $K$ in \eqref{eq:final regret before order balance} is achieved as
 \[
 1-\frac{p^{2}-(m+1)^{2}(z_{\max,c}+2)-(m+1)(2d_{\mathcal{S}}+2m+4)}{p(p+m+1)(z_{\max,c}+2)+p(2d_{\mathcal{S}}+2m+4)}
 \]
 if we
take 
\begin{eqnarray*}
    \beta=\frac{p+(m+1)(z_{\max,c}+2)}{p(p+m+1)(z_{\max,c}+2)+p(2d_{\mathcal{S}}+2m+4)},\quad \gamma=\frac{(2d_{\mathcal{S}}+2m+4)\beta_{2}-1}{z_{\max,c}+2}.
\end{eqnarray*}
 Then with probability at least $1-6\delta$, the following optimal regret bound holds that:
\begin{eqnarray}\label{eq:final regret with leading term shown}
{\rm Regret}(K) \lesssim \Gamma HK^{1-\frac{p^{2}-(m+1)^{2}(z_{\max,c}+2)-(m+1)(2d_{\mathcal{S}}+2m+4)}{p(p+m+1)(z_{\max,c}+2)+p(2d_{\mathcal{S}}+2m+4)}}{{\Bigg(\log\left(\frac{2HK^{2}}{\delta}\right)\Bigg)^{\frac{3m}{2}+2}}},
\end{eqnarray}
where $\lesssim$ omits constants that  are independent of $H,K$ and lower order terms; $\Gamma$ is a constant depends on $M_{p},\widetilde{L}_{1},C_{\max}, \overline{C}_{\max},\widetilde{C}_{1}$. 
\end{theorem}
The choices of \(\rho\) and \(r_0\) in Theorem \ref{thm:final high prob regret bound} balance the adaptive-partition error, the localization error, and the moment-dependent concentration terms. The resulting regret is sublinear in \(K\), with leading order
\[
\widetilde O\!\left(
H K^{\,1-\frac{p^2-(m+1)^2(z_{\max,c}+2)-(m+1)(2d_S+2m+4)}
{p(p+m+1)(z_{\max,c}+2)+p(2d_S+2m+4)}}
\right).
\]
This order shows explicitly how heavier-tailed initial distributions and faster reward growth slow learning, while as \(p\to\infty\) the exponent approaches \((z_{\max,c}+1)/(z_{\max,c}+2)\), recovering the bounded-state benchmark in terms of the episode number \(K\). See more discussion in Remark \ref{re:comparison}.

\begin{proof}
To optimize \eqref{eq:final regret before order balance} in $K$, use $N_r(Z_h^{r,\rho})/\rho^{d_{\mathcal S}}\le cr^{-z_{h,c}}$ and set $\rho=M_p^{1/p}K^\beta$, $r_0=K^\gamma$. Then it is sufficient to solve the minimization problem of the following objective function $U(\beta,\gamma)$.

\begin{eqnarray}\label{eq:U}
    U(\beta,\gamma)&:=&\max\Big\{1+\gamma+(m+1)\beta, (2d_{\mathcal{S}}+3m+5)\beta-\gamma(z_{\max,c}+1), \\
 &&\qquad\quad \frac{1}{2}+\frac{m+1}{p}, \frac{1}{2}+(m+1)\beta, 1-(p-(m+1))\beta, 1+\frac{m+1}{p}-p\beta \Big\}.\nonumber
\end{eqnarray}

We analyze the problem under two regimes: {\bf (1)} $\beta \geq \frac{1}{p}$ and  {\bf (2)} $\beta < \frac{1}{p}$. 

\noindent \underline{Case {\bf (1)}}: In this regime, we can simplify \eqref{eq:U} as
$U(\beta,\gamma)=\max\Big\{1+\gamma+(m+1)\beta,(2d_{\mathcal{S}}+3m+5)\beta-\gamma(z_{\max,c}+1), \frac{1}{2}+(m+1)\beta\Big\}$. Clearly, over this region, the minimizer $(\beta_{1},\gamma_{1})$ satisfies the following equation:
\begin{eqnarray*}
    1+\gamma_{1}+(m+1)\beta_{1}=(2d_{\mathcal{S}}+3m+5)\beta_{1}-\gamma_{1}(z_{\max,c}+1), \,\,\,\,\textrm{ and }\,\,\,\,
    \beta_{1}=\frac{1}{p}.
\end{eqnarray*}
By straightforward calculations, we get $\gamma_{1}=\frac{2d_{\mathcal{S}}+2m+4-p}{p(z_{\max,c}+2)}$ and $U(\beta_{1},\gamma_{1})=1-\frac{(p-(m+1)(z_{\max,c}+4)-2d_{\mathcal{S}}-2)}{{p(z_{\max,c}+2)}}$.

\noindent \underline{Case {\bf (2)}}:
In this regime, we can simplify \eqref{eq:U} as
$U(\beta,\gamma)=\max\Big\{1+\gamma+(m+1)\beta,(2d_{\mathcal{S}}+3m+5)\beta-\gamma(z_{\max,c}+1), \frac{1}{2}+\frac{m+1}{p},1+\frac{m+1}{p}-p\beta\Big\}$. Then the minimum of $U(\cdot,\cdot)$ shall be $U(\beta_{2},\gamma_{2})$ where $(\beta_{2},\gamma_{2})$ satisfies: 
\begin{eqnarray*}
        1+\gamma_{2}+(m+1)\beta_{2}&=&(2d_{\mathcal{S}}+3m+5)\beta_{2}-\gamma_{2}(z_{\max,c}+1),\\
        1+\frac{m+1}{p}-p\beta_{2}&=&1+\gamma_{2}+(m+1)\beta_{2}.
\end{eqnarray*}
By straightforward calculations, we get 
$\beta_{2}=\frac{p+(m+1)(z_{\max,c}+2)}{p(p+m+1)(z_{\max,c}+2)+p(2d_{\mathcal{S}}+2m+4)}$, $\gamma_{2}=\frac{(2d_{\mathcal{S}}+2m+4)\beta_{2}-1}{z_{\max,c}+2}$ and $U(\beta_{2},\gamma_{2})=1-\frac{p^{2}-(m+1)^{2}(z_{\max,c}+2)-(m+1)(2d_{\mathcal{S}}+2m+4)}{p(p+m+1)(z_{\max,c}+2)+p(2d_{\mathcal{S}}+2m+4)}$. 

In addition, we can show that $U(\beta_{1},\gamma_{1})>U(\beta_{2},\gamma_{2})$. 


Therefore, the optimal leading order on $K$ is achieved at
\[
1-\frac{p^{2}-(m+1)^{2}(z_{\max,c}+2)-(m+1)(2d_{\mathcal{S}}+2m+4)}{p(p+m+1)(z_{\max,c}+2)+p(2d_{\mathcal{S}}+2m+4)}
\]
if we take
\[
\beta=\frac{p+(m+1)(z_{\max,c}+2)}{p(p+m+1)(z_{\max,c}+2)+p(2d_{\mathcal{S}}+2m+4)},
\qquad
\gamma=\frac{(2d_{\mathcal{S}}+2m+4)\beta_{2}-1}{z_{\max,c}+2}.
\]
Combined with \eqref{eq:final regret before order balance}, we can verify that \eqref{eq:final regret with leading term shown} holds with probability at least $1-6\delta$. 
\end{proof}

\begin{remark}[Comparison of our regret to the literature]\label{re:comparison}

Note that  $1-\frac{p^{2}-(m+1)^{2}(z_{\max,c}+2)-(m+1)(2d_{\mathcal{S}}+2m+4)}{p(p+m+1)(z_{\max,c}+2)+p(2d_{\mathcal{S}}+2m+4)} \rightarrow\frac{z_{\max,c}+1}{z_{\max,c}+2}$ as $p$ tends to infinity.  This suggests that if the initial distribution has all moments bounded, we recover the regret of the AdaMB algorithm proposed in \cite{sinclair2023adaptive} for bounded state space in terms of the episode number $K$.

A detailed comparison between our algorithms and those proposed in \cite{sinclair2023adaptive} is presented in Table \ref{table:regret}, where $z'_{\max,c}$  is defined in Definition 2.7 of \cite{sinclair2023adaptive}.\footnote{
Following the usual convention in \cite{domingues2021kernel,sinclair2023adaptive}, we suppress the dependence of the Lipschitz constants on the horizon $H$, and hence also the dependence of $C_{\max}$, $\overline{C}_{\max}$, $M_p$, $\widetilde{C}_1$, and $\widetilde{L}_1$ on $H$ in Theorem \ref{thm:final high prob regret bound}. In the bounded-reward and bounded-state-space setting, such dependence can be removed by appropriately rescaling the system; see Lemma 2.4 in \cite{sinclair2023adaptive}. Extending this rescaling argument to our framework with unbounded state spaces and reward functions, however, appears to be more difficult.

Subject to this convention, the explicit dependence of our regret bound on $H$ is linear, as obtained from the Lipschitz-type properties of the value functions. By contrast, the analysis in \cite{sinclair2023adaptive} yields a higher-order explicit dependence on $H$, because it uses the fact that the cumulative reward over $H$ time steps is bounded by $H$. Nevertheless, since both analyses suppress the implicit dependence of the Lipschitz constants on $H$, a direct comparison based solely on the displayed order in $H$ may not be fully accurate. We therefore do not place particular emphasis on this comparison.}

\begin{table}[h!]
\centering
\resizebox{\textwidth}{!}{%
\begin{tabular}{|c| c| c| c| c|} 
 \hline
\rm{AdaMB }  & \rm{AdaQL } & \rm{APL-Diffusion } & \rm{APL-Diffusion  }  \\ 
\cite{sinclair2023adaptive} &\cite{sinclair2023adaptive}   &(ours)  & (ours) $(p\mapsto \infty)$   \\ 
[0.5ex] 
 \hline
$H^{\frac{3}{2}}K^{\frac{{z'_{\max,c}}+\max\{d_{\mathcal{S}},2\}-1}{{z'_{\max,c}}+\max\{d_{\mathcal{S}},2\}}}$   & $H^{\frac{5}{2}}K^{\frac{{z'_{\max,c}}+1}{{z'_{\max,c}}+2}}$  & $HK^{1-\frac{p^{2}-(m+1)^{2}(z_{\max,c}+2)-(m+1)(2d_{\mathcal{S}}+2m+4)}{p(p+m+1)(z_{\max,c}+2)+p(2d_{\mathcal{S}}+2m+4)}}$ & $HK^{\frac{z_{\max,c}+1}{z_{\max,c}+2}}$\\ [1ex] 
 \hline
\end{tabular}
}
\caption{Comparison of the regret orders.}
\label{table:regret}
\end{table}
\end{remark}

\begingroup
\setlength{\columnsep}{14pt}
\ifnum\mainonlyflag=1\relax
\nocite{\originalreferencekeys}
\fi
\bibliographystyle{siamplain}
\bibliography{references}
\endgroup

\ifnum\mainonlyflag=1\relax
\end{document}
\fi
\fi

\ifnum\supplementonlyflag=1\relax
\else
\newpage
\fi
\appendix
\begin{center}
{\huge \bf Appendix}
\end{center}
\vspace{10pt}
\section{Technical details in Section \ref{sec:set-up}}

\subsection{Proof of Proposition \ref{thm:Mp estimation}}\label{app:proof-2-2}

\begin{proof}
We first prove $\mathbb{E}[\|X_2\|^p]<\breve{c}_1(1+\mathbb{E}[\|X_1\|^p])$ for some constant $\breve{c}_1$. 
By the dynamics of state process, we have 
\begin{eqnarray*}
    \|X_2\|&\le& \|X_1\|+\|\mu_1(X_1, A_1)\|\Delta+\|\sigma_1(X_1, A_1)\|\|B_1\|\sqrt{\Delta}\\
    &\le& \|X_1\|+(L_{0}+\ell_{\mu}(\|X_1\|+\bar{a}))\Delta+(L_{0}+\ell_{\sigma}(\|X_1\|+\bar{a}))\|B_1\|\sqrt{\Delta}\\
    &=&(1+\ell_{\mu}\Delta+\ell_{\sigma}\|B_1\|\sqrt{\Delta})\|X_1\|+(\ell_{\mu}\bar{a}+L_{0})\Delta+(\ell_{\sigma}\bar{a}+L_{0})\|B_1\|\sqrt{\Delta}.
\end{eqnarray*}
So for any $p\ge 1$, there exists a constant $\breve{c}_3$ depending on $p$ only, such that 

\begin{eqnarray*}
    \|X_2\|^p&\le& \breve{c}_3\Big((1+\ell_{\mu}\Delta+\ell_{\sigma}\|B_1\|\sqrt{\Delta})^p\|X_1\|^p\\
    &&\quad+((\ell_{\mu}\bar{a}+L_{0})\Delta+(\ell_{\sigma}\bar{a}+L_{0})\|B_1\|\sqrt{\Delta})^p\Big)\\
    &\le& \breve{c}_3\Big(\breve{c}_3(1+\ell_{\mu}^p\Delta^p+\ell_{\sigma}^p\|B_1\|^p\Delta^{p/2})\|X_1\|^p\\
    &&\quad+\breve{c}_3\big((\ell_{\mu}\bar{a}+L_{0})^p\Delta^p
    +(\ell_{\sigma}\bar{a}+L_{0})^p\|B_1\|^p\Delta^{p/2}\big)\Big).
\end{eqnarray*}

Together with the fact that $B_1$ is independent with $X_1$, we have 
\begin{eqnarray}\label{eq:moment estimation for the next step}
\mathbb{E}[\|X_2\|^p]&\le& 
\breve{c}_3^2\Big((1+\ell_{\mu}^p\Delta^p+\ell_{\sigma}^p\mathbb{E}[\|B_1\|^p]\Delta^{p/2} )\mathbb{E}[\|X_1\|^p] \nonumber\\&&+(\ell_{\mu}\bar{a}+L_{0})^p\Delta^p+(\ell_{\sigma}\bar{a}+L_{0})^p\mathbb{E}[\|B_1\|^p]\Delta^\frac{p}{2}\Big)\nonumber\\
&\le& \breve{c}_4(1+\mathbb{E}[\|X_1\|^p]) ,
\end{eqnarray}
for some constant $\breve{c}_4$ depending only on $p, \Delta, \ell_{\mu}, \ell_{\sigma},\bar{a},L_{0}$.

By the same argument, we have 
$\mathbb{E}[\|X_3\|^p]\le \breve{c}_4(1+\mathbb{E}[\|X_2\|^p)\le \breve{c}_5(1+\mathbb{E}[\|X_1\|^p])$
for some $\breve{c}_5$ depending only on $\breve{c}_4,$
as well as $\mathbb{E}[\|X_{h}\|^p]\le \breve{c}_{h+2}(1+\mathbb{E}[\|X_1\|^p])$ for some $\breve{c}_{h+3}$ depending only on $p, \Delta, \ell_{\mu}, \ell_{\sigma},\bar{a},L_{0}$.

Finally, 
$$\mathbb{E}\left[\sup_{h\in [H]}\|X_h\|^p\right]<\sum_{h\in[H]}\mathbb{E}[\|X_h\|^p]\le M(1+\mathbb{E}[\|X_1\|^p]),$$
where $M$ only depends on $p, \Delta, \ell_{\mu}, \ell_{\sigma},\bar{a},L_{0}$ and $H$.

\end{proof}

\subsection{Proof of Proposition \ref{thm:value function local Lipschitz}}\label{app:proof-2-5}

\begin{proof} 
In this proof, we will often use the fact that for any functions $f$ and $g$ on the same domain, we have
$$|\max_x f(x)-\max_y g(y)|\le \max_x |f(x)-g(x)|;$$ 
and for any nonnegative real numbers $a,b,c$, any integer $m>0$, 
$$(a+b+c)^m\le k_3(m)(a^m+b^m+c^m)$$ 
for some function $k_3(\cdot)$.

We prove the statement by backward induction. For the last step $h=H$, we have
    \begin{eqnarray*}
        \left|V^*_{H}(x_1)-V_{H}^{*}(x_2)\right|&=&\left|\max_{a\in A} \bar{R}_{H}(x_1,a)-\max_{b\in A}\bar{R}_{H}(x_2,b)\right|\\
        &\le&\max_{a\in A}|\bar{R}_H(x_1,a)-\bar{R}_H(x_2,a)|\\
        &\le& \ell_{r} \Big(1+\|x_1\|^m+\|x_2\|^m\Big)\,\Big(\|x_1-x_2\|\Big).
    \end{eqnarray*}
Let 
\begin{eqnarray}\label{eq:CH value}
 \overline{C}_{H}:=\ell_{r},
\end{eqnarray}
with $\ell_{r}$ defined in \eqref{ass:expected reward local lipschitz}.

Now suppose the inequality \eqref{eq:value function local Lipschitz} holds for $h=j>0$. We study the inequality for $h=j-1$. For any state $x\in\mathbb{R}^{d_{\mathcal{S}}}$ and any action $a\in \mathcal{A}$, denote by $X^{(x, a)}:=x+\mu_{j-1}(x,a)\Delta +\sigma_{j-1}(x,a)B_{j-1} \sqrt{\Delta}$. 

From \eqref{eq:bellman V star}, we know  that
$V^*_{j-1}(x)=\max_{a\in A}\{\bar{R}_{j-1}(x, a)+\mathbb{E}[V_j^*(X^{(x,a)})]\}$. Hence,
\begin{eqnarray}
    &&|V^*_{j-1}(x)-V^*_{j-1}(y)|\nonumber\\
    &\le& \max_{a\in A}|\bar{R}_{j-1}(x,a)+\mathbb{E}[V^*_j(X^{(x,a)})]-\bar{R}_{j-1}(y,a)-\mathbb{E}[V^*_{j}(X^{(y,a)})]|\nonumber\\
    &\le&\max_{a\in A}|\bar{R}_{j-1}(x,a)-\bar{R}_{j-1}(y,a)|+\max_{a\in A} \mathbb{E}[|V^*_j(X^{(x,a)})-V^*_{j}(X^{(y,a)})|]. \label{ieq:difV*} 
\end{eqnarray}

The first term in \eqref{ieq:difV*} is bounded by 
$\ell_{r}(1+\|x\|^m+\|y\|^m)\|x-y\|$, so it suffices to estimate the second term in \eqref{ieq:difV*}.

By the induction hypothesis, we have 
\begin{eqnarray}\label{ieq:difV*point}
    |V^*_j(X^{(x,a)})-V^*_{j}(X^{(y,a)})|
\le \overline{C}_j (1+\|X^{(x,a)}\|^m+\|X^{(y,a)}\|^m)\|X^{(x,a)}-X^{(y,a)}\|.
\end{eqnarray}

Note that 

\begin{eqnarray}
\|X^{(x,a)}\|^m 
    &\le& \Big(\|x\|+(L_{0}+\ell_{\mu}(\|x\| +\bar{a}))\Delta\nonumber\\
    &&\quad +(L_{0}+\ell_{\sigma}(\|x\|+\bar{a}))\|B_{j-1}\|\sqrt{\Delta}\Big)^m \nonumber\\
 &\le& k_3(m)\Big((1+\ell_{\mu}\Delta)^m\|x\|^m
 +((L_{0}+\bar{a})\ell_{\mu}\Delta)^m\nonumber\\
 &&\quad +(L_{0}+\ell_{\sigma}(\|x\|+\bar{a}))^m\|B_{j-1}\|^m\Big)\nonumber\\
 &=& \breve{c}_1+\breve{c}_2\|B_{j-1}\|^m+(\breve{c}_3+\breve{c}_4\|B_{j-1}\|^m)\|x\|^m \label{ineq:xupperbound},
\end{eqnarray}

where $\breve{c}_i$ are all constant depending only on $m, \ell_{\mu},\ell_{\sigma},\Delta, L_{0}, \bar{a}$. Hence, we also have   
\begin{eqnarray}
    &&\|X^{(x,a)}-X^{(y,a)}\|\nonumber\\
&\le&\|x-y\|+\|\mu_{j-1}(x,a)-\mu_{j-1}(y,a)\|\Delta +\|\sigma_{j-1}(x,a)-\sigma_{j-1}(y,a)\|\sqrt{\Delta }\|B_{j-1}\|\nonumber\\
&\le&(1+\ell_{\mu}\Delta)\|x-y\|+\ell_{\sigma}\|x-y\|\sqrt{\Delta}\|B_{j-1}\|.  \label{ineq:1diff}
\end{eqnarray}

  By \eqref{ineq:xupperbound} and \eqref{ineq:1diff}, we have 
\begin{eqnarray}
 && \|X^{(x,a)}\|^m\|X^{(x,a)}-X^{(y,a)}\|  \nonumber\\
 &\le&\Big((\breve{c}_1+\breve{c}_2\|B_{j-1}\|^m)+( \breve{c}_3+\breve{c}_4\|B_{j-1}\|^m))\|x\|^m\Big) (1+\ell_{\mu}\Delta+\ell_{\sigma}\sqrt{\Delta}\|B_{j-1}\|)\|x-y\|\nonumber\\
 &=&\|x-y\|\Big(f_1(\|B_{j-1}\|)+f_2(\|B_{j-1}\|)\|x\|^m\Big),\label{ieq:product}
\end{eqnarray}
where $f_1(z)=\breve{c}_5+\breve{c}_6 z+\breve{c}_7 
z^m+\breve{c}_8 z^{m+1} $
and $f_2(z)=\breve{c}_9+\breve{c}_{10} z+\breve{c}_{11} z^m+\breve{c}_{12} z^{m+1}$, 
with $\breve{c}_i$ depends only on $\overline{C}_j, m, \bar{a}, \Delta, \ell_{\mu}, \ell_{\sigma}, L_{0}$. 

By the fact that $\mathbb{E}\big[\|B_{j-1}\|^q\big]$ is a finite constant for any integer $q$, we have 
\begin{eqnarray}\label{ineq:xmdiff}
\mathbb{E}\big[\|X^{(x,a)}\|^m\|X^{(x,a)}-X^{(y,a)}\|\big]\le \breve{c}_{13}(1+ \|x\|^m) \|x-y\|.
\end{eqnarray}
for some $\breve{c}_{13}$ only depending on $\overline{C}_j, m, \bar{a}, \Delta, \ell_{\mu}, \ell_{\sigma}, L_{0}$. 

Similarly, we have 
\begin{eqnarray}\label{ineq:ymdiff}
\mathbb{E}[\|X^{(y,a)}\|^m\|X^{(x,a)}-X^{(y,a)}\|]\le \breve{c}_{13}(1+ \|x\|^m) \|x-y\|.
\end{eqnarray}
Applying \eqref{ineq:1diff}, \eqref{ineq:xmdiff} and \eqref{ineq:ymdiff} to \eqref{ieq:difV*point}, we get 

\begin{eqnarray}\label{ineq:xymdiff}
    \mathbb{E}\left[V^*_{j-1}(X^{(x,a)})-V^*_{j-1}(X^{(y,a)})\right]\le \breve{c}_{14}(1+\|x\|^m+\|y\|^m)\|x-y\|,
\end{eqnarray}
with $\breve{c}_{14}=\overline{C}_j\Big(2\breve{c}_9+ (1+\ell_{\mu}\Delta+\ell_{\sigma}\sqrt{\Delta}\mathbb{E}[\|B_{j-1}\|])\Big)$.

Finally, let 
\begin{eqnarray}\label{eq:Ck-1 value}
    \overline{C}_{j-1}:=\ell_{r}+\breve{c}_{14},
\end{eqnarray}
and we have shown that
\begin{eqnarray}
    |V^*_{j-1}(x)-V^*_{j-1}(y)|\leq \overline{C}_{j-1}(1+\|x\|^m+\|y\|^m)\|x-y\|. 
\end{eqnarray}

\end{proof}

\subsection{Proof of Proposition \ref{thm: value function with any policy growth rate}}\label{app:proof-2-6}

\begin{proof}
We prove the statement by backward induction.

For the last step $h=H$,

\begin{eqnarray}\label{eq:value function growth H}
    |V_{H}^{\pi}(x)|&=&|\mathbb{E}_{a\sim \pi_{H}(x)}[\bar{R}_H(x,a)]|\nonumber\\
    &\leq&\mathbb{E}_{a\sim \pi_{H}(x)}[|\bar{R}_H(x,a)-\bar{R}_H(0,0)|+|\bar{R}_H(0,0)|]\nonumber\\
    &\leq&\ell_{r}(\|x\|^m+1)(\|x\|+\bar{a})+ L_{0}\nonumber\\
&\leq&\ell_{r}\|x\|^{m+1}
+\ell_{r}\bar{a}\left(\frac{m}{m+1}\|x\|^{m+1}+\frac{1}{m+1}\right)\nonumber\\
&&\quad+\ell_{r}\left(\frac{1}{m+1}\|x\|^{m+1}+\frac{m}{m+1}\right)
+\ell_{r}\bar{a}+L_{0}\nonumber\\
    &\leq&\widetilde{C}_{H}(\|x\|^{m+1}+1),
\end{eqnarray}

where $\widetilde{C}_{H}:=\max\{\ell_{r}(1+\frac{\bar{a}m+1}{m+1}),\ell_{r}(\frac{\bar{a}+m}{m+1}+\bar{a})+L_{0}\}$. Here, the first equality holds by \eqref{eq:bellman}, the second inequality holds by Assumption \ref{ass:lipschitz}, and the third inequality holds due to the fact that $\|x\|^m\leq \frac{m}{m+1}\|x\|^{m+1}+\frac{1}{m+1}$ and $\|x\| \leq \frac{1}{m+1}\|x\|^{m+1}+\frac{m}{m+1}$ . Now suppose the inequality \eqref{eq:value function with any policy growth rate} holds for $h=j>0$. We now prove the inequality for $h=j-1$:

\begin{eqnarray}\label{eq:value function growth k-1}  
  |V_{j-1}^{\pi}(x)|&\leq& \mathbb{E}_{a\sim \pi_{j-1}(x)}\Big[|\bar{R}_{j-1}(x,a)|\Big]\nonumber\\
  &&\quad+\mathbb{E}_{X_{j} \sim T_{j-1}(\cdot|x,a),a\sim \pi_{j-1}(x)}
  \Big[\big|V_{j}^{\pi}(X_{j})\big|\,\big|\,X_{j-1}=x\Big]\nonumber\\
   &\leq& \ell_{r}(\|x\|^m+1)(\overline{a}+\|x\|)+ L_{0}\nonumber\\
   &&\quad+\mathbb{E}_{X_{j} \sim T_{j-1}(\cdot|x,a),a\sim \pi_{j-1}(x)}
   \Big[\widetilde{C}_j(\|X_j\|^{m+1}+1)\,\big|\,X_{j-1}=x\Big] \nonumber\\
   &\leq& \ell_{r}\|x\|^{m+1}
   +\ell_{r}\bar{a}\left(\frac{m}{m+1}\|x\|^{m+1}+\frac{1}{m+1}\right)\nonumber\\
   &&\quad+\ell_{r}\left(\frac{1}{m+1}\|x\|^{m+1}+\frac{m}{m+1}\right)
   +\ell_{r}\bar{a}+L_{0}\nonumber\\
    &&+\widetilde{C}_j+ \widetilde{C}_j \breve{c}_4 (1+\|x\|^{m+1})\nonumber\\
   &\leq& \widetilde{C}_{j-1}(\|x\|^{m+1}+1),
\end{eqnarray}

where $\breve{c}_4$ depends only on $m, \Delta, \ell_{\mu}, \ell_{\sigma},\bar{a},L_{0}$, and we define $\widetilde{C}_{j-1}:=\max\{\ell_{r}(1+\frac{\bar{a}m+1}{m+1})+\widetilde{C}_j \breve{c}_4,\ell_{r}(\frac{\bar{a}+m}{m+1}+\bar{a}+\widetilde{C}_j\}$. Here, the first inequality holds due to \eqref{eq:bellman} and triangle inequality, the second inequality holds by Assumption \ref{ass:lipschitz}. In addition, the third inequality holds due to the fact that $\|x\|^m\leq \frac{m}{m+1}\|x\|^{m+1}+\frac{1}{m+1}$, $\|x\| \leq \frac{1}{m+1}\|x\|^{m+1}+\frac{m}{m+1}$ and an argument simular to \eqref{eq:moment estimation for the next step} such that
\begin{eqnarray*}
    \mathbb{E}_{X_{j} \sim T_{j-1}(\cdot|x,a),a\sim \pi_{j-1}(x)}\Big[\|X_j\|^{m+1}|X_{j-1}=x\Big] \leq  \breve{c}_4 (1+\|x\|^{m+1}).
\end{eqnarray*}
\end{proof}

\subsection{Local lipschitz property for optimal Q function}

In this subsection, we establish the local Lipschitz property of the optimal $Q$-function. This result plays an important role in the proof of Lemma \ref{lemma:theorem F.3 conclusion}. The proof follows the same general strategy as that for the Lipschitz property of the optimal value function. For completeness, we present the full argument here.

\begin{proposition}
Suppose Assumptions \ref{ass:lipschitz} and \ref{ass:expected reward local lipschitz} hold. Then for all $ h\in [H]$, it holds that 
\begin{eqnarray}\label{eq:Q local lipschitz}
    |Q_h^{*}(x_1,a_1)-Q_h^{*}(x_2,a_2)|\leq 2\overline{C}_h (1+\|x_1\|^m+\|x_2\|^m)(\|x_1-x_2\|+\|a_1-a_2\|),
\end{eqnarray}
where $\overline{C}_h$ is defined in \eqref{eq:value function local Lipschitz}.
\end{proposition}
\begin{proof}
We prove the statement by backward induction. For the last step $h=H$, we have
\begin{eqnarray*}
    |Q_H^{*}(x_1,a_1)-Q_H^{*}(x_2,a_2)|& =&|\bar{R}_{H}(x_1,a_1)-\bar{R}_{H}(x_2,a_2)|\nonumber\\
    &\leq&\ell_{r}(1+\|x_1\|^m+\|x_2\|^m)(\|x_1-x_2\|+\|a_1-a_2\|)\nonumber\\
    &\leq&2\overline{C}_{H}(1+\|x_1\|^m+\|x_2\|^m)(\|x_1-x_2\|+\|a_1-a_2\|),\nonumber
\end{eqnarray*}
where the first inequality holds due to Assumption \ref{ass:expected reward local lipschitz} and the second inequality holds due to the fact that $\overline{C}_{H}=\ell_{r}$ by \eqref{eq:CH value}.

Then suppose the inequality \eqref{eq:Q local lipschitz} holds for $h=j>1$. We then study the inequality for $h=j-1$.

For any state $x$ at time $j-1$ and any action $a\in \mathcal{A}$, denote by $X^{(x, a)}:=x+\mu_{j-1}(x,a)\Delta +\sigma_{j-1}(x,a)B_{j-1} \sqrt{\Delta}$ the next state.

By {{\eqref{eq:bellman Q star}}} we know 
$Q^*_{j-1}(x,a)=\bar{R}_{j-1}(x, a)+\mathbb{E}[V_j^*(X^{(x,a)})].$

Therefore
\begin{eqnarray}\label{ineq:Q decomposition}  
        &&|Q_{j-1}^*(x_1,a_1)-Q_{j-1}^*(x_2,a_2)|\nonumber\\
       &\leq& |Q_{j-1}^*(x_1,a_1)-Q_{j-1}^*(x_2,a_1)|+|Q_{j-1}^*(x_2,a_1)-Q_{j-1}^*(x_2,a_2)|\nonumber\\
       &\leq& \underbrace{|\bar{R}_{j-1}(x_1,a_1)-\bar{R}_{j-1}(x_2,a_1)|+|\bar{R}_{j-1}(x_2,a_1)-\bar{R}_{j-1}(x_2,a_2)|}_{(I)}\nonumber\\
        &&+\underbrace{\mathbb{E}[|V_j^*(X^{(x_1,a_1)})-V_j^*(X^{(x_2,a_1)})|]}_{(II)}+\underbrace{\mathbb{E}[|V_j^*(X^{(x_2,a_1)})-V_j^*(X^{(x_2,a_2)})|]}_{(III)}.       
\end{eqnarray}

For term (I), by Assumption \ref{ass:expected reward local lipschitz}, we have: 
      \begin{eqnarray}\label{ineq:rdiff}
          (I) \leq \ell_{r}(1+\|x_1\|^m+\|x_2\|^m)(\|x_1-x_2\|+\|a_1-a_2\|).
      \end{eqnarray}

For term (II), by \eqref{ineq:xymdiff}, we have:
         \begin{eqnarray}\label{ineq:x1x2diff}
             (II)\leq \breve{c}_{14}(1+\|x_1\|^m+\|x_2\|^m)\|x_1-x_2\|,
         \end{eqnarray}
where $\breve{c}_{14}$ is defined in \eqref{ineq:xymdiff}. 

Next, we handle term (III).  By Theorem \ref{eq:value function local Lipschitz},
we have: 

\begin{eqnarray}\label{ieq:diffV}
    &&|V^*_j(X^{(x_2,a_1)})-V^*_{j}(X^{(x_2,a_2)})|\nonumber\\
&\le& \overline{C}_j (1+\|X^{(x_2,a_1)}\|^m+\|X^{(x_2,a_2)}\|^m)\nonumber\\
&&\quad\times\|X^{(x_2,a_1)}-X^{(x_2,a_2)}\|.
\end{eqnarray}

By \eqref{ineq:xupperbound}, we have: 
\begin{eqnarray} \label{ineq:x2upperbound}
    \max\Big\{\|X^{(x_2,a_1)}\|^m, \|X^{(x_2,a_2)}\|^m\Big\}\leq \breve{c}_1+\breve{c}_2\|B_{j-1}\|^m+(\breve{c}_3+\breve{c}_4\|B_{j-1}\|^m)\|x_2\|^m, 
\end{eqnarray}
where $\breve{c}_1,\breve{c}_2, \breve{c}_3, \breve{c}_4$ are defined in \eqref{ineq:xupperbound}. 

By Assumption \eqref{ass:lipschitz}, we have:
\begin{eqnarray}
    &&\|X^{(x_2,a_1)}-X^{(x_2,a_2)}\|\nonumber\\
&\le&\|\mu_{j-1}(x_2,a_1)-\mu_{j-1}(x_2,a_2)\|\Delta +\|\sigma_{j-1}(x_2,a_1)-\sigma_{j-1}(x_2,a_2)\|\sqrt{\Delta }\|B_{j-1}\|\nonumber\\
&\le&(1+\ell_{\mu}\Delta)\|a_1-a_2\|+\ell_{\sigma}\|a_1-a_2\|\sqrt{\Delta}\|B_{j-1}\|.  \label{ineq:2diff}
\end{eqnarray}

By \eqref{ineq:x2upperbound} and \eqref{ineq:2diff}, we have

\begin{eqnarray}
 && (\|X^{(x_2,a_1)}\|^m+\|X^{(x_2,a_2)}\|^m)\|X^{(x_2,a_1)}-X^{(x_2,a_2)}\|  \nonumber\\
 &\le&2\Big((\breve{c}_1+\breve{c}_2\|B_{j-1}\|^m)
 +(\breve{c}_3+\breve{c}_4\|B_{j-1}\|^m)\|x_2\|^m\Big)\nonumber\\
 &&\quad\times(1+\ell_{\mu}\Delta+\ell_{\sigma}\sqrt{\Delta}\|B_{j-1}\|)\|a_1-a_2\|\nonumber\\
 &=&2\|a_1-a_2\|\Big(f_1(\|B_{j-1}\|)+f_2(\|B_{j-1}\|)\|x_2\|^m\Big),\label{ineq:product for a}
\end{eqnarray}

where $f_1(z)=\breve{c}_{5}+\breve{c}_{6} z+\breve{c}_{7} 
z^m+\breve{c}_{8} z^{m+1} $
and $f_2(z)=\breve{c}_{9}+\breve{c}_{10} z+\breve{c}_{11} z^m+\breve{c}_{12} z^{m+1}$ 
with $\breve{c}_i$ all defined in \eqref{ieq:product}.

By the fact that $\mathbb{E}\big[\|B_{j-1}\|^q\big]$ is a finite constant for any integer $q$, we have 
\begin{eqnarray}\label{ineq:expectation product for a}
\mathbb{E}[(\|X^{(x_2,a_1)}\|^m+\|X^{(x_2,a_2)}\|^m)\|X^{(x_2,a_1)}-X^{(x_2,a_2)}\|]\le 2\breve{c}_{13}(1+ \|x_2\|^m) \|a_1-a_2\|.
\end{eqnarray}
for $\breve{c}_{13}$ defined in \eqref{ineq:xmdiff}.

Combine \eqref{ineq:2diff} and \eqref{ineq:expectation product for a} in \eqref{ieq:diffV}, we get 
\begin{eqnarray}\label{ineq:a1a2diff}
    (III) \le 2\breve{c}_{14}(1+\|x_2\|^m)\|a_1-a_2\|,
\end{eqnarray}
with $\breve{c}_{14}$ defined in \eqref{ineq:xymdiff}. 

Applying \eqref{ineq:rdiff}, \eqref{ineq:x1x2diff} and \eqref{ineq:a1a2diff} to \eqref{ineq:Q decomposition}, we get: 

\begin{eqnarray*}
    &&|Q_{j-1}^*(x_1,a_1)-Q_{j-1}^*(x_2,a_2)|\\
    & \leq& (\ell_{r}+2\breve{c}_{14})(1+\|x_1\|^m+\|x_2\|^m)\nonumber\\
    &&\quad\times(\|x_1-x_2\|+\|a_1-a_2\|)\\
    & \leq& 2\overline{C}_{j-1}(1+\|x_1\|^m+\|x_2\|^m)\nonumber\\
    &&\quad\times(\|x_1-x_2\|+\|a_1-a_2\|),
\end{eqnarray*} 

where the second inequality holds due to \eqref{eq:Ck-1 value}. 
\end{proof}

\section{Technical details in Section \ref{sec:concentration}}
\label{app:concentration}
Note that $X_{h+1}^{k_{1}}-X_{h}^{k_{1}},...,X_{h+1}^{k_{n_h^k(B)}}-X_{h}^{k_{n_h^k(B)}}$ are conditionally independent given $X_{h}^{k_{1}},A_{h}^{k_{1}},...,X_{h}^{k_{n_h^k(B)}}$, and $A_{h}^{k_{n_h^k(B)}}$. Hence it is straightforward to derive concentration inequality for $\widehat{\mu}_{h}^{k}(B)$. However, the  expectation and variance of the estimator  $\widehat{\Sigma}_{h}^{k}(B)$ are challenging to analyze as $X_{h+1}^{k_{1}}-X_{h}^{k_{1}}-\widehat{\mu}_{h}^{k}(B)\Delta$ ,$\cdots$,$X_{h+1}^{k_{n_h^k(B)}}-X_{h}^{k_{n_h^k(B)}}-\widehat{\mu}_{h}^{k}(B)\Delta$ are {\it dependent}.
Hence, we consider the following intermediate quantity and decomposition to proceed:
\begin{eqnarray*}
    \widetilde{\Sigma}_{h}^{k}(B):=\frac{\sum_{i}\Big((X_{h+1}^{k_{i}}-X_{h}^{k_{i}})-\Delta\overline{\mathbb{E}}[\widehat{\mu}_{h}^{k}(B)]\Big)\Big((X_{h+1}^{k_{i}}-X_{h}^{k_{i}})-\Delta\overline{\mathbb{E}}[\widehat{\mu}_{h}^{k}(B)]\Big)^{\top}}{n_h^k(B)\Delta},
\end{eqnarray*}
and
\begin{eqnarray}
    &&\|\widehat{\Sigma}_{h}^{k}(B)-\Sigma_{h}(x,a)\|_{F}\nonumber\\
   & \leq & \underbrace{\Big\|\widehat{\Sigma}_{h}^{k}(B)-\widetilde{\Sigma}_{h}^{k}(B)\Big\|_{F}}_{(I)}+\underbrace{\Big\|\widetilde{\Sigma}_{h}^{k}(B)-\overline{\mathbb{E}}[\widetilde{\Sigma}_{h}^{k}(B)]\Big\|_{F}}_{(II)}+\underbrace{\Big\|\overline{\mathbb{E}}[\widetilde{\Sigma}_{h}^{k}(B)]-\Sigma_{h}(x,a)\Big\|_{F}}_{(III)}. \label{eq:volatility decomposition}
\end{eqnarray}

We analyze (I)-(III) in the next subsection. As a heads-up,
\begin{itemize}
    \item Term (II) on the RHS is straightforward to bound as $X_{h+1}^{k_{1}}-X_{h}^{k_{1}}-\overline{\mathbb{E}}[\widehat{\mu}_{h}^{k}(B)]\Delta,\cdots,X_{h+1}^{k_{n_h^k(B)}}-X_{h}^{k_{n_h^k(B)}}-\overline{\mathbb{E}}[\widehat{\mu}_{h}^{k}(B)]\Delta$ are conditionally independent. We handle this term 
    by Lemma \ref{lemma:conditional distribution of the volatility term} and Proposition \ref{lemma:Tail Estimates for Gaussian Sample Covariance}.
\item To bound term  (I), let $P_{i}:=(X_{h+1}^{k_{i}}-X_{h}^{k_{i}})-\widehat{\mu}_{h}^{k}(B)\Delta$ and $Q_{i}:=(X_{h+1}^{k_{i}}-X_{h}^{k_{i}})-\Delta\overline{\mathbb{E}}[\widehat{\mu}_{h}^{k}(B)]$. Then we have
\begin{eqnarray}
    \|\widehat{\Sigma}_{h}^{k}(B)-\widetilde{\Sigma}_{h}^{k}(B)\|_{F}&=&\Bigg\|\frac{\sum_{i}P_{i}P_{i}^{\top}}{n_h^k(B)\Delta}-\frac{\sum_{i}Q_{i}Q_{i}^{\top}}{n_h^k(B)\Delta}\Bigg\|_{F}\nonumber\\
    &=&\Bigg\|\frac{\sum_{i}P_{i}(P_{i}^{\top}-Q_{i}^{\top})}{n_h^k(B)\Delta}+\frac{\sum_{i}(P_{i}-Q_{i})Q_{i}^{\top}}{n_h^k(B)\Delta}\Bigg\|_{F}\nonumber\\
    &=&\Bigg\|\Big(\widehat{\mu}_{h}^{k}(B)-\overline{\mathbb{E}}[\widehat{\mu}_{h}^{k}(B)]\Big)\Big(\widehat{\mu}_{h}^{k}(B)-\overline{\mathbb{E}}[\widehat{\mu}_{h}^{k}(B)]\Big)^{\top}\Delta\Bigg\|_{F}, \label{eq:volatility estimator difference}
\end{eqnarray}
which will be handled by 
Lemma \ref{lemma:conditional distribution of drift term} and Proposition \ref{lemma:Tail Estimate for Standard Normal Distribution}. 
\item As for term (III), we provide an upper bound in Theorem \ref{thm:BIAS Bound}.
\end{itemize}

\subsection{Lemma \ref{lemma:conditional distribution of drift term}}\label{app:proof-4-2}
\begin{lemma}\label{lemma:conditional distribution of drift term}
For all $ (h,k) \in {[H-1]\times [K]}$, we have:

\begin{eqnarray}\label{eq:conditional distribution of drift term}
       &&\frac{X_{h+1}^{k_{i}}-X_{h}^{k_{i}}}{\Delta}\,\Big|\,(X_{h}^{k_{i}},A_{h}^{k_{i}})\sim \NN\Bigg(\mu_{h}(X_{h}^{k_{i}},A_{h}^{k_{i}}),\frac{\Sigma_{h}(X_{h}^{k_{i}},A_{h}^{k_{i}})}{\Delta}\Bigg); \\
       &&\widehat{\mu}_{h}^{k}(B)-\frac{\sum_{i=1}^{n_h^k(B)}\mu_{h}(X_{h}^{k_{i}},A_{h}^{k_{i}})}{n_h^k(B)}\nonumber\\
       &&\quad\Big|\,(X_{h}^{k_{1}},A_{h}^{k_{1}},\ldots,
       X_{h}^{k_{n_h^k(B)}},A_{h}^{k_{n_h^k(B)}})\nonumber\\
       &&\quad\sim\NN\Bigg(0,\frac{\sum_{i=1}^{n_h^k(B)}\Sigma_{h}(X_{h}^{k_{i}},A_{h}^{k_{i}})}{n^{2}\Delta}\Bigg).\nonumber
\end{eqnarray}

\end{lemma}
\begin{proof}
 The first and second statements are straightforward by  the definition in \eqref{eq:state_process} and the independence among $X_{h+1}^{k_{1}}-X_{h}^{k_{1}},...,X_{h+1}^{k_{n_h^k(B)}}-X_{h}^{k_{n_h^k(B)}}$ given $X_{h}^{k_{1}},A_{h}^{k_{1}},...,X_{h}^{k_{n_h^k(B)}},A_{h}^{k_{n_h^k(B)}}$.
\end{proof}

\subsection{Lemma \ref{lemma:conditional distribution of the volatility term}}\label{app:proof-4-4}
\begin{lemma}\label{lemma:conditional distribution of the volatility term}
 The following holds for all $ (h,k) \in[H-1]\times [K]$, $B \in \mathcal{P}_h^k$ such that for $n\in \mathbb{N}_+$:

\begin{eqnarray*}
&&\overline{\mathbb{E}}\big[\widetilde{\Sigma}_{h}^{k}(B)\big]\\
&=&\frac{1}{n_h^k(B)}\sum_{i=1}^{n_h^k(B)}\Big(\Sigma_{h}(X_{h}^{k_{i}},A_{h}^{k_{i}})\nonumber\\
&&\quad+\big(\mu_{h}(X_{h}^{k_{i}},A_{h}^{k_{i}})-\overline{\mathbb{E}}[\widehat{\mu}_{h}^{k}(B)]\big)\nonumber\\
&&\qquad\times\big(\mu_{h}(X_{h}^{k_{i}},A_{h}^{k_{i}})-\overline{\mathbb{E}}[\widehat{\mu}_{h}^{k}(B)]\big)^{\top}\Delta\Big).
\end{eqnarray*}

\end{lemma}
\begin{proof}
From Lemma \ref{lemma:conditional distribution of drift term} we know that 
\begin{eqnarray*}
\frac{X_{h+1}^{k_{i}}-X_{h}^{k_{i}}}{\Delta}\,\Big|\,(X_{h}^{k_{i}},A_{h}^{k_{i}})\sim \NN\left(\mu_{h}(X_{h}^{k_{i}},A_{h}^{k_{i}}),\frac{\Sigma_{h}(X_{h}^{k_{i}},A_{h}^{k_{i}})}{\Delta}\right).
\end{eqnarray*}
Therefore 

\begin{eqnarray*}
    &&\frac{X_{h+1}^{k_{i}}-X_{h}^{k_{i}}-\Delta\overline{\mathbb{E}}[\widehat{\mu}^k_{h}(B)]}
    {\sqrt{n_h^k(B)\Delta}}\,\Big|\,(X_{h}^{k_{1}},A_{h}^{k_{1}},\ldots)\nonumber\\
    &&\quad\sim \NN\Bigg(
    \frac{\big(\mu_{h}(X_{h}^{k_{i}},A_{h}^{k_{i}})-\overline{\mathbb{E}}[\widehat{\mu}_{h}^{k}(B)]\big)\sqrt{\Delta}}
    {\sqrt{n_h^k(B)}},\nonumber\\
    &&\qquad\frac{\Sigma_{h}(X_{h}^{k_{i}},A_{h}^{k_{i}})}{n_h^k(B)}\Bigg),
\end{eqnarray*}

where the expression for the conditional mean  follows  the independence and the property of Gaussian distribution.
\end{proof}

Next, we establish concentration inequalities for the estimators of the drift and volatility terms, as presented in Propositions \ref{lemma:Tail Estimate for Standard Normal Distribution} and \ref{lemma:Tail Estimates for Gaussian Sample Covariance}.

\subsection{Proof of Proposition \ref{lemma:Tail Estimate for Standard Normal Distribution}}\label{app:proof-4-3}

\begin{proposition}\label{lemma:Tail Estimate for Standard Normal Distribution}
 Suppose Assumption \ref{ass:lipschitz} holds, then we have the following result:  
\begin{eqnarray} \label{eq:Tail Estimate for Standard Normal Distribution}
\mathbb{P}
\begin{pmatrix}
  &  \left\|\widehat{\mu}_{h}^{k}(B)-\overline{\mathbb{E}}[\widehat{\mu}_{h}^{k}(B)]\right\|\leq \kappa_{\mu}\big(\delta,\|\tilde{x}(^{o}B)\|,n_h^k(B)\big),\\
&\forall h\in [H-1], k\in [K], B \in \mathcal{P}_h^k\,\,\mbox{ with}\,\, n_h^k(B)>0
\end{pmatrix}\ge 1-\delta.
\end{eqnarray}
\end{proposition}

\begin{proof}
For fixed $h,k$ and $B\in \mathcal{P}_h^k$ s.t. $n_h^k(B)>0$, according to Lemma \ref{lemma:conditional distribution of drift term}, we have for all $z \in \mathbb{R}^{d_{\mathcal{S}}}$,
\begin{eqnarray*}
\overline{\mathbb{E}}\Big[\exp\Big(z^{\top}\big(\widehat{\mu}_{h}^{k}(B)-\overline{\mathbb{E}}[\widehat{\mu}_{h}^{k}(B)]\big)\Big)\Big] 
&\leq&\exp\Big(\frac{1}{2}\|z\|^2\left\|\frac{\sum_{i=1}^{n_h^k(B)}\Sigma_{h}(X_{h}^{k_{i}},A_{h}^{k_{i}})}{n^{2}\Delta}\right\|\Big)\\
&\leq&\exp\Big(\frac{1}{2}\|z\|^2\frac{\sum_{i=1}^{n_h^k(B)}\|\sigma_{h}(X_{h}^{k_{i}},A_{h}^{k_{i}})\|^{2}}{n^{2}\Delta}\Big)\\
&\leq&\exp\Big(\frac{1}{2}\|z\|^2\frac{\eta(\|\tilde{x}(^{o}B)\|)^{2}}{n_h^k(B)\Delta}\Big).
\end{eqnarray*}

Then Theorem 1 in \cite{Hsu2011ATI} guarantees that: 

\begin{eqnarray*}
&&\overline{\mathbb{P}}\Bigg(\Big\|\widehat{\mu}_{h}^{k}(B)
-\overline{\mathbb{E}}[\widehat{\mu}_{h}^{k}(B)]\Big\|^{2}\geq\nonumber\\
&&\quad\frac{\eta(\|\tilde{x}(^{o}B)\|)^{2}}{n_h^k(B)\Delta}
\Bigg(d_{\mathcal{S}}+2d_{\mathcal{S}}\sqrt{\log\Big(\frac{HK^2}{\delta}\Big)}
+2\log\Big(\frac{HK^2}{\delta}\Big)\Bigg)\Bigg)\nonumber\\
&&\quad\leq \frac{\delta}{HK^2}.
\end{eqnarray*}

Note that $d_{\mathcal{S}}+2d_{\mathcal{S}}\sqrt{\log\Big(\frac{HK^2}{\delta}\Big)}+2\log\Big(\frac{HK^2}{\delta}\Big)\leq \left(\sqrt{d_{\mathcal{S}}}+\sqrt{2\log\Big(\frac{HK^2}{\delta}\Big)}\right)^{2}$, hence we have:
\begin{eqnarray*}   \overline{\mathbb{P}}\Big(\Big\|\widehat{\mu}_{h}^{k}(B)-\overline{\mathbb{E}}[\widehat{\mu}_{h}^{k}(B)]\Big\|\geq \kappa_{\mu}(\delta,\|\tilde{x}(^{o}B)\|,n_h^k(B))\Big)\leq \frac{\delta}{HK^2}.
\end{eqnarray*}
Taking expectations on both side, we have: 
\begin{eqnarray*}
    \mathbb{P}\Big(\Big\|\widehat{\mu}_{h}^{k}(B)-\overline{\mathbb{E}}[\widehat{\mu}_{h}^{k}(B)]\Big\|\geq \kappa_{\mu}(\delta,\|\tilde{x}(^{o}B)\|,n_h^k(B))\Big)\leq \frac{\delta}{HK^2}.
\end{eqnarray*}
Then taking a union bound, we get: 

\begin{eqnarray*}
&&\mathbb{P}\Bigg(\cap_{h=1}^{H-1}\cap_{k=1}^{K}
\cap_{B\in\mathcal{P}_h^k,n_h^k(B)>0}\Bigg\{\nonumber\\
&&\qquad\Big\|\widehat{\mu}_{h}^{k}(B)-\overline{\mathbb{E}}[\widehat{\mu}_{h}^{k}(B)]\Big\|
\leq \kappa_{\mu}(\delta,\|\tilde{x}(^{o}B)\|,n_h^k(B))\Bigg\}\Bigg) \\
        &=& \mathbb{P}\Bigg(\cap_{h=1}^{H-1}\cap_{k=1}^{K}
        \cap_{n_h^k(B_h^k)=1,B_h^k\in \mathcal{P}_h^{k-1}}^{K}\Bigg\{\nonumber\\
        &&\qquad\Big\|\widehat{\mu}_{h}^{k}(B_h^k)-\overline{\mathbb{E}}[\widehat{\mu}_{h}^{k}(B_h^k)]\Big\|\nonumber\\
        &&\qquad\leq \kappa_{\mu}(\delta,\|\tilde{x}(^{o}B_h^k)\|,n_h^k(B_h^k))\Bigg\}\Bigg)\\
        &\geq & 1-\sum_{h=1}^{H-1}\sum_{k=1}^K
        \sum_{n_h^k(B_h^k)=1,B_h^k\in\mathcal{P}_h^k}^K\nonumber\\
        &&\qquad\mathbb{P}\Big(\Big\|\widehat{\mu}_{h}^{k}(B_h^k)
        -\overline{\mathbb{E}}[\widehat{\mu}_{h}^{k}(B_h^k)]\Big\|\nonumber\\
        &&\qquad\geq \kappa_{\mu}(\delta,\|\tilde{x}(^{o}B_h^k)\|,n_h^k(B_h^k))\Big)\\
        &\geq & 1-\delta,  
\end{eqnarray*}

where $B_h^k$ is selected according to Algorithm \ref{alg:ML 2}, and note that $\widehat{\mu}_{h}^{k}(B_h^k) $ depends on $n_h^k(B_h^k)$. The first equality holds since only the estimate for the selected block $B_h^k$ is updated for each $(h,k)$ pair. The first inequality holds since, for a countable set of events $E_1,E_2,...,$ we have $\mathbb{P}(\cap_{i}E_i)\geq 1-\sum_{i}\mathbb{P}(E_i^{\complement})$. 

\end{proof}

\subsection{Proof of Proposition \ref{lemma:Tail Estimates for Gaussian Sample Covariance}}\label{app:proof-4-5}
\begin{proposition}\label{lemma:Tail Estimates for Gaussian Sample Covariance}
 Suppose Assumption \ref{ass:lipschitz} holds. Then  there exist universal constants $D_1>0,D_2>1,D_3>0$ (independent of $\rho$) such that:
 \begin{eqnarray} \label{eq:Tail Estimates for Gaussian Sample Covariance}
\mathbb{P}
\begin{pmatrix}
\Big\|\widetilde{\Sigma}_{h}^{k}(B)-\overline{\mathbb{E}}[\widetilde{\Sigma}_{h}^{k}(B)]\Big\|\leq \kappa_{\Sigma}\big(\delta,\|\tilde{x}(^{o}B)\|,n_h^k(B)\big),\\
\forall h\in [H-1], k\in [K], B \in \mathcal{P}_h^k \,\,\mbox{ with}\,\,  n_h^k(B)>0
\end{pmatrix}\ge 1-\delta.
 \end{eqnarray}
\end{proposition}
\begin{proof}
For any $h,k \in {[H-1]\times [K]}$ and $B \in \mathcal{P}_h^k, n_h^k(B)>0$, denote $Z_{i}:=\frac{X_{h+1}^{k_{i}}-X_{h}^{k_{i}}-\Delta\overline{\mathbb{E}}[\widehat{\mu}_{h}^{k}(B)]}{\sqrt{\Delta}}$, then by Lemma \ref{lemma:conditional distribution of the volatility term}:
\begin{eqnarray*}
    \widetilde{\Sigma}_{h}^{k}(B)-\overline{\mathbb{E}}[\widetilde{\Sigma}_{h}^{k}(B)]&=&\frac{\sum_{i=1}^{n_h^k(B)}Z_{i}Z_{i}^{\top}}{n_h^k(B)}-\frac{\sum_{i=1}^{n_h^k(B)}\overline{\mathbb{E}}[Z_{i}]\overline{\mathbb{E}}[Z_{i}^{\top}]}{n_h^k(B)}-\frac{\sum_{i=1}^{n_h^k(B)}\overline{\mathbb{V}}[Z_{i}]}{n_h^k(B)}\\  &=&\frac{\sum_{i=1}^{n_h^k(B)}\Big(Z_{i}x_{i}^{\top}-\overline{\mathbb{E}}[Z_{i}Z_{i}^{\top}]\Big)}{n_h^k(B)}.
\end{eqnarray*}
Notice that $Z_{1},...,Z_{n_h^k(B)}$ are conditionally independent given
\[
X_{h}^{k_{1}},A_{h}^{k_{1}},...,X_{h}^{k_{n_h^k(B)}},A_{h}^{k_{n_h^k(B)}}.
\]
They share the same sub-Gaussian variance proxy $\eta(\|\tilde{x}(^{o}B)\|)$ with
$\|\Sigma_{h}(X_{h}^{k_{i}},A_{h}^{k_{i}})\|\leq \eta(\|\tilde{x}(^{o}B)\|)^{2}$.\\
Then by Theorem 6.5 in \cite{wainwright2019high}, there exist universal constants $D_{1}>0,D_{2}>1$ and $D_{3}>0$ such that:

\begin{eqnarray*}
&&\overline{\mathbb{P}}\Bigg(\Big\|\widetilde{\Sigma}_{h}^{k}(B)
-\overline{\mathbb{E}}[\widetilde{\Sigma}_{h}^{k}(B)]\Big\|\geq\nonumber\\
&&\quad\eta(\|\tilde{x}(^{o}B)\|)^{2}
\left(D_{1}\left(\sqrt{\frac{d_{\mathcal{S}}}{n_h^k(B)}}
+\frac{d_{\mathcal{S}}}{n_h^k(B)}\right)+\epsilon\right)\Bigg)\nonumber\\
&&\quad\leq D_{2}e^{-D_{3}n_h^k(B)\min\{\epsilon,\epsilon^{2}\}}. 
\end{eqnarray*}

Notice that for $a,b,c \in \mathbb{R}$, we have $\max\{a,b\}\leq a+b$ and $\frac{1}{c}\leq \sqrt{\frac{1}{c}}$ for $c\geq 1$.Therefore, we conclude that:
\begin{eqnarray*}
\overline{\mathbb{P}}\Bigg(\Big\|\widetilde{\Sigma}_{h}^{k}(B)-\overline{\mathbb{E}}[\widetilde{\Sigma}_{h}^{k}(B)]\Big\|\geq \kappa_{\Sigma}\Big(\delta,\|\tilde{x}(^{o}B)\|,n_h^k(B)\Big) \Bigg)\le \frac{\delta}{HK^2}.
\end{eqnarray*} 

Taking expectations, we have: 
\begin{eqnarray*}
\mathbb{P}\Bigg(\Big\|\widetilde{\Sigma}_{h}^{k}(B)-\overline{\mathbb{E}}[\widetilde{\Sigma}_{h}^{k}(B)]\Big\|\geq \kappa_{\Sigma}\Big(\delta,\|\tilde{x}(^{o}B)\|,n_h^k(B)\Big) \Bigg)\le \frac{\delta}{HK^2}.
\end{eqnarray*} 

Then taking a union bound, we get:
\begin{eqnarray*}
    &&\mathbb{P}\Bigg(\cap_{h=1}^{H-1}\cap_{k=1}^{K}\cap_{B\in\mathcal{P}_h^k,n_h^k(B)>0}\Big\|\widetilde{\Sigma}_{h}^{k}(B)-\overline{\mathbb{E}}[\widetilde{\Sigma}_{h}^{k}(B)]\Big\|\leq \kappa_{\Sigma}(\delta,\|\tilde{x}(^{o}B)\|,n_h^k(B))\Bigg\}\Bigg) \nonumber\\
        &=& \mathbb{P}\Bigg(\cap_{h=1}^{H-1}\cap_{k=1}^{K}\cap_{n_h^k(B_h^k)=1}^{K}\Big\|\widetilde{\Sigma}_{h}^{k}(B_h^k)-\overline{\mathbb{E}}[\widetilde{\Sigma}_{h}^{k}(B_h^k)]\Big\|\leq \kappa_{\Sigma}(\delta,\|\tilde{x}(^{o}B_h^k)\|,n_h^k(B_h^k))\Bigg\}\Bigg)\nonumber\\
        &\geq & 1-\sum_{h=1}^{H-1}\sum_{k=1}^{K}\sum_{n_h^k(B_h^k)=1}^K\mathbb{P}\Bigg(\Big\|\widetilde{\Sigma}_{h}^{k}(B_h^k)-\overline{\mathbb{E}}[\widetilde{\Sigma}_{h}^{k}(B_h^k)]\Big\|\geq \kappa_{\Sigma}(\delta,\|\tilde{x}(^{o}B_h^k)\|,n_h^k(B_h^k)) \Bigg)\nonumber\\
        &\geq & 1-\delta, 
\end{eqnarray*}
where $B_h^k$ is selected according to Algorithm \ref{alg:ML 2} and note that $\widetilde{\Sigma}_{h}^{k}(B_h^k)$ depends on $n_h^k(B_h^k)$. The first equality holds since only the estimate for the selected block $B_h^k$ is updated for each $(h,k)$ pair. The first inequality holds since, for a countable set of events $E_1,E_2,...,$ we have $\mathbb{P}(\cap_{i}E_i)\geq 1-\sum_{i}\mathbb{P}(E_i^{\complement})$. 

\end{proof}

\subsection{Proof of Theorem\ref{thm:high_prob_w_bound}}\label{app:high_prob_w_bound}
\begin{proof} 
We have
\begin{eqnarray}  &&\mathcal{W}_{2}\Big(\NN(\widehat{\mu}_{h}^{k}(B)\Delta,\widehat{\Sigma}_{h}^{k}(B)\Delta),\NN(\mu_{h}(x,a)\Delta,\Sigma_{h}(x,a)\Delta)\Big)\nonumber\\
    &=&\Bigg(\|\widehat{\mu}_{h}^{k}(B)\Delta-\mu_{h}(x,a)\Delta\|^{2}\nonumber\\
    &&+{\rm Tr}\Big(\widehat{\Sigma}_{h}^{k}(B)\Delta+\Sigma_{h}(x,a)\Delta-2((\widehat{\Sigma}_{h}^{k}(B)\Delta)^{\frac{1}{2}}(\Sigma_{h}(x,a)\Delta)(\widehat{\Sigma}_{h}^{k}(B)\Delta)^{\frac{1}{2}})^{\frac{1}{2}}\Big)\Bigg)^{\frac{1}{2}} \nonumber
    \nonumber \\
    &\leq&\sqrt{\|\widehat{\mu}_{h}^{k}(B)\Delta-\mu_{h}(x,a)\Delta\|^{2}+\|(\widehat{\Sigma}_{h}^{k}(B)\Delta)^{\frac{1}{2}}-(\Sigma_{h}(x,a)\Delta)^{\frac{1}{2}}\|_{F}^{2}} \nonumber \\
    &\leq& \|\widehat{\mu}_{h}^{k}(B)\Delta-\mu_{h}(x,a)\Delta\|+\|(\widehat{\Sigma}_{h}^{k}(B)\Delta)^{\frac{1}{2}}-(\Sigma_{h}(x,a)\Delta)^{\frac{1}{2}}\|_{F} \nonumber \\
    &\leq&  \underbrace{\|\widehat{\mu}_{h}^{k}(B)\Delta-\mu_{h}(x,a)\Delta\|}_{(I)}+ \underbrace{\frac{1}{\sqrt{\lambda}}\|\widehat{\Sigma}_{h}^{k}(B)\Delta^{\frac{1}{2}}-\Sigma_{h}(x,a)\Delta^{\frac{1}{2}}\|_{F}}_{(II)}. \label{eq:Wasserstein part 1}  
\end{eqnarray}
Here, the first equality holds by Proposition 7 in \cite{Givens1984ACO} and the first inequality holds by Theorem 1 in \cite{Bhatia2017OnTB}. The second inequality holds since  $\sqrt{a^{2}+b^{2}}\leq a+b$ for $a\geq 0, b\geq 0$; and, to get the third inequality, we apply (1.1)-(1.3) in \cite{SCHMITT1992215} and \eqref{ass-eqn:volatility}.

For term (I), we have:
\begin{eqnarray}\label{eq:drift estimator decomposition}
    \big\|\widehat{\mu}_{h}^{k}(B)\Delta-\mu_{h}(x,a)\Delta\big\|\leq \big\|\widehat{\mu}_{h}^{k}(B)\Delta-\overline{\mathbb{E}}[\widehat{\mu}_{h}^{k}(B)]\Delta\big\|+\big\|\overline{\mathbb{E}}[\widehat{\mu}_{h}^{k}(B)]\Delta-\mu_{h}(x,a)\Delta\big\|.
\end{eqnarray}

For term (II), we have: 
    \begin{eqnarray}
        &&\frac{1}{\sqrt{\lambda}}\|\widehat{\Sigma}_{h}^{k}(B)\Delta^{\frac{1}{2}}-\Sigma_{h}(x,a)\Delta^{\frac{1}{2}}\|_{F}\nonumber \\
    &\leq&\frac{1}{\sqrt{\lambda}}(\|\widehat{\Sigma}_{h}^{k}(B)\Delta^{\frac{1}{2}}-\widetilde{\Sigma}_{h}^{k}(B)\Delta^{\frac{1}{2}}\|_{F}+\|\widetilde{\Sigma}_{h}^{k}(B)\Delta^{\frac{1}{2}}-\overline{\mathbb{E}}[\widetilde{\Sigma}_{h}^{k}(B)]\Delta^{\frac{1}{2}}\|_{F}\nonumber\\
    &&+\|\overline{\mathbb{E}}[\widetilde{\Sigma}_{h}^{k}(B)]\Delta^{\frac{1}{2}}-\Sigma_{h}(x,a)\Delta^{\frac{1}{2}}\|_{F} )\nonumber\\
    &\leq& \frac{1}{\sqrt{\lambda}}\Big(\Big\|\Big(\widehat{\mu}_{h}^{k}(B)-\overline{\mathbb{E}}[\widehat{\mu}_{h}^{k}(B)]\Big)\Big\|^{2}\Delta^{\frac{3}{2}}+\sqrt{d_{\mathcal{S}}}\|\widetilde{\Sigma}_{h}^{k}(B)\Delta^{\frac{1}{2}}-\overline{\mathbb{E}}[\widetilde{\Sigma}_{h}^{k}(B)]\Delta^{\frac{1}{2}}\|\nonumber\\
   &&+\|\overline{\mathbb{E}}[\widetilde{\Sigma}_{h}^{k}(B)]\Delta^{\frac{1}{2}}-\Sigma_{h}(x,a)\Delta^{\frac{1}{2}}\|_{F} \Big).\label{eq:volatility estimator decomposition}
    \end{eqnarray}
Here, we apply \eqref{eq:volatility decomposition} to get the first inequality and \eqref{eq:volatility estimator difference} to get the second inequality. 

Note that by Propositions \ref{lemma:Tail Estimate for Standard Normal Distribution} and \ref{lemma:Tail Estimates for Gaussian Sample Covariance}, we have \eqref{eq:Tail Estimate for Standard Normal Distribution} and \eqref{eq:Tail Estimates for Gaussian Sample Covariance} hold. Combine \eqref{eq:Tail Estimate for Standard Normal Distribution}, \eqref{eq:Tail Estimates for Gaussian Sample Covariance}, \eqref{eq:Wasserstein part 1}, \eqref{eq:drift estimator decomposition} and \eqref{eq:volatility estimator decomposition}, we verify that it holds with probability at least $1-2\delta$ that, for any $ (h,k) \times [H-1]\times [K]$, $B \in \mathcal{P}_h^k$ with $n_h^k(B)>0$, and any $(x,a) \in B$, 

\begin{eqnarray*}
&&\mathcal{W}_{2}\Big(\NN(\widehat{\mu}_{h}^{k}(B)\Delta,\widehat{\Sigma}_{h}^{k}(B)\Delta),\NN(\mu_{h}(x,a)\Delta,\Sigma_{h}(x,a)\Delta)\Big) \nonumber\\
    &\leq& \Delta\kappa_{\mu}(\delta,\|\tilde{x}(^{o}B)\|,n_h^k(B))\nonumber\\
    &&+\frac{\Delta^{\frac{3}{2}}}{\sqrt{\lambda}}
    \kappa_{\mu}(\delta,\|\tilde{x}(^{o}B)\|,n_h^k(B))^{2}\nonumber\\
    &&+\frac{\sqrt{d_{\mathcal{S}}}\Delta^{\frac{1}{2}}}{\sqrt{\lambda}}
    \kappa_{\Sigma}(\delta,\|\tilde{x}(^{o}B)\|,n_h^k(B))\nonumber\\
    &&+\Big\|\overline{\mathbb{E}}[\widehat{\mu}_{h}^{k}(B)]-\mu_{h}(x,a)\Big\|\Delta\nonumber\\
    &&+\Big\|\overline{\mathbb{E}}[\widetilde{\Sigma}_{h}^{k}(B)]-\Sigma_h(x,a)\Big\|
    \frac{\sqrt{\Delta}}{\sqrt{\lambda}}.
    \end{eqnarray*}

\end{proof}

\subsection{Proof of Theorem \ref{thm:transition kernel wasserstein local lipschitz all together}}\label{app:proof-4-8}

We first introduce two technical lemmas. 

\begin{lemma}\label{lemma:2q moment of multi normal}
    Suppose $Z\sim \NN(\mu,\Sigma)$ with $\mu \in \mathbb{R}^d$, $\Sigma \in \mathbb{R}^{d\times d}$ and $\Sigma \succeq 0 $, then $\forall q \in \mathbb{N}^{+}$:
    \begin{eqnarray}\label{eq:2-norm mulitinormal estimate}
        (\mathbb{E}_{Z\sim \NN(\mu,\Sigma)}[\|Z\|^{2q}])^{\frac{1}{2}}\leq \widetilde{C}(q,d)(\|\mu\|^{q}+\|\Sigma\|^{\frac{q}{2}}), 
    \end{eqnarray}
\end{lemma}
where $\widetilde{C}(q,d)$ is defined in \eqref{eq:C(q,d)}.

\begin{proof}
Denote $Z_{j}$ as the $j$th random variable of the random vector $Z$, and hence $Z_{j}\sim \NN(\mu_{j},\Sigma_{jj})$ where $\mu_{j}$ is the $j$th component of $\mu$ and $\Sigma_{jj}$ is the $(j,j)$th component of $\Sigma$. Therefore,
    \begin{eqnarray}\label{eq:2q moment of multi normal}
            (\mathbb{E}_{Z\sim \NN(\mu,\Sigma)}[\|Z\|]^{2q})^{\frac{1}{2}}&=& \Big(\mathbb{E}_{Z\sim \NN(\mu,\Sigma)}[Z_{1}^{2}+...+Z_{d}^{2}]^{q}\Big)^{\frac{1}{2}} \nonumber\\
            &\leq&\Big(d^{q-1}\sum_{j=1}^{d}\mathbb{E}_{Z_{j}\sim \NN(\mu_{j},\Sigma_{jj})}[Z_{j}]^{2q}\Big)^{\frac{1}{2}}\nonumber\\
            &\leq&d^{\frac{q-1}{2}}\sum_{j=1}^{d}\Big(\mathbb{E}_{Z_{j}\sim \NN(\mu_{j},\Sigma_{jj})}[|Z_{j}-\mu_{j}|+|\mu_{j}|]^{2q}\Big)^{\frac{1}{2}}\nonumber\\
            &\leq&d^{\frac{q-1}{2}}\sum_{j=1}^{d}\Big(2^{2q-1}(\mathbb{E}_{Z_{j}\sim \NN(\mu_{j},\Sigma_{jj})}[|Z_{j}-\mu_{j}|]^{2q}+|\mu_{j}|^{2q})\Big)^{\frac{1}{2}},
    \end{eqnarray}
  where the first inequality holds by the power-mean inequality. 
  The second inequality follows from $\sqrt{a+b}\leq \sqrt{a} +\sqrt{b}$ when $a,b>0$.  
  
According to \cite{winkelbauer2012moments}, we have
\begin{eqnarray}\label{eq:2q central moment of uni normal}
    \mathbb{E}_{Z_{j}\sim \NN(\mu_{j},\Sigma_{jj})}[|Z_{j}-\mu_{j}|]^{2q}=\frac{2^{q}\Gamma(q+\frac{1}{2})}{\sqrt{\pi}}(\Sigma_{jj})^{q}.
\end{eqnarray}
Hence, combining \eqref{eq:2q moment of multi normal} and \eqref{eq:2q central moment of uni normal}, we have:
\begin{eqnarray}\label{eq:2-norm 2q multi normal}
        (\mathbb{E}_{Z\sim \NN(\mu,\Sigma)}[\|Z\|]^{2q})^{\frac{1}{2}} &\leq& d^{\frac{q-1}{2}}\Sigma_{j=1}^{d}\Big(2^{2q-1}(\mathbb{E}_{Z_{j}\sim \NN(\mu_{j},\Sigma_{jj})}[|Z_{j}-\mu_{j}|]^{2q}+|\mu_{j}|^{2q})\Big)^{\frac{1}{2}}\nonumber\\
        &\leq& d^{\frac{q-1}{2}}\Sigma_{j=1}^{d}\Big(2^{2q-1}(\frac{2^{q}\Gamma(q+\frac{1}{2})}{\sqrt{\pi}}(\Sigma_{jj})^{q}+|\mu_{j}|^{2q})\Big)^{\frac{1}{2}}\nonumber\\
        &\leq&  d^{\frac{q-1}{2}}\Sigma_{j=1}^{d}\Big(2^{2q-1}(\frac{2^{q}\Gamma(q+\frac{1}{2})}{\sqrt{\pi}}(d^{\frac{q}{2}}\|\Sigma\|^{q}+\|\mu\|^{2q})\Big)^{\frac{1}{2}}\nonumber\\
        &\leq&\widetilde{C}(q,d)(\|\mu\|^{q}+\|\Sigma\|^{\frac{q}{2}}),
\end{eqnarray}
where third inequality holds due to $\Sigma_{jj}\leq \sqrt{d}\|\Sigma\|$ and $\mu_j\leq \|\mu\|$.
\end{proof}

\begin{lemma}\label{thm:wasserstein local lipschitz}
Suppose a function $U: \mathbb{R}^{d}\mapsto \mathbb{R}$ has the following property:
\begin{eqnarray}\label{eq:local lipschitz condition}
    |U(x_1)-U(x_2)|\leq \breve{C}(1+\|x_1\|^{m}+\|x_2\|^{m})\|x_1-x_2\|,
\end{eqnarray}
where $\breve{C}$ is a constant. Then it holds that
\begin{eqnarray}
    &&\left|\mathbb{E}_{X \sim \bar{T}_{h}^{k}(\cdot|B)}[U(X)]- \mathbb{E}_{Y \sim {T}_{h}(\cdot|x,a)}[U(Y)]\right|\nonumber\\
    &\leq& L_{U}(B,x,a)
   \mathcal{W}_{2}\Big(\NN(\widehat{\mu}_{h}^{k}(B)\Delta,\widehat{\Sigma}_{h}^{k}(B)\Delta),\NN(\mu_{h}(x,a)\Delta,\Sigma_{h}(x,a)\Delta)\Big),\label{eq:wasserstein local lipschitz}
\end{eqnarray}
where 
\begin{eqnarray*}
    L_{U}(B,x,a):&=&\sqrt{3}\breve{C}\Big(1+\widetilde{C}(m,d)(\|\widehat{\mu}_{h}^{k}(B)\|^{m}\Delta^{m}\\
    &&+\|\widehat{\Sigma}_{h}^{k}(B)\|^{\frac{m}{2}}\Delta^{\frac{m}{2}}
+\|\mu_{h}(x,a)\|^{m}\Delta^{m}+\|\Sigma_{h}(x,a)\|^{\frac{m}{2}}\Delta^{\frac{m}{2}})\Big).
\end{eqnarray*}
\end{lemma}

\begin{proof}

    \begin{eqnarray}\label{eq:value function local lipschitz tucb wasserstein I}
      &&\left|\mathbb{E}_{X \sim \bar{T}_{h}^{k}(\cdot|B)}[U(X)]- \mathbb{E}_{Y \sim {T}_{h}(\cdot|x,a)}[U(Y)]\right|\\
     &\leq& \mathbb{E}_{X \sim \bar{T}_{h}^{k}(\cdot|B),Y \sim {T}_{h}(\cdot|x,a)}[\left|U(X)-U(Y)\right|]\nonumber\\
     &\leq&  \breve{C}\mathbb{E}_{X \sim \bar{T}_{h}^{k}(\cdot|B),Y \sim {T}_{h}(\cdot|x,a)}\Big[\Big(1+\|X\|^{m}+\|Y\|^{m}\Big)\Big(\|X-Y\|\Big)\Big]\nonumber\\
     &\leq& \breve{C}\Big(\mathbb{E}_{X \sim \bar{T}_{h}^{k}(\cdot|B),Y \sim {T}_{h}(\cdot|x,a)}
     [1+\|X\|^{2m}+\|Y\|^{2m}]\Big)^{\frac{1}{2}}\nonumber\\
     &&\quad\times\Big(\mathbb{E}_{X \sim \bar{T}_{h}^{k}(\cdot|B),Y \sim {T}_{h}(\cdot|x,a)}
     [\|X-Y\|^{2}]\Big)^{\frac{1}{2}}\nonumber\\
      &\leq & \sqrt{3}\breve{C}\Big(1+(\mathbb{E}_{X \sim \bar{T}_{h}^{k}(\cdot|B)}[\|X\|^{2m}])^{\frac{1}{2}}\nonumber\\
      &&\quad+(\mathbb{E}_{Y \sim {T}_{h}(\cdot|x,a)}[\|Y\|^{2m}])^{\frac{1}{2}}\Big)\nonumber\\
      &&\quad\times\Big(\mathbb{E}_{X \sim \bar{T}_{h}^{k}(\cdot|B),Y \sim {T}_{h}(\cdot|x,a)}
      [\|X-Y\|^{2}]\Big)^{\frac{1}{2}}\nonumber\\
      &\leq & \sqrt{3}\breve{C}\Big(1+\widetilde{C}(m,d)
      (\|\widehat{\mu}_{h}^{k}(B)\|^{m}\Delta^{m}
      +\|\widehat{\Sigma}_{h}^{k}(B)\|^{\frac{m}{2}}\Delta^{\frac{m}{2}}\nonumber\\
      &&\quad+\|\mu_{h}(x,a)\|^{m}\Delta^{m}
      +\|\Sigma_{h}(x,a)\|^{\frac{m}{2}}\Delta^{\frac{m}{2}})\Big)\nonumber\\
      &&\quad\times \Big({E}_{X \sim \bar{T}_{h}^{k}(\cdot|B),Y \sim {T}_{h}(\cdot|x,a)}
      \|X-Y\|^{2}\Big)^{\frac{1}{2}}\nonumber,
\end{eqnarray}

where the second inequality holds due to \eqref{eq:local lipschitz condition}, the third inequality holds by H\"older's inequality and the last inequality holds due to \eqref{eq:2-norm mulitinormal estimate}.

Note that $\eqref{eq:value function local lipschitz tucb wasserstein I}$ holds for any joint distribution (coupling) of $\bar{T}_{h}^{k}(\cdot|B)$ and ${T}_{h}(\cdot|x,a)$, hence we may choose the one which can minimize
\[
\Big(\mathbb{E}_{X \sim \bar{T}_{h}^{k}(\cdot|B),Y \sim {T}_{h}(\cdot|x,a)}[\|X-Y\|^{2}]\Big)^{\frac{1}{2}}.
\]
Then we obtain the following:
\begin{eqnarray*}
   && \left|\mathbb{E}_{X \sim \bar{T}_{h}^{k}(\cdot|B)}[U(X)]- \mathbb{E}_{Y \sim {T}_{h}(\cdot|x,a)}[U(Y)]\right|\nonumber\\
   &\leq& L_{U}(B,x,a)\,
\mathcal{W}_{2}\Big(\NN(\widehat{\mu}_{h}^{k}(B)\Delta,\widehat{\Sigma}_{h}^{k}(B)\Delta),\NN(\mu_{h}(x,a)\Delta,\Sigma_{h}(x,a)\Delta)\Big).
\end{eqnarray*}
\end{proof}

Then with the two lemmas above, we are ready to provide the  proof for Theorem \ref{thm:transition kernel wasserstein local lipschitz all together}. 

\begin{proof}
    From \eqref{eq:value function local Lipschitz}, we know $V_{h+1}^{*}$ has local lipschitz property required to apply Theorem \ref{thm:wasserstein local lipschitz}, so \eqref{eq:wasserstein local lipschitz} holds for $U=V_{h+1}^{*}$ with $\breve{C}=\overline{C}_{h+1}$.

   For the drift term, with probability at least $1-\delta$,  it holds that,  $\forall (h,k) \in [H-1]\times[K]$ and $\forall B\in \mathcal{P}_h^k$ with $n_h^k(B)>0$,
    \begin{eqnarray}\label{eq:drift estimator norm upper bound}
            \|\widehat{\mu}_{h}^{k}(B)\|^{m}&\leq&\Big(\|\widehat{\mu}_{h}^{k}(B)-\bar{\mathbb{E}}[\widehat{\mu}_{h}^{k}(B)]\|+\|\bar{\mathbb{E}}[\widehat{\mu}_{h}^{k}(B)]\|\Big)^{m}\nonumber\\
    &\leq&2^{m}\Big(\|\widehat{\mu}_{h}^{k}(B)-\bar{\mathbb{E}}[\widehat{\mu}_{h}^{k}(B)]\|^{m}+\|\bar{\mathbb{E}}[\widehat{\mu}_{h}^{k}(B)]\|^{m}\Big)\nonumber\\
            &\leq&2^{m}\Bigg(\kappa_{\mu}(\delta,\|\tilde{x}(^{o}B)\|,n_h^k(B))^{m}+\eta(\|\tilde{x}(^{o}B)\|)^m\Bigg),     
    \end{eqnarray}    
    where the second inequality holds by power-mean inequality. 

Let $Z:=\|\overline{\mathbb{E}}[\widetilde{\Sigma}_h(B)]\|$, we have: 
\begin{eqnarray}\label{eq:volatitlity sum bound}
    Z
&\leq& \frac{\sum_{i=1}^n\|\sigma_h(X_h^{k_i},A_h^{k_i})\|^{2}}{n_h^k(B)}+\frac{\sum_{i=1}^{n_h^k(B)}\|\mu_{h}(X_{h}^{k_{i}},A_{h}^{k_{i}})-\overline{\mathbb{E}}[\widehat{\mu}_{h}^{k}(B)]\|^2}{n_h^k(B)}\Delta\nonumber\\
&\leq&\eta(\|\tilde{x}(^{o}B)\|)^{2}+L^2D^2\Delta,
\end{eqnarray}
where the last inequality holds by \eqref{eq:Lip-bound}. 

Then similar to \eqref{eq:drift estimator norm upper bound}, with probability at least $1-\delta$, it holds that $\forall (h,k) \in [H-1]\times[K]$, $\forall B\in \mathcal{P}_h^k$ with $n_h^k(B)>0$:
\begin{eqnarray}\label{eq:covariance estimator norm upper bound}
                \|\widehat{\Sigma}_{h}^{k}(B)\|^{\frac{m}{2}}&\leq &(\|\widehat{\Sigma}_{h}^{k}(B)-\widetilde{\Sigma}_h(B)\|+\|\widetilde{\Sigma}_h(B)-\overline{\mathbb{E}}[\widetilde{\Sigma}_h(B)]\|+\|\overline{\mathbb{E}}[\widetilde{\Sigma}_h(B)]\|)^{\frac{m}{2}} \nonumber \\
       &\leq& 3^{\frac{m}{2}}\Bigg(\kappa_{\mu}(\delta,\|\tilde{x}(^{o}B)\|,n_h^k(B))^{m}\Delta^{\frac{m}{2}}\nonumber\\
        &&+\kappa_{\Sigma}(\delta,\|\tilde{x}(^{o}B)\|,n_h^k(B))^{\frac{m}{2}}+\Big(\eta(\|\tilde{x}(^{o}B)\|)^{2}+L^2D^2\Delta\Big)^{\frac{m}{2}}\Bigg),
    \end{eqnarray}
where the second inequality holds due to \eqref{eq:volatility estimator difference} and Proposition \ref{lemma:Tail Estimates for Gaussian Sample Covariance}. 

    Therefore, with probability at least $1-2\delta$ , it holds that $\forall (h,k) \in [H-1]\times[K]$, $\forall B\in \mathcal{P}_h^k$ with $n_h^k(B)>0$, and $\forall (x,a)\in B$:
  \begin{eqnarray}\label{eq:transition kernel wasserstein local lipschitz I}
         &&\left|\mathbb{E}_{X \sim \bar{T}_{h}^{k}(\cdot|B)}[V_{h+1}^{*}(X)]- \mathbb{E}_{Y \sim {T}_{h}(\cdot|x,a)}[V_{h+1}^{*}(Y)]\right|\nonumber\\
         &\leq& \sqrt{3}\,\overline{C}_{h+1}\Big(1+{\widetilde{C}(m,d_{\mathcal{S}})}(\|\widehat{\mu}_{h}^{k}(B)\|^{m}\Delta^{m}+\|\widehat{\Sigma}_{h}^{k}(B)\|^{\frac{m}{2}}\Delta^{\frac{m}{2}}
        +\|\mu_{h}(x,a)\|^{m}\Delta^{m}\nonumber \\
         &&\quad+\|\Sigma_{h}(x,a)\|^{\frac{m}{2}}\Delta^{\frac{m}{2}})\Big)\nonumber \times \mathcal{W}_{2}\Big(\NN(\widehat{\mu}_{h}^{k}(B)\Delta,\widehat{\Sigma}_{h}^{k}(B)\Delta),\NN(\mu_{h}(x,a)\Delta,\Sigma_{h}(x,a)\Delta)\Big)\nonumber\\
         &\leq& \sqrt{3}\,\overline{C}_{\max}\Bigg(1+{\widetilde{C}(m,d_{\mathcal{S}})}\Bigg( 2^{m}\Big(\kappa_{\mu}(\delta,\|\tilde{x}(^{o}B)\|,n_h^k(B))^{m}+\eta(\|\tilde{x}(^{o}B)\|)^m\Big)\Delta^{m} \nonumber\\
         &&\quad+3^{\frac{m}{2}}\Big(\kappa_{\mu}(\delta,\|\tilde{x}(^{o}B)\|,n_h^k(B))^{m}\Delta^{\frac{m}{2}}+\kappa_{\Sigma}(\delta,\|\tilde{x}(^{o}B)\|,n_h^k(B))^{\frac{m}{2}}\nonumber
         \\&&\quad +\Big(\eta(\|\tilde{x}(^{o}B)\|)^2+L^2D^2\Delta\Big)^{\frac{m}{2}}\Big)\Delta^{\frac{m}{2}}\nonumber\\
&&\quad+\eta(\|\tilde{x}(^{o}B)\|)^m\Delta^{m}+\eta(\|\tilde{x}(^{o}B)\|)^m\Delta^{\frac{m}{2}}\Bigg)\Bigg)\nonumber\\
        &&\quad \times \Big(\kappa_{\mu}(\delta,\|\tilde{x}(^{o}B)\|,n_h^k(B))\Delta+ \kappa_{\mu}(\delta,\|\tilde{x}(^{o}B)\|,n_h^k(B))^{2}\frac{\Delta^{\frac{3}{2}}}{\sqrt{\lambda}} \nonumber\\
        &&\quad+\kappa_{\Sigma}(\delta,\|\tilde{x}(^{o}B)\|,n_h^k(B))\frac{\sqrt{d_{\mathcal{S}}}\Delta^{\frac{1}{2}}}{\sqrt{\lambda}}+\mbox{\rm T-BIAS}(B)\Big)\nonumber\\
        &\leq& \mbox{\rm T-UCB}_{h}^{k}(B)+ L_{V}(\delta, \|\tilde{x}(^{o}B)\|)\,\,\mbox{\rm T-BIAS}(B),
    \end{eqnarray}
    where the first inequality holds due to Lemma \ref{thm:wasserstein local lipschitz}. In addition, the second inequality holds due to \eqref{eq:drift estimator norm upper bound},  \eqref{eq:covariance estimator norm upper bound}, the power-mean inequality and \eqref{eq:high_prob_w_bound}. Finally, the third inequality holds since $\sqrt{n_h^k(B)}\geq 1$. 

\end{proof} 

\subsection{Proof of Theorem \ref{app:reward high prob bound 1}}\label{app:reward high prob bound 1}
\begin{proof}
For fixed $h,k$ and $B\in \mathcal{P}_h^k$ such that $n_h^k(B)>0$, by sub-Gaussian assumption of expected reward in \eqref{ass-eqn:subGaussian reward}, we have:
\begin{eqnarray*}
\overline{\mathbb{P}}\Bigg(\left|\frac{\sum_{i=1}^{n_h^k(B)}r_{h}^{k_{i}}}{n_h^k(B)}-\frac{\sum_{i=1}^{n_h^k(B)}\bar{R}_{h}(X_{h}^{k_{i}},A_{h}^{k_{i}})}{n_h^k(B)}\right|\geq \mbox{\rm R-UCB}_{h}^{k}(B)\Bigg)\leq \frac{\delta}{HK^{2}}.    
\end{eqnarray*}

Taking expectations we get: 
\begin{eqnarray*}
\mathbb{P}\Bigg(\left|\frac{\sum_{i=1}^{n_h^k(B)}r_{h}^{k_{i}}}{n_h^k(B)}-\frac{\sum_{i=1}^{n_h^k(B)}\bar{R}_{h}(X_{h}^{k_{i}},A_{h}^{k_{i}})}{n_h^k(B)}\right|\geq \mbox{\rm R-UCB}_{h}^{k}(B)\Bigg)\leq \frac{\delta}{HK^{2}}.    
\end{eqnarray*}

Then we have: 

\begin{eqnarray} \label{eq:union bound reward}
&&\mathbb{P}\Bigg(\cap_{h=1}^{H}\cap_{k=1}^{K}
\cap_{B\in\mathcal{P}_h^k,n_h^k(B)>0}\Bigg\{\nonumber\\
&&\quad\left|\frac{\sum_{i=1}^{n_h^k(B)}r_{h}^{k_{i}}}{n_h^k(B)}
-\frac{\sum_{i=1}^{n_h^k(B)}\bar{R}_{h}(X_{h}^{k_{i}},A_{h}^{k_{i}})}{n_h^k(B)}\right|\nonumber\\
&&\quad\leq \mbox{\rm R-UCB}_{h}^{k}(B)\Bigg\}\Bigg) \nonumber \\
        &=& \mathbb{P}\Bigg(\cap_{h=1}^{H}\cap_{k=1}^{K}
        \cap_{n_h^k(B_h^k)=1}^{K}\Bigg\{\nonumber\\
        &&\quad\left|\frac{\sum_{i=1}^{n_h^k(B_h^k)}r_{h}^{k_{i}}}{n_h^k(B_h^k)}
        -\frac{\sum_{i=1}^{n_h^k(B_h^k)}\bar{R}_{h}(X_{h}^{k_{i}},A_{h}^{k_{i}})}{n_h^k(B_h^k)}\right|\nonumber\\
        &&\quad\leq \sqrt{\frac{\log(\frac{2HK^{2}}{\delta})\,\theta}{n_h^k(B_h^k)}}
        \Bigg\}\Bigg)\nonumber\\
        &\geq& 1- \sum_{h=1}^{H}\sum_{k=1}^{K}
        \sum_{n_h^k(B_h^k)=1}^{K}\nonumber\\
        &&\quad\mathbb{P}\Bigg(\left|\frac{\sum_{i=1}^{n_h^k(B_h^k)}r_{h}^{k_{i}}}{n_h^k(B_h^k)}
        -\frac{\sum_{i=1}^{n_h^k(B_h^k)}\bar{R}_{h}(X_{h}^{k_{i}},A_{h}^{k_{i}})}{n_h^k(B_h^k)}\right|\nonumber\\
        &&\quad\geq \sqrt{\frac{\log(\frac{2HK^{2}}{\delta})\,\theta}{n_h^k(B_h^k)}}\Bigg)\nonumber\\
        &\geq& 1-\delta,
\end{eqnarray}

where $B_h^k$ is selected according to Algorithm \ref{alg:ML 2}. The first equality holds since {\it only} the estimate of selected block $B_h^k$ is updated for each $(h,k)$ pair. The first inequality holds since for a countable sets of events $\{E_i\}$ we have $\mathbb{P}(\cap_{i}E_i)\geq 1-\sum_{i}\mathbb{P}(E_i^{\complement})$. 

Furthermore, we have
\begin{eqnarray}
        &&|\widehat{R}_{h}^{k}(B)-\bar{R}_h(x,a)|\nonumber \\ 
        &\leq& \left|\frac{\sum_{i=1}^{n_h^k(B)}r_{h}^{k_{i}}}{n_h^k(B)}-\frac{\sum_{i=1}^{n_h^k(B)}\bar{R}_{h}(X_{h}^{k_{i}},A_{h}^{k_{i}})}{n_h^k(B)}\right|+\left|\frac{\sum_{i=1}^{n_h^k(B)}\bar{R}_{h}(X_{h}^{k_{i}},A_{h}^{k_{i}})}{n_h^k(B)}-\bar{R}_h(x,a)\right|.\label{eq:triangle reward}
\end{eqnarray} 
Combine \eqref{eq:union bound reward} and \eqref{eq:triangle reward}, we verify that the desirable result in \eqref{eq:reward high prob bound} holds.
\end{proof}

\subsection{Proof of Lemma \ref{thm:quoted cumulative bound lemma}}\label{app:proof-4-10}

\begin{proof}
For $B\in \mathcal{P}_h^k$, $h \in [H]$ and $k \in J_{\rho}^{K}$:
\begin{eqnarray}\label{eq:CONF splitting bound parent}
        {\rm CONF}_h^{k}(B)&=&\frac{g_{1}(\delta, \|\tilde{x}(^{o}B)\|)}{\sqrt{n_h^k(B)}}\nonumber\\
        &\leq&\frac{g_{1}(\delta,\|\tilde{x}(^{o}{\rm par}(B))\|)}{\sqrt{n_h^k({\rm par}(B))}}= {\rm CONF}_h^{k}({\rm par}(B))\nonumber\\
        &\leq&{\rm diam}({\rm par}(B))= 2\,\,{\rm diam}(B),
\end{eqnarray}
where ${\rm par}(B)$ is the parent block of $B$ and we  use the fact that $^{o}B =\, ^{o}{\rm par}(B)$. 

Rearranging \eqref{eq:CONF splitting bound parent}, we get:
\begin{eqnarray}\label{eq:n min}
    n_h^k(B)\geq \Big(\frac{g_{1}(\delta, \|\tilde{x}(^{o}B)\|)}{2\,\,{\rm diam}(B)}\Big)^{2}.
\end{eqnarray}

In addition, $n_h^k(B)$ must satisfy ${\rm CONF}_h^k(B)>{\rm diam}(B)$, hence
\begin{eqnarray}\label{eq: n max}
    n_h^k(B)< \Big(\frac{g_{1}(\delta, \|\tilde{x}(^{o}B)\|)}{{\rm diam}(B)}\Big)^{2}.
\end{eqnarray}

Let $l(B)$ be the total number of ancestors of $B$ in the adaptive partition and denote them as $B_{0}, B_{1},..., B_{l(B)-1}$ arranged in descending order of size. Also denote $B$ as $B_{l(B)}$ for consistency.  Then we have 
\begin{eqnarray*}
    \frac{\sum_{i=1}^{n_h^k(B)}{\rm diam}(B_{h}^{k_{i}})}{n_h^k(B)}\leq \frac{\sum_{l=0}^{l(B)-1}|\{k^{\prime}:B_h^{k^{\prime}}=B_l\}|{\rm diam}(B_l)}{\sum_{l=0}^{l(B)-1}|\{k^{\prime}:B_h^{k^{\prime}}=B_l\}|}.
\end{eqnarray*}

By \eqref{eq:n min} and \eqref{eq: n max}: 
\begin{eqnarray*}
    |\{k^{\prime}:B_h^{k^{\prime}}=B_l\}|\leq \left(\frac{g_{1}(\delta, \|\tilde{x}(^{o}B)\|)}{{\rm diam}(B)}\right)^{2}-\left(\frac{g_{1}(\delta, \|\tilde{x}(^{o}B)\|)}{2\,\,{\rm diam}(B)}\right)^{2}= \frac{3}{4}\frac{g_{1}(\delta, \|\tilde{x}(^{o}B)\|) ^{2}}{{\rm diam}(B_l)^{2}}.
\end{eqnarray*}

Note that ${\rm diam}(B_l)= 2^{-l}D$, we have: 
\begin{eqnarray*}
        \frac{\sum_{i=1}^{n_h^k(B)}{\rm diam}(B_{h}^{k_{i}})}{n_h^k(B)}&\leq&\frac{\sum_{l=0}^{l(B)-1}|\{k^{\prime}:B_h^{k^{\prime}}=B_l\}|{\rm diam}(B_l)}{\sum_{l=0}^{l(B)-1}|\{k^{\prime}:B_h^{k^{\prime}}=B_l\}|}\\
        &\leq&\frac{\sum_{l=0}^{l(B)-1}2^{-l}2^{2 l}}{\sum_{l=0}^{l(B)-1}2^{2 l}}D
        \leq 4\times 2^{-l(B)}D
        = 4\,\,{\rm diam}(B).
\end{eqnarray*}

Then since ${\rm diam}(B_{h}^{k_{i}}) \leq D$, we have:
\begin{eqnarray*}
        \frac{\sum_{i=1}^{n_h^k(B)} {\rm diam}(B_{h}^{k_{i}})^{2}}{n_h^k(B)} &\leq&\frac{\sum_{i=1}^{n_h^k(B)} {\rm diam}(B_{h}^{k_{i}})}{n_h^k(B)}D \nonumber\\
        &\leq&4 D\,\ {\rm diam}(B),  
\end{eqnarray*}
where the second inequality holds due to \eqref{eq:cumulative bound I}.
\end{proof}

\subsection{Proof of Theorem \ref{thm:BIAS Bound}}\label{app:BIAS Bound}
\begin{proof}
Recall from Lemma \ref{lemma:conditional distribution of drift term} and Lemma \ref{lemma:conditional distribution of the volatility term} that,   

\begin{eqnarray*}
         \overline{\mathbb{E}}[\widehat{\mu}_{h}^{k}(B)]&=&\frac{\sum_{i=1}^{n_h^k(B)}\mu_{h}(X_{h}^{k_{i}},A_{h}^{k_{i}})}{n_h^k(B)},\\
         &&\overline{\mathbb{E}}[\widetilde{\Sigma}_{h}^{k}(B)]\\
         &=&\frac{1}{n_h^k(B)}\sum_{i=1}^{n_h^k(B)}
         \Big(\Sigma_{h}(X_{h}^{k_{i}},A_{h}^{k_{i}})\nonumber\\
         &&\quad+\big(\mu_{h}(X_{h}^{k_{i}},A_{h}^{k_{i}})
         -\overline{\mathbb{E}}[\widehat{\mu}_{h}^{k}(B)]\big)\nonumber\\
         &&\qquad\times\big(\mu_{h}(X_{h}^{k_{i}},A_{h}^{k_{i}})
         -\overline{\mathbb{E}}[\widehat{\mu}_{h}^{k}(B)]\big)^{\top}\Delta\Big).
\end{eqnarray*}

Then we have,
    \begin{eqnarray}\label{eq: BIAS F0,1,2,3}
        &&\Big\|\overline{\mathbb{E}}[\widehat{\mu}_{h}^{k}(B)]-\mu_{h}(x,a)\Big\|\Delta+\Big\|\overline{\mathbb{E}}[\widetilde{\Sigma}_{h}^{k}(B)]-\Sigma_h(x,a)\Big\|\frac{\sqrt{\Delta}}{\sqrt{\lambda}}\nonumber\\
       &\leq&\,\, \Big\|\frac{\sum_{i=1}^{n_h^k(B)}\mu_{h}(X_{h}^{k_{i}},A_{h}^{k_{i}})}{n_h^k(B)}-\mu_{h}(x,a)\Big\|\Delta
        + \Big\|\frac{\sum_{i=1}^{n_h^k(B)}\Sigma_{h}(X_{h}^{k_{i}},A_{h}^{k_{i}})}{n_h^k(B)}-\Sigma_{h}(x,a)\Big\|\frac{\sqrt{\Delta}}{\sqrt{\lambda}}\nonumber\\
        &&\qquad + \sum_{i=1}^{n_h^k(B)}\frac{\Big\|(\mu_{h}(X_{h}^{k_{i}},A_{h}^{k_{i}})-\mu_{h}(x,a)+\mu_{h}(x,a)-\frac{\sum_{i}^{n_h^k(B)}\mu_{h}(X_{h}^{k_{i}},A_{h}^{k_{i}})}{n_h^k(B)}) \Big\|^{2}}{n_h^k(B)}\frac{\Delta^{\frac{3}{2}}}{\sqrt{\lambda}}\nonumber\\
        &\leq&F_{0}\Delta+(F_{1}+2F_{2}+2F_{3})\frac{\sqrt{\Delta}}{\sqrt{\lambda}},
    \end{eqnarray} 
in which 

\begin{eqnarray}\label{eq:F0,F1,F2,F3}
F_{0}&=&\Big\|\frac{\sum_{i=1}^{n_h^k(B)}\mu_{h}(X_{h}^{k_{i}},A_{h}^{k_{i}})}{n_h^k(B)}
-\mu_{h}(x,a)\Big\|, \\
F_{1}&=&\Big\|\frac{\sum_{i=1}^{n_h^k(B)}\Sigma_{h}(X_{h}^{k_{i}},A_{h}^{k_{i}})}{n_h^k(B)}
-\Sigma_{h}(x,a)\Big\|,\nonumber\\
F_{2}&=&\sum_{i=1}^{n_h^k(B)}\frac{\Big\|\mu_{h}(X_{h}^{k_{i}},A_{h}^{k_{i}})-\mu_{h}(x,a)\Big\|^{2}}{n_h^k(B)}\Delta,\nonumber\\
F_{3}&=&\sum_{i=1}^{n_h^k(B)}\frac{\Big\|\mu_{h}(x,a)
-\frac{\sum_{i}^{n_h^k(B)}\mu_{h}(X_{h}^{k_{i}},A_{h}^{k_{i}})}{n_h^k(B)}\Big\|^{2}}{n_h^k(B)}\Delta.\nonumber
    \end{eqnarray}


 For $F_0$ 
 in \eqref{eq:F0,F1,F2,F3},  
 since $(x,a)$ and $ (X_{h}^{k_{i}},A_{h}^{k_{i}})$ lie in $ B_{h}^{k_{i}}$, we have $\Big\|\mu_{h}(X_{h}^{k_{i}},A_{h}^{k_{i}})-\mu_{h}(x,a)\Big\|\leq L \, {\rm diam}(B_{h}^{k_{i}})$. Hence, 
\begin{eqnarray}\label{eq:F0 bound}
      F_{0}
       \leq \frac{\sum_{i=1}^{n_h^k(B)}\Big\|\mu_{h}(X_{h}^{k_{i}},A_{h}^{k_{i}})-\mu_{h}(x,a)\Big\|}{n_h^k(B)}
       \leq  \frac{2L\sum_{i=1}^{n_h^k(B)} {\rm diam}(B_{h}^{k_{i}})}{n_h^k(B)}.
\end{eqnarray}

For $F_1$, 
we have: 
\begin{eqnarray}\label{eq:F1 bound}
    F_{1}
        &\leq&   \frac{\sum_{i=1}^{n_h^k(B)}\Big\|\Sigma_{h}(X_{h}^{k_{i}},A_{h}^{k_{i}})-\Sigma_{h}(x,a)\Big\|}{n_h^k(B)} \nonumber\\
        &\leq& \frac{\sum_{i=1}^{n_h^k(B)}\Big(\Big\|\sigma_{h}(X_{h}^{k_{i}},A_{h}^{k_{i}})\Big\|+\Big\|\sigma_{h}(x,a)\Big\|\Big)\Big\|\sigma_{h}(X_{h}^{k_{i}},A_{h}^{k_{i}})-\sigma_{h}(x,a)\Big\|}{n_h^k(B)}\nonumber\\
        &\leq& \frac{\sum_{i=1}^{n_h^k(B)}2\eta(\|\tilde{x}(^{o}B)\|)\Big\|\sigma_{h}(X_{h}^{k_{i}},A_{h}^{k_{i}})-\sigma_{h}(x,a)\Big\|}{n_h^k(B)}\nonumber\\
        &\leq& 4L\,\,\eta(\|\tilde{x}(^{o}B)\|)\frac{\sum_{i=1}^{n_h^k(B)}{\rm diam}(B_h^{k_i})}{n_h^k(B)},
\end{eqnarray}
where we used \eqref{eq:Lip-bound} in getting the third inequality of \eqref{eq:F1 bound}.

For $F_{2}$ and $F_{3}$ 
we have
\begin{eqnarray}\label{eq:F2 F3 bound}
  F_{2} \leq \frac{4L^{2}\sum_{i=1}^{n_h^k(B)} {\rm diam}(B_{h}^{k_{i}})^{2}}{n_h^k(B)}, \quad F_{3}\leq \left(2L\frac{\sum_{i=1}^{n_h^k(B)} {\rm diam}(B_{h}^{k_{i}})}{n_h^k(B)}\right)^{2},
\end{eqnarray}
Combining \eqref{eq: BIAS F0,1,2,3}, \eqref{eq:F0 bound}, \eqref{eq:F1 bound}, \eqref{eq:F2 F3 bound}, \eqref{eq:cumulative bound I}, 
we get \eqref{eq:BIAS BOUND}.
\end{proof}

\subsection{Proof of Theorem \ref{thm:RBIAS bound}}\label{app:RBIAS bound}
\begin{proof}

\begin{eqnarray}
 &&\left|\frac{\sum_{i=1}^{n_h^k(B)}\bar{R}_{h}(X_{h}^{k_{i}},A_{h}^{k_{i}})}{n_h^k(B)}
 -\bar{R}_h(x,a)\right|\nonumber\\
 &\leq&\frac{\sum_{i=1}^{n_h^k(B)}|\bar{R}_{h}(X_{h}^{k_{i}},A_{h}^{k_{i}})-\bar{R}_h(x,a)|}{n_h^k(B)}\nonumber\\
 &\leq&\frac{L}{n_h^k(B)}\sum_{i=1}^{n_h^k(B)}(1+\|X_h^{k_i}\|^m+\|x\|^m)\nonumber\\
 &&\quad\times(\|X_h^{k_{i}}-x\|+\|A_h^{k_{i}}-a\|)\nonumber\\
        &\leq&2L\Big(1+2(\|\tilde{x}(^{o}B)\|+D)^m\Big)\sum_{i=1}^{n_h^k(B)}\frac{{\rm diam}(B_{h}^{k_{i}})}{n_h^k(B)}\nonumber\\
        &\leq& \mbox{\rm R-BIAS}(B).\nonumber
\end{eqnarray}

Here, we apply \eqref{eq:cumulative bound I} to get the last inequality.
\end{proof}

\section{Technical details in Section \ref{sec:5}}

\subsection{Proof of Theorem \ref{thm: True Upper Bounded By Estimates}}\label{app:proof-5-3}

\begin{proof} 
Recall from Theorems \ref{thm:transition kernel wasserstein local lipschitz all together} and \ref{thm:reward high prob bound}, we know that with probability at least $1-3\delta$,
\begin{eqnarray}\label{eq:simultanuous_result}
   \eqref{eq:transition kernel wasserstein local lipschitz all together}  \mbox{ and }   \eqref{eq:reward high prob bound} \mbox{ hold simultaneously}, 
\end{eqnarray}
 This fact serves as a building  block of the proof.\\
   

For $k=0$, with the initialization of $(\overline{Q}_{h}^{0}$, $\widetilde{V}_h^0)_{h\in [H]}$ in  \eqref{eq:initial value for estimation}, we know \eqref{eq:upperbound} holds.
Now assume that \eqref{eq:upperbound} holds for $k-1$ and we prove it holds for $k$.\\

\noindent \underline{For the case of $h=H$}: 
   For $B \in \mathcal{P}_H^k$ with $n_H^k(B)>0$ and for any $(x,a) \in B$, note that by \eqref{eq:reward high prob bound} and \eqref{eq:RBIAS bound} : 
    \begin{eqnarray}
          \widehat{R}_H^k(B)-\bar{R}_{H}(x,a) \geq -\mbox{\rm R-UCB}_H^k(B)-\mbox{\rm R-BIAS}(B). 
    \end{eqnarray}

    Therefore
$
          \overline{Q}_{H}^{k}(B)=\widehat{R}_{H}^{k}(B)+\mbox{\rm R-UCB}_{H}^{k}(B)+\mbox{\rm R-BIAS}(B)
          \geq Q_{H}^{*}(x,a).        
$    For $B\in \mathcal{P}_H^k$ with $n_H^k(B)=0$, by \eqref{eq:initial value for estimation}, 
                  we have $\overline{Q}_{H}^{k}(B)=\overline{Q}_{H}^{0}(B)\geq Q_{H}^{*}(x,a).$ So we proved the first inequality in \eqref{eq:upperbound}.

For any $S \in \Gamma_{\mathcal{S}}(\mathcal{P}_H^k)$ and any $x\in S$, 
we have $\widetilde{V}_{H}^{k-1}(S)\ge V^*_H(x)$ by induction. 
Furthermore, 
   \begin{eqnarray*}
           \widetilde{V}_{H}^{k}(S)=\max_{B \in \mathcal{P}_H^k, \Gamma_{\mathcal{S}}(B) \supset S}\overline{Q}_{H}^{k}(B)\geq\overline{Q}_{H}^{k}(B^{*})\geq  Q_{H}^{*}(x,a_{H}^{*}(x))=V_{H}^{*}(x),
   \end{eqnarray*}
where $B^{*} \in \mathcal{P}_H^k$ is defined such that $(x,a_{H}^{*}(x)) \in B^{*}$. Hence 
we have $\widetilde{V}_{H}^{k}(S)\ge V^*_H(x)$. 

Finally, we check $\bar V^k_H(x)\ge V_{H}^*(x)$.  For any $x \in \mathcal{S}_1$, there exits  some $S^{\prime}\in \Gamma_{\mathcal{S}}(\mathcal{P}_H^k)$ such that
   \begin{eqnarray*}
           \overline{V}_{H}^{k}(x)&=&\widetilde{V}_{h}^{k}(S^{\prime})+C_{H}(1+\|x\|^m+\|\tilde{x}(S^{\prime})\|^m)\|x-\tilde{x}(S^{\prime})\|\\
           &\geq& V_{H}^{*}(\tilde{x}(S^{\prime}))+C_{H}(1+\|x\|^m+\|\tilde{x}(S^{\prime})\|^m)\|x-\tilde{x}(S^{\prime})\|\\
           &\geq& V_{H}^{*}(x),
   \end{eqnarray*}
where the last inequality holds since $|V_{H}^{*}(\tilde{x}(S^{\prime}))-V_{H}^{*}(x)|\leq C_{H}(1+\|x\|^m+\|\tilde{x}(S^{\prime})\|^m)\|x-\tilde{x}(S^{\prime})\|$ by \eqref{eq:value function local Lipschitz}.

For $x\in \mathbb{R}^{d_{\mathcal{S}}}\setminus \mathcal{S}_{1}$, by \eqref{eq:Q,V estimation for h=H outside}, we know that 
\begin{eqnarray*}
        \overline{V}_{H}^{k}(x)&=&\overline{V}_{H}^{k}\left(\frac{\rho}{\|x\|}x
        \right)+C_{H}(1+\|x\|^{m}+\rho^{m})\left\|\left(1-\frac{\rho}{\|x\|}\right)x\right\|\\
        &\geq& V_{H}^{*}(\frac{\rho}{\|x\|}x)+C_{H}(1+\|x\|^{m}+\rho^{m})\left\|(1-\frac{\rho}{\|x\|})x\right\|\\
        &\geq& V_{H}^{*}(x),
\end{eqnarray*}
where the first inequality holds since $\frac{\rho}{\|x\|}x\in \mathcal{S}_{1}$.\\

\noindent \underline{Induction ($h+1 \mapsto h$)}: 
Assume \eqref{eq:upperbound} holds for $h+1$ and we now show it also holds for $h$.


For $B\in \mathcal{P}_h^k$ with $n_h^k(B)>0$, by \eqref{eq:reward high prob bound} and \eqref{eq:RBIAS bound}, we have 
\begin{eqnarray}\label{eq:RUCB lower bound} 
          \widehat{R}_h^k(B)-\bar{R}_{h}(x,a) \geq -\mbox{\rm R-UCB}_h^k(B)-\mbox{\rm R-BIAS}(B). 
    \end{eqnarray}

We also have 
    \begin{eqnarray}\label{eq:TUCB lower bound}
     &&   \mathbb{E}_{X \sim \bar{T}_{h}^{k}(\cdot|B)}[\overline{V}_{h+1}^k(X)]- \mathbb{E}_{X \sim {T}_{h}(\cdot|x,a)}[V_{h+1}^{*}(X)]\nonumber\\
         &\geq& \mathbb{E}_{X \sim \bar{T}_{h}^{k}(\cdot|B)}[V_{h+1}^{*}(X)]- \mathbb{E}_{X \sim {T}_{h}(\cdot|x,a)}[V_{h+1}^{*}(X)]\nonumber\\
         &\geq& -\mbox{\rm T-UCB}_{h}^{k}(B)-L_{V}(\delta, \|\tilde{x}(^{o}B)\|)\mbox{\rm T-BIAS}(B). 
    \end{eqnarray}
The first inequality holds by induction hypothesis on $h+1$ and the second inequality holds due to \eqref{eq:transition kernel wasserstein local lipschitz all together} and \eqref{eq:BIAS BOUND}. 

Combining \eqref{eq:RUCB lower bound} and \eqref{eq:TUCB lower bound},  we have
  \begin{align*}
  \overline{Q}_{h}^{k}(B)
  ={}&\widehat{R}_{h}^{k}(B)+\mbox{\rm R-UCB}_{h}^{k}(B)
  +\mathbb{E}_{X \sim \bar{T}_{h}^{k}(\cdot|B)}[\overline{V}_{h+1}^{k}(X)]\\
  &+\mbox{\rm T-UCB}_{h}^{k}(B)+{\rm BIAS}(B)
  \geq Q_{h}^{*}(x,a).
  \end{align*}

For $B\in \mathcal{P}_h^k$ with $n_h^k(B)=0$, by \eqref{eq:initial value for estimation}, we also have
    \begin{eqnarray}
                  \overline{Q}_{h}^{k}(B)=\overline{Q}_{h}^{0}(B)\geq Q_{h}^{*}(x,a).
    \end{eqnarray}

   For any $S \in \Gamma_{\mathcal{S}}(\mathcal{P}_h^k)$ and any $x\in S$, 
   $\widetilde{V}_{h}^{k-1}(S)\ge V^*_h(x)$ holds by induction, and 
           $$\max_{B \in \mathcal{P}_h^k,\Gamma_{\mathcal{S}}(B) \supset S}\overline{Q}_{h}^{k}(B)\geq\overline{Q}_{h}^{k}(B^{*})\geq  Q_{h}^{*}(x,a_{h}^{*}(x))=V_{h}^{*}(x),$$
where $B^{*} \in \mathcal{P}_h^k$ is the block containing $(x,a_{h}^{*}(x))$. 
Hence $\widetilde{V}_{h}^{k}(S)\ge V^*_h(x).$

  Finally, for any $x\in \mathcal{S}_{1}$,  there exists a $S^{\prime}\in \Gamma_{\mathcal{S}}(\mathcal{P}_h^k)$ such that,
   \begin{eqnarray*}
           \overline{V}_{h}^{k}(x)&=&\widetilde{V}_{h}^{k}(S^{\prime})+C_h(1+\|x\|^m+\|\tilde{x}(S^{\prime})\|^m)\|x-\tilde{x}(S^{\prime})\|\\
           &\geq&V_{h}^{*}(\tilde{x}(S^{\prime}))+C_h(1+\|x\|^m+\|\tilde{x}(S^{\prime})\|^m)\|x-\tilde{x}(S^{\prime})\|\\
           &\geq& V_{h}^{*}(x),      
   \end{eqnarray*}
where the last inequality holds since $|V_{h}^{*}(\tilde{x}(S^{\prime}))-V_{h}^{*}(x)|\leq C_h(1+\|x\|^m+\|\tilde{x}(S^{\prime})\|^m)\|x-\tilde{x}(S^{\prime})\|$ by \eqref{eq:value function local Lipschitz}.

For $x\in \mathbb{R}^{d_{\mathcal{S}}}\setminus \mathcal{S}_{1}$, we know that 
\begin{eqnarray*}
        \overline{V}_{h}^{k}(x)&=&\overline{V}_{h}^{k}\left(\frac{\rho}{\|x\|}x\right)+C_h(1+\|x\|^{m}+\rho^{m})\left\|\left(1-\frac{\rho}{\|x\|}\right)x\right\|\\
        &\geq& V_{h}^{*}(\frac{\rho}{\|x\|}x)+C_h(1+\|x\|^{m}+\rho^{m})\left\|(1-\frac{\rho}{\|x\|})x\right\|\\
        &\geq& V_{h}^{*}(x), 
\end{eqnarray*}
where the first inequality holds since $\frac{\rho}{\|x\|}x\in \mathcal{S}_{1}$.
\end{proof}

\subsection{Proof of Theorem \ref{thm:barV local lipschitz property}}\label{app:proof-5-1}

\begin{proof}
    We divide the proof subject into three cases.

\vspace{5pt}
\noindent\underline{Case {\bf (1)} $\|x_1\|\leq \rho, \|x_2\|\leq\rho$.}  
  Without lose of generality, let us assume $\|x_1\|\geq \|x_2\|$. 

For $i=1,2$, define $\overline{S}_i := \arg \min_{S \in \Gamma_{S}(\mathcal{P}_h^k)}V_{h,k}^{\rm local}(x_i,S)$, and denote $\widetilde{S}_{i}$ as the state block such that  $x_{i} \in \widetilde{S}_{i}$ and $\widetilde{S}_i \in \Gamma_{\mathcal{S}}(\mathcal{P}_h^k) $. 
{{
Then we  have
\begin{eqnarray}\label{eq:V_local_ineq}
    \overline{V}_{h}^{k} (x_i)=V_{h,k}^{\rm local}(x_i,\overline{S}_i) \leq V_{h,k}^{\rm local}(x_i,\tilde S_i).
\end{eqnarray}

\noindent By the last inequality,  we have 

\begin{eqnarray}\label{eq:bar V linear bound}
 &&\|x_i-\tilde{x}(\overline{S}_i)\| \nonumber\\
 &\leq& \frac{\widetilde{V}_h^k(\widetilde{S}_i)+C_h \Big(1+\|x_i\|^m+\|\tilde{x}(\widetilde{S}_i)\|^m\Big)\|x_i-\tilde{x}(\widetilde{S}_i)\|- \widetilde{V}_{h}^{k}(\overline{S}_i)}{C_h \Big(1+\|x_i\|^m+\|\tilde{x}(\overline{S}_i)\|^m\Big)} \nonumber\\
   &\leq& \frac{\Big|\widetilde{V}_h^k(\widetilde{S}_i)\Big|+\Big|\widetilde{V}_{h}^{k}(\overline{S}_i)\Big|+C_h \Big(1+\|x_i\|^m+\|\tilde{x}(\widetilde{S}_i)\|^m\Big)D}{C_h \Big(1+\|x_i\|^m+\|\tilde{x}(\overline{S}_i)\|^m\Big)}\nonumber\\
   &\leq& \frac{\widetilde{C}_{h}\Big(2+(\|x_i\|+2D)^{m+1}
   +(\|\tilde{x}(\overline{S}_i)\|+2D)^{m+1}\Big)}
   {C_h \Big(1+\|x_i\|^m+\|\tilde{x}(\overline{S}_i)\|^m\Big)}\nonumber\\
   &&\quad+\frac{C_h \Big(1+\|x_i\|^m+(\|x_i\|+D)^m\Big)D}
   {C_h \Big(1+\|x_i\|^m+\|\tilde{x}(\overline{S}_i)\|^m\Big)}\nonumber\\
   &\leq& \frac{\widetilde{C}_{h}\Big(2+2^{m}\|x_i\|^{m+1}+2^{m}\|\tilde{x}(\overline{S}_i)\|^{m+1}+2^{2m+1}D^{m+1}\Big)}{C_h \Big(1+\|x_i\|^m+\|\tilde{x}(\overline{S}_i)\|^m\Big)}\nonumber\\
      &&\,\,+\frac{C_h \Big(1+(2^{m-1}+1)\|x_i\|^m+2^{m-1}D^m\Big)D}{C_h \Big(1+\|x_i\|^m+\|\tilde{x}(\overline{S}_i)\|^m\Big)}\nonumber\\
   &\leq& \frac{1}{2}(\|x_i\|+\|\tilde{x}(\overline{S}_i)\|)+\breve{c}_{0},
\end{eqnarray}

where $\breve{c}_{0}$  a positive constant depending on $D,m,C_h,\widetilde{C}_h$. 
The first inequality holds due to \eqref{eq:V_local_ineq} and the definition of $V_{h,k}^{\rm local}(.,.)$ in \eqref{V-estimate}. The third inequality holds with probability at least $1-3\delta$ due to \eqref{eq:initial value for estimation}, Theorem \ref{thm: True Upper Bounded By Estimates}, and the fact that $\|x-\tilde{x}(\widetilde{S}_i)\|\leq D$. The fourth inequality holds due to the power-mean inequality. The last inequality holds by the fact that $C_h\geq 2^{m+1}\widetilde{C}_h$. 

By the triangle inequalty  $\|\tilde{x}(\overline{S}_i)\|-\|x_i\|\leq  \|x_i-\tilde{x}(\overline{S}_i)\| $  and \eqref{eq:bar V linear bound}, we have 
\begin{eqnarray}\label{eq:bar V linear bound II}
    \|\tilde{x}(\overline{S}_i)\|\leq 3\|x_i\|+ 2\breve{c}_{0}.
\end{eqnarray}

Now we are ready to  bound $|\overline{V}_{h}^{k}(x_1)-\overline{V}_{h}^{k}(x_2)|$ by two terms.
\begin{eqnarray}\label{eq:barV local lipschitz}
     &&|\overline{V}_{h}^{k}(x_1)-\overline{V}_{h}^{k}(x_2)|  \\
     &=&(\overline{V}_{h}^{k}(x_1)-\overline{V}_{h}^{k}(x_2))\mathbb{I}_{\{\overline{V}_{h}^{k}(x_1)-\overline{V}_{h}^{k}(x_2)\geq 0\}}+ (\overline{V}_{h}^{k}(x_2)-\overline{V}_{h}^{k}(x_1))\mathbb{I}_{\{\overline{V}_{h}^{k}(x_1)-\overline{V}_{h}^{k}(x_2)< 0\}}\nonumber\\
     &\leq&\underbrace{\Big|V_{h,k}^{\rm local}(x_1,\overline{S}_2)- V_{h,k}^{\rm local}(x_2,\overline{S}_2)\Big|}_{(I)}+ \underbrace{\Big|V_{h,k}^{\rm local}(x_2,\overline{S}_1)- V_{H,k}^{\rm local}(x_1,\overline{S}_1)\Big|}_{(II)},\nonumber
\end{eqnarray}
 where the inequality holds due to \eqref{eq:V_local_ineq}.

For term (I), 

\begin{eqnarray}\label{eq:bar V term I bound}
        &&\Big|V_{h,k}^{\rm local}(x_1,\overline{S}_2)- V_{h,k}^{\rm local}(x_2,\overline{S}_2)\Big| \\     
   &\leq& C_h\Bigg(1+\|\tilde{x}(\overline{S}_2)\|^m\Bigg)\|x_1-x_2\|\nonumber\\
   &&\quad+C_h\Bigg|\|x_1\|^m\|x_1-\tilde{x}(\overline{S}_2)\|
   -\|x_2\|^m\|x_2-\tilde{x}(\overline{S}_2)\|\Bigg|\nonumber\\
    &\leq & C_h\Bigg(1+(3\|x_2\|+ 2\breve{c}_{0})^m\Bigg)\|x_1-x_2\|\nonumber\\
    &&\quad+\|x_1\|^m\|x_1-x_2\|
    +(4\|x_2\|+2\breve{c}_{0})\Big| \|x_1\|^m-\|x_2\|^m\Big|.\nonumber 
\end{eqnarray}

where the first inequality holds by triangle inequality and the second inequality holds due to \eqref{eq:bar V linear bound II}. 

Similarly,  for term (II):

\begin{eqnarray}\label{eq:bar V term II bound}
         &&\Big|V_{h,k}^{\rm local}(x_2,\overline{S}_1)- V_{H,k}^{\rm local}(x_1,\overline{S}_1)\Big|\\
    &\leq & C_h\Bigg(1+(3\|x_1\|+ 2\breve{c}_{0})^m\Bigg)\|x_1-x_2\|\nonumber\\
    &&\quad+\|x_2\|^m\|x_1-x_2\|
    +(4\|x_1\|+2\breve{c}_{0})\Big| \|x_1\|^m-\|x_2\|^m\Big|.\nonumber
\end{eqnarray}

It is clear that $m=0$ is a trivial case, so we only consider $m\ge 1$, with which we have
  $$          \Big|\|x_1\|^m-\|x_2\|^m\Big|\leq (m-1)\|x_1\|^{m-1}\|x_1-x_2\|, 
    \mbox{ and }    \|x_1\|^{m-1}\leq \frac{m-1}{m}\|x_1\|^m+\frac{1}{m}.
    $$

Combine \eqref{eq:barV local lipschitz}, \eqref{eq:bar V term I bound}, \eqref{eq:bar V term II bound} and the facts above, we have: 
\begin{eqnarray}\label{eq:local lip inside}
    |\overline{V}_{h}^{k}(x_1)-\overline{V}_{h}^{k}(x_2)|\leq \breve{c}^{1}_{h}(1+\|x_1\|^m+\|x_2\|^m)\|x_1-x_2\|,
\end{eqnarray}
where $\breve{c}^{1}_{h}$ depends only on $C_h,\wtCh, m, D$. 

}}

\vspace{10pt}

\noindent\underline{Case {\bf (2)} $\|x_1\|> \rho, \|x_2\|\leq\rho$.} 
 In this case,
\begin{eqnarray}\label{eq:local lipschitz inside and outside}
       \Big|\overline{V}_{h}^{k}(x_1)-\overline{V}_{h}^{k}(x_2)\Big| &=& \Big|\overline{V}_{h}^{k}\Big(\frac{\rho}{\|x_1\|}x_1\Big)-\overline{V}_{h}^{k}(x_2)+C_h(1+\|x_1\|^{m}+\rho^m)(\|x_1\|-\rho)\Big| \nonumber\\
       &\leq& \breve{c}^{3}_{h}(1+2\rho^m)\Big\|\frac{\rho}{\|x_1\|}x_1-x_2\Big\|+C_h(1+\|x_1\|^{m}+\rho^m)(\|x_1\|-\rho)\nonumber\\
&\leq&\breve{c}^{4}_{h}\Big(1+\|x_1\|^m+\|x_2\|^m\Big)\|x_1-x_2\|,
    \end{eqnarray}
 where the first inequality holds due to \eqref{eq:Q,V estimation for h=H} ,  \eqref{V-estimate} and \eqref{eq:local lip inside}; the second inequality holds due to $\|\frac{\rho}{\|x_1\|}x_1-x_2\|\leq \|x_1-x_2\|$  and $\|x_1\|-\rho\leq \|x_1-x_2\|$ ; and finally, the third inequality holds since  $\rho^m\leq \|x_1\|^m$. Note that $\breve{c}^{4}_{h}$ depends only on $C_h,\wtCh, m, D$.

\vspace{10pt}

\noindent \underline{Case {\bf (3)}  $\|x_1\|> \rho, \|x_2\|>\rho$.}   In this case,
    \begin{eqnarray} \label{eq:local lipschitz both outside}
        \Big|\overline{V}_{h}^{k}(x_1)-\overline{V}_{h}^{k}(x_2)\Big| &=& \Big|\overline{V}_{h}^{k}\Big(\frac{\rho}{\|x_1\|}x_1\Big)-\overline{V}_{h}^{k}\Big(\frac{\rho}{\|x_2\|}x_2\Big)+C_h(1+\rho^m)(\|x_1\|-\|x_2\|)\nonumber\\
        &&+C_h\|x_1\|^{m}(\|x_1\|-\rho)-C_h\|x_2\|^{m}(\|x_2\|-\rho)\Big|\nonumber\\
        &\leq& \Big|\overline{V}_{h}^{k}\Big(\frac{\rho}{\|x_1\|}x_1\Big)-\overline{V}_{h}^{k}\Big(\frac{\rho}{\|x_2\|}x_2\Big)\Big|+\Big|C_h(1+\rho^m)(\|x_1\|-\|x_2\|)\Big|\nonumber\\
        &&+\Big|C_h\|x_1\|^{m}(\|x_1\|-\rho)-C_h\|x_2\|^{m}(\|x_2\|-\rho)\Big\|\nonumber\\
        &\leq& \breve{c}^{3}_{h}(1+2\rho^m)\Big\|\frac{\rho}{\|x_1\|}x_1-\frac{\rho}{\|x_2\|}x_2\Big\|+C_h(1+\rho^m)\Big|\|x_1\|-\|x_2\|\Big|\nonumber\\
        &&+C_h\Big|\|x_1\|^{m+1}-\|x_2\|^{m+1}\Big|+\rho C_h\Big|\|x_1\|^{m}-\|x_2\|^{m}\Big|\nonumber\\
&\leq&\breve{c}^{5}_{h}\Big(1+\|x_1\|^m+\|x_2\|^m\Big)\|x_1-x_2\|,
    \end{eqnarray}
where the second inequality holds due to \eqref{eq:Q,V estimation for h=H}, \eqref{V-estimate} and \eqref{eq:local lip inside}. In addition, the third inequality holds due to the facts that $|a^m-b^m|\leq |a-b|(a+b)^{m-1}$ for $a,b\geq 0$ and $\|\frac{\rho}{\|x_1\|}x_1-\frac{\rho}{\|x_2\|}x_2\|\leq \|x_1-x_2\|$. Also,  the fourth inequality holds since $\rho^m\leq\|x_1\|^m+\|x_2\|^m$. $\breve{c}^{5}_{h}$ depends only on $C_h,\wtCh, m, D$.

Finally, let $\hCh=\max\{\breve{c}^{3}_{h},\breve{c}^{4}_{h},\breve{c}^{5}_{h}\}$, combine \eqref{eq:local lip inside}, \eqref{eq:local lipschitz inside and outside} and \eqref{eq:local lipschitz both outside}, we conclude that:
\begin{eqnarray}
    |\overline{V}_{h}^{k}(x_1)-\overline{V}_{h}^{k}(x_2)|\leq \hCh(1+\|x_1\|^m+\|x_2\|^m)\|x_1-x_2\|. 
\end{eqnarray}
\end{proof}

\subsection{Several Useful Properties of Value Estimators}\label{app:value-estimators}

Applying Lemma \ref{thm:wasserstein local lipschitz} in the same fashion as in the proof of Theorem \ref{thm:transition kernel wasserstein local lipschitz all together} yields the following corollary.
\begin{corollary}\label{thm:barV wasserstein local lipschitz all together}
    Assume the same assumptions as  in Theorem \ref{thm:transition kernel wasserstein local lipschitz all together}.   With probability at least $1-2\delta$, for any $ (h,k)\in [H-1]\times [K], B\in\mathcal{P}_h^k$ with $n_h^k(B)>0$, and any $(x,a)\in B$,  we have the following:
 \begin{eqnarray}\label{eq:barV transition} 
     &&\left|\mathbb{E}_{X \sim \bar{T}_{h}^{k}(\cdot|B)}[\overline{V}_{h+1}^k(X)]- \mathbb{E}_{Y \sim {T}_{h}(\cdot|x,a)}[\overline{V}_{h+1}^k(Y)]\right|\nonumber\\
     &\leq&\,\, \frac{\widehat{C}_{\max}}{\overline{C}_{\max}}\Big(\mbox{\rm T-UCB}_{h}^{k}(B)+L_{V}(\delta, \|\tilde{x}(^{o}B)\|)\,\,\mbox{\rm T-BIAS}(B)\Big),
\end{eqnarray}
where $\widehat{C}_{\max}$ is defined in \eqref{eq:G_h^k definition},  $\overline{C}_{\max}$   in \eqref{eq:barCmax}, $\tilde{x}$  in \eqref{eq:center definition}, $^{o}B$   in \eqref{eq:oB definition} and $L_{V}$ in \eqref{eq:Lv definition}.
\end{corollary}

We then bound the difference between the Q-estimators and the true Q-functions in the following Theorem \ref{thm:optimism high prob upper bound } and Proposition \ref{thm:optimism high prob upper bound,n=0}. 

\begin{theorem} \label{thm:optimism high prob upper bound }
Assume Assumptions \ref{ass:lipschitz}-\ref{ass:expected reward local lipschitz} hold. The following inequality holds with probability at least $1-3\delta$, for any $(h,k) \in [H]\times[K]$, $B \in \mathcal{P}_h^k$ with $n_h^k(B)>0$ , and $(x,a) \in B$, 
  \begin{eqnarray}\label{eq:optimism high prob upper bound I}
        \overline{Q}_{H}^{k}(B)-Q_{H}^{*}(x,a)        &\leq& 2\frac{\widehat{C}_{\max}}{\overline{C}_{\max}}\Big(\mbox{\rm R-UCB}_{H}^{k}(B)+\mbox{\rm R-BIAS}(B)\Big); \nonumber\\
     \overline{Q}_{h}^{k}(B)-Q_{h}^{*}(x,a)&\leq& 2\frac{\widehat{C}_{\max}}{\overline{C}_{\max}}\Big(\mbox{\rm R-UCB}_{h}^{k}(B)+ \mbox{\rm T-UCB}_{h}^{k}(B)+{\rm BIAS}(B)\Big)\nonumber\\
     &&+\mathbb{E}_{X \sim {T}_{h}(\cdot|x,a)}[\overline{V}_{h+1}^{k}(X)]-\mathbb{E}_{X \sim {T}_{h}(\cdot|x,a)}[V_{h+1}^{*}(X)] ,h<H.  
    \end{eqnarray}
\end{theorem}

The proof of Theorem \ref{thm:optimism high prob upper bound } is deferred to Appendix \ref{app:proof-5-4}.

\begin{proposition}
    \label{thm:optimism high prob upper bound,n=0}
Assume that Assumptions \ref{ass:lipschitz}-\ref{ass:expected reward local lipschitz} hold. For any $(h,k) \in [H]\times[K]$, $B \in \mathcal{P}_h^k$ with $n_h^k(B)=0$, $(x,a) \in B$,
 the following inequality holds:
\begin{eqnarray}\label{eq:optimism high prob upper bound,n=0}
        \overline{Q}_{h}^{k}(B)-Q_{h}^{*}(x,a)         \leq 2\frac{\widetilde{C}_h}{D}(1+(\|\tilde{x}(^{o}B)\|+D)^{m+1}){\rm diam}(B),
    \end{eqnarray}
where $\widetilde{C}_h$ is defined in \eqref{eq:value function with any policy growth rate}.
\end{proposition}

The proof of Proposition \ref{thm:optimism high prob upper bound,n=0} is deferred to Appendix \ref{app:proof-5-4-1}.

We also have the following bounds on value function estimators evaluated at $X_h^k$.

\begin{proposition}\label{thm: optimism high prob upper bound II}
    For any $(h,k)\in [H]\times[K]$, conditioned on $X_h^k \in \mathcal{S}_1$, we have:
    \begin{eqnarray}\label{eq:optimism high prob upper bound II}
    \overline{V}_{h}^{k-1}(X_h^k)\leq \overline{Q}_h^{k-1}(B_h^k)+C_h(1+2(\|\tilde{x}(^{o}B_h^k)\|+D)^m){\rm diam}(B_h^k),
\end{eqnarray} 
where $B_h^k$ is selected according to Algorithm \ref{alg:ML 2} and $^{o}B_h^k$ is defined in \eqref{eq:oB definition}.
\end{proposition}
The proof of Proposition \ref{thm: optimism high prob upper bound II} is deferred to Appendix \ref{app:proof-5-5-1}. 

\subsubsection{Proof of Theorem \ref{thm:optimism high prob upper bound }} \label{app:proof-5-4}
\begin{proof}

Combine Theorem \ref{thm:reward high prob bound}, Proposition \ref{thm:RBIAS bound}, Corollary \ref{thm:barV wasserstein local lipschitz all together}, and the fact that
$\frac{\widehat{C}_{\max}}{\overline{C}_{\max}}>1$, we have the following result.  With probability at least $1-3\delta$, it holds that $\forall (h,k) \in [H]\times[K]$, $\forall B\in \mathcal{P}_h^k$ with $n_h^k(B)>0$, and $\forall (x,a)\in B$:
\begin{eqnarray}\label{eq:RUCB upper bound}
    \widehat{R}_h^k(B)-\bar{R}_{h}(x,a) \leq \frac{\widehat{C}_{\max}}{\overline{C}_{\max}}\Big( \mbox{\rm R-UCB}_h^k(B)+\mbox{\rm R-BIAS}(B)\Big);
\end{eqnarray}
\begin{eqnarray}\label{eq:TUCB upper bound}
   && \mathbb{E}_{X \sim \bar{T}_{h}^{k}(\cdot|B)}[\overline{V}_{h+1}^{k}(X)]- \mathbb{E}_{X \sim {T}_{h}(\cdot|x,a)}[\overline{V}_{h+1}^{k}(X)] \nonumber\\
    &\leq& \frac{\widehat{C}_{\max}}{\overline{C}_{\max}}\Big(\mbox{\rm T-UCB}_{h}^{k}(B)+ L_{V}(\delta, \|\tilde{x}(^{o}B)\|)\mbox{\rm T-BIAS}(B)\Big).
\end{eqnarray}

Also, we have the following decomposition:

\begin{eqnarray}\label{eq:TUCB bound decomposition}
        &&\mathbb{E}_{X \sim \bar{T}_{h}^{k}(\cdot|B)}[\overline{V}_{h+1}^{k}(X)]- \mathbb{E}_{X \sim {T}_{h}(\cdot|x,a)}[V_{h+1}^{*}(X)] \\
        &=& \mathbb{E}_{X \sim \bar{T}_{h}^{k}(\cdot|B)}[\overline{V}_{h+1}^{k}(X)]- \mathbb{E}_{X \sim {T}_{h}(\cdot|x,a)}[\overline{V}_{h+1}^{k}(X)] \nonumber
    \\&&\quad+ \mathbb{E}_{X \sim {T}_{h}(\cdot|x,a)}[\overline{V}_{h+1}^{k}(X)]
    - \mathbb{E}_{X \sim {T}_{h}(\cdot|x,a)}[V_{h+1}^{*}(X)].\nonumber
\end{eqnarray}

Combining the results in \eqref{eq:RUCB upper bound}, \eqref{eq:TUCB upper bound} and \eqref{eq:TUCB bound decomposition}, it holds with probability at least $1-3\delta$ that, $\forall (h,k) \in [H]\times[K]$, $\forall B\in \mathcal{P}_h^k$ with $n_h^k(B)>0$, and $\forall (x,a)\in B$:
\begin{eqnarray*}
      \overline{Q}_{h}^{k}(B)-Q_{h}^{*}(x,a) &\leq& 2\frac{\widehat{C}_{\max}}{\overline{C}_{\max}}\Big(\mbox{\rm R-UCB}_{h}^{k}(B)+ \mbox{\rm T-UCB}_{h}^{k}(B)+{\rm BIAS}(B)\Big)\\
     &&+\mathbb{E}_{X \sim {T}_{h(\cdot|x,a)}}[\overline{V}_{h+1}^{k}(X)]-\mathbb{E}_{X \sim {T}_{h(\cdot|x,a)}}[V_{h+1}^{*}(X)] ,\,\,  h<H,\\  
    \overline{Q}_{H}^{k}(B)-Q_{H}^{*}(x,a) &\leq & 2\frac{\widehat{C}_{\max}}{\overline{C}_{\max}}\Big(\mbox{\rm R-UCB}_{H}^{k}(B)+\mbox{\rm R-BIAS}(B)\Big). 
\end{eqnarray*}
\end{proof}

\subsubsection{Proof of Proposition \ref{thm:optimism high prob upper bound,n=0}}\label{app:proof-5-4-1}
\begin{proof}
Since $n_h^k(B)=0$, we must have $B\in \mathcal{P}_h^0$, ${\rm diam}(B)=D$, $\overline{Q}_{h}^{k}(B)=\overline{Q}_{h}^{0}(B)$ and $^{o}B=B$. Hence
$$         \overline{Q}_{h}^{k}(B)-Q_{h}^{*}(x,a)
         \leq  \overline{Q}_{h}^{0}(B) + |V_h^{*}(x)|
        \leq 2\frac{\widetilde{C}_h}{D}(1+(\|\tilde{x}(^{o}B)\|+D)^{m+1}){\rm diam}(B),
$$
where the last inequality holds by \eqref{eq:initial value for estimation} and \eqref{eq:value function with any policy growth rate}.

\end{proof}

\subsubsection{Proof of Proposition \ref{thm: optimism high prob upper bound II}}\label{app:proof-5-5-1}
\begin{proof}
Note that conditioned on $X_h^k\in \mathcal{S}_{1}$, we have $B_h^k \in \mathcal{P}_h^{k-1}$. We then divide the proof into two cases: {\bf (1)} $k>1$ and {\bf (2)} $k=1$. 

\vspace{5pt}

\noindent\underline{Case {\bf (1)}.} For $k>1$, we have
    \begin{eqnarray}\label{eq:last property}
        \overline{V}_h^{k-1}(X_h^k)
        &\leq&\widetilde{V}_{h}^{k-1}(\Gamma_{\mathcal{S}}(B_h^k))+C_h(1+2(\|\tilde{x}(^{o}B_h^k)\|+D)^m){\rm diam}(B_h^k)\nonumber\\
        &=& \max_{B \in P_h^{k-1}: \Gamma_{\mathcal{S}}(B_h^k)\subset \Gamma_{\mathcal{S}}(B)}\overline{Q}_h^{k-1}(B)+C_h(1+2(\|\tilde{x}(^{o}B_h^k)\|+D)^m){\rm diam}(B_h^k)\nonumber\\
        &=&\overline{Q}_h^{k-1}(B_h^k)+C_h(1+2(\|\tilde{x}(^{o}B_h^k)\|+D)^m){\rm diam}(B_h^k).
\end{eqnarray}
The first inequality holds by the definition of $\overline{V}_h^{k-1}(X_h^k)$ in \eqref{eq:Q,V estimation for h=H} and \eqref{V-estimate}, and the first equality holds due to the greedy selection rule (line \ref{Greedy selection rule})  in Algorithm \ref{alg:ML}.  

\vspace{5pt}

\noindent \underline{Case {\bf (2)}.}
For $k=1$, we have
\begin{eqnarray*}
    \overline{V}_h^{0}(X_h^1)\leq \overline{Q}_h^{0}(B_h^1)+C_h(1+2(\|\tilde{x}(^{o}B_h^1)\|+D)^m){\rm diam}(B_h^1).
\end{eqnarray*}
This inequality holds due to the initial estimators we set in \eqref{eq:initial value for estimation} and the fact that $^{o}B_h^1=B_h^1$ since $B_h^1 \in \mathcal{P}_h^0$.
\end{proof}

\subsection{Auxiliary Results for Regret composition}\label{app:inside-outside-regret-composition}

\subsubsection{Proof of Theorem \ref{thm: upper bound via clipping}}\label{app:proof-5-5}

\begin{proof}

By the definition in \eqref{eq:gap definition},
\begin{eqnarray}  
   {\rm Gap}_h(B_h^k)&\leq&  {\rm \widetilde{G}ap}_h(X_h^k,A_h^k) \nonumber \\
        &\leq&\overline{V}_h^{k-1}(X_h^k)-Q_h^{*}(X_h^k,A_h^k)\nonumber\\
        &\leq&\overline{Q}_h^{k-1}(B_h^k)-Q_h^{*}(X_h^k,A_h^k)+C_h(1+2(\|\tilde{x}(^{o}B_h^k)\|+D)^m){\rm diam}(B_h^k)\nonumber\\
        &\leq&G_h^{k}(B_h^k)+f_{h+1}^{k-1}(X_h^k, A_h^k):=\phi_1 +\phi_2, \label{eq:gap_B}
\end{eqnarray}
in which the second inequality holds due to Theorem \ref{thm: True Upper Bounded By Estimates} , the third inequality holds by \eqref{eq:optimism high prob upper bound II}. and the fourth inequality holds due to \eqref{eq:optimism high prob upper bound I} and \eqref{eq:optimism high prob upper bound,n=0}. In the last line, we use the simplified notations  $\phi_1:=G_h^{k}(B_h^k)$ and $\phi_2:=f_{h+1}^{k-1}(X_h^k, A_h^k)$.  


Let $\phi:=\overline{V}_h^{k-1}(X_h^k)-Q_h^{*}(X_h^k,A_h^k)$. We claim that 
\begin{eqnarray}\label{eq:clipping inequality}
    \phi\leq {\rm CLIP}\Bigg(\phi_1\Bigg|\frac{{\rm Gap}_h(B_h^k)}{H+1}\Bigg)+\Big(1+\frac{1}{H}\Big)\phi_2.
\end{eqnarray}

When $\phi_1\geq \frac{{\rm Gap}_h(B_h^k)}{H+1}$, \eqref{eq:clipping inequality} is trivial.  
So we only need to prove the claim when 
$\phi_1< \frac{{\rm Gap}_h(B_h^k)}{H+1}$. In this case, 
\begin{eqnarray}\label{eq:phi}
        {\rm Gap}_h(B_h^k)\leq\phi_1+\phi_2
        \leq\frac{{\rm Gap}_h(B_h^k)}{H+1}+\phi_2. 
\end{eqnarray}

Rearranging terms in \eqref{eq:phi}, we have 
    ${\rm Gap}_h(B_h^k)\leq \frac{H+1}{H}\phi_2$, and hence 
    $\phi_1+\phi_2\le \frac{1}{H+1}\frac{H+1}{H}\phi_2+\phi_2=(1+\frac{1}{H})\phi_2.$ 
This implies that 
        $$\phi\leq \phi_1+\phi_2
        \leq {\rm CLIP}\Bigg(\phi_1\Bigg|\frac{{\rm Gap}_h(B_h^k)}{H+1}\Bigg)+\Big(1+\frac{1}{H}\Big)\phi_2.  $$

With the inequality \eqref{eq:clipping inequality},   we have

\begin{eqnarray}\label{eq:Delta h k bound}
        \Delta_{h}^{(k)}&=&\overline{V}_h^{k-1}(X_h^k)-Q_h^{*}(X_h^k,A_h^k)\nonumber\\
        &&\quad+Q_h^{*}(X_h^k,A_h^k)-V_h^{\tilde{\pi}^{k}}(X_h^k)\nonumber\\
        &\leq&{\rm CLIP}\Bigg(G_h^{k}(B_h^k)\Bigg|\frac{{\rm Gap}_h(B_h^k)}{H+1}\Bigg)\nonumber\\
        &&\quad+\Big(1+\frac{1}{H}\Big)f_{h+1}^{k-1}(X_h^k, A_h^k)
        +Q_h^{*}(X_h^k,A_h^k)-V_h^{\tilde{\pi}^{k}}(X_h^k), 
\end{eqnarray}

\end{proof}

\subsubsection{Proof of Proposition \ref{thm:concentration on size of J0}}\label{app:proof-5-10-prop}
\begin{proof}
Let $\mathcal{G}_{k}=\sigma((X_{h}^{k^{\prime}},A_{h}^{k^{\prime}},r_{h}^{k^{\prime}})_{h\in [H]},k^{\prime}\leq k)$ be the information generated up to episode $k$ with $\mathcal{G}_{0}$ being the null information. Then we have $\mathbb{E}[I_{k}|\mathcal{G}_{k-1}] \geq 1-\frac{M_{p}}{\rho^{p}}$ given \eqref{ineq:alpha_p}.

Let  
$Y_{0}=0$ and $ Y_{k}=\sum_{i=1}^{k}(I_{k}-\mathbb{E}[I_{k}|\mathcal{G}_{k-1}] )$ for $k>1$. 
Then it is clear that $\{Y_{k}\}_{k=0,1,...,K}$ is a martingale and we have $|Y_{k}-Y_{k-1}|\leq 1$. By Azuma-Hoeffding inequality, for any $\epsilon>0$  we have 
\begin{eqnarray*}
    \mathbb{P}(Y_{K}-Y_{0}\leq -\epsilon)\leq \exp \left(-\frac{\epsilon^{2}}{2K}\right).
\end{eqnarray*}
By the fact that $Y_K=K_0-K\mathbb{E}[I_k|\mathcal{G}_{k-1}]\le K_0-K\left(1-\frac{M_{p}}{\rho^{p}}\right)$, we have
\begin{equation}\label{eq:K_0_inequality}
    \mathbb{P}\left(K_{0}-K\left(1-\frac{M_{p}}{\rho^{p}}\right)\geq -\epsilon\right)
    \geq \mathbb{P}(Y_{K}-Y_{0}\geq -\epsilon) 
    \geq 1- \exp\left(-\frac{\epsilon^{2}}{2K}\right). 
\end{equation}
 Let $\delta=\exp(-\frac{\epsilon^{2}}{2K})$, then  we have \eqref{thm:concentration on size of J0} hold with probability at least $1-\delta$. 
\end{proof}

\subsubsection{Proof of Theorem \ref{thm:concentration on value function at initial state}}\label{app:proof-5-7}
\begin{proof}
Denote sets $J_{1}$ and $J_{2}$ as the following:
\begin{eqnarray*}
    J_{1}&:=&\left\{k \in [K]:\|X_{1}^{k}\|>\rho \right\},\\
    J_{2}&:=&\left\{k \in [K]:\|X_{1}^{k}\| \leq \rho, \sup_{h=2,...,H}\|X_{h}^{k}\|> \rho \right\}.
\end{eqnarray*}
Then it is clear that $J_{1}\cup J_{2}=[K]\backslash J_{\rho}^{K}$ and $J_{1}\cap J_{2}=\emptyset$. Further denote $K_i=|J_i|$ for $i=1,2$, then 
$K-K_{0}=K_{1}+K_{2}$. 
With these notation, we have
\begin{eqnarray}\label{eq:bounds on value function outside I}
        \sum_{k \in J\backslash J_{\rho}^{K}}|V_{1}^{\pi}(X_{1}^{k})|&=&\sum_{k \in J_{1}}|V_{1}^{\pi}(X_{1}^{k})| +\sum_{k \in J_{2}}|V_{1}^{\pi}(X_{1}^{k})| \nonumber\\
        &=&\sum_{k=1}^{K}|V_{1}^{\pi}(X_{1}^{k})|\mathbb{I}_{\{\|X_1^k\|>\rho\}}+\sum_{k \in J_{2}}|V_{1}^{\pi}(X_{1}^{k})| \nonumber\\
&\leq&\sum_{k=1}^{K}\widetilde{C}_1\Big(\|X_1^k\|^{m+1}+1\Big)\mathbb{I}_{\{\|X_1^k\|>\rho\}}+\widetilde{C}_1(K-K_{0})\Big(\rho^{m+1}+1\Big),
\end{eqnarray}
where the inequality holds due to Proposition \ref{thm: value function with any policy growth rate}.

Let $Y:= \sum_{k=1}^{K}\widetilde{C}_1\Big(\|X_1^k\|^{m+1}+1\Big)\mathbb{I}_{\{\|X_1^k\|>\rho\}}$, then
\begin{eqnarray}\label{eq:bound on expectation of tail probability}
        \mathbb{E}[Y]
        &\leq&K\widetilde{C}_1\Big(\mathbb{P}(\|\xi\|>\rho)+ \mathbb{E}_{\xi\sim \Xi}\big[\|\xi\|^{m+1}\mathbb{I}_{\{\|\xi\|>\rho\}}\big]\Big)\nonumber\\
        &\leq& K\widetilde{C}_1\Big(\frac{\mathbb{E}_{\xi \sim \Xi}[\|\xi\|^{p}]}{\rho^p}+\big(\mathbb{E}_{\xi \sim \Xi}[\|\xi\|^{p}]\big)^{\frac{m+1}{p}}(\mathbb{P}(\|\xi\|>\rho))^{1-\frac{m+1}{p}}\Big)\nonumber\\
        &\leq& K\widetilde{C}_1\Big(\frac{\mathbb{E}_{\xi \sim \Xi}[\|\xi\|^{p}]}{\rho^p}+ \frac{\mathbb{E}_{\xi \sim \Xi}[\|\xi\|^{p}]}{\rho^{p-(m+1)}}\Big)=\delta K \kappa_{m+1}(\delta, \rho),
\end{eqnarray}
where 
the  second inequality holds by applying H\"older's inequality. By this inequality, 
 we have 
\begin{eqnarray}\label{eq:bound on tail probability}
        \mathbb{P}\Big(Y\geq K\kappa_{m+1}(\delta,\rho)\Big)\leq \mathbb{P}\left(Y\geq \frac{\mathbb{E}[Y]}{\delta}\right) \leq \delta,
\end{eqnarray}
where the last inequality holds by Markov inequality.
Putting \eqref{eq:bounds on value function outside I} and \eqref{eq:bound on tail probability} together,
we have \eqref{eq:concentration on value function at initial state} holds with probability at least $1-\delta$.
\end{proof}

\subsection{Proof of Lemma \ref{thm:induction}}\label{app:thm induction}

\begin{proof}

\begin{eqnarray}\label{eq:Thm F1 quote}
    && \Big(1+\frac{1}{H}\Big)f_{h+1}^{k-1}(X_h^k, A_h^k)+Q_h^{*}(X_h^k,A_h^k)-V_h^{\tilde{\pi}^{k}}(X_h^k)\nonumber\\
    &=& \Big(1+\frac{1}{H}\Big)\Big(
    \mathbb{E}_{Y \sim {T}_{h}(\cdot|X_h^k,A_h^k)}[\overline{V}_{h+1}^{k-1}(Y)]\nonumber\\
    &&\quad-\mathbb{E}_{Y \sim {T}_{h}(\cdot|X_h^k,A_h^k)}[V_{h+1}^{*}(Y)]\Big)
    +\bar{R}_{h}(X_h^k,A_h^k)\nonumber\\
    &&+\mathbb{E}_{Y \sim {T}_{h}(\cdot|X_h^k,A_h^k)}[V_{h+1}^{*}(Y)]
    -\mathbb{E}_{a\sim \pi_h^k(X_h^k)}[\bar{R}_{h}(X_h^k,a)]\nonumber\\
    &&-\mathbb{E}_{a\sim \pi_h^k(X_h^k), Y^{'} \sim {T}_{h}(\cdot|X_h^k,a)}
    [V_{h+1}^{\tilde{\pi}^{k}}(Y^{'})]\nonumber\\
    &\leq& \Big(1+\frac{1}{H}\Big)\Big(
    \mathbb{E}_{Y \sim {T}_{h}(\cdot|X_h^k,A_h^k)}[\overline{V}_{h+1}^{k-1}(Y)]\nonumber\\
    &&\quad-\mathbb{E}_{Y \sim {T}_{h}(\cdot|X_h^k,A_h^k)}[V_{h+1}^{*}(Y)]\Big)\nonumber\\
    &&+\Big(1+\frac{1}{H}\Big)\Big(
    \mathbb{E}_{Y \sim {T}_{h}(\cdot|X_h^k,A_h^k)}[V_{h+1}^{*}(Y)]\nonumber\\
    &&\quad-\mathbb{E}_{Y \sim {T}_{h}(\cdot|X_h^k,A_h^k)}[V_{h+1}^{\tilde{\pi}^{k}}(Y)]\Big)\nonumber\\
    &&+\bar{R}_{h}(X_h^k,A_h^k)-\mathbb{E}_{a\sim \pi_h^k(X_h^k)}[\bar{R}_{h}(X_h^k,a)]
    \nonumber\\&&+\mathbb{E}_{Y \sim {T}_{h}(\cdot|X_h^k,A_h^k)}[V_{h+1}^{\tilde{\pi}^{k}}(Y)]\nonumber\\
    &&\quad-\mathbb{E}_{a\sim \pi_h^k(X_h^k), Y^{'} \sim {T}_{h}(\cdot|X_h^k,a)}
    [V_{h+1}^{\tilde{\pi}^{k}}(Y^{'})]\nonumber\\
    &=& \Big(1+\frac{1}{H}\Big)(\Delta_{h+1}^{(k)}+\xi_{h+1}^{k})+\zeta_{h+1}^k\nonumber,
\end{eqnarray}

where the first inequality holds due to \eqref{eq:bellman}. 
\end{proof}

\subsection{Proof of \eqref{eq:xi_1}-\eqref{eq:zeta_2}}\label{app:proof-5-9}

We first provide a high probability bound for the state process that holds simultaneously across all episodes. For convenience, we denote 
\begin{eqnarray}\label{def:sup X_h^k}
      Z:=\sup_{h\in [H],k\in [K]}\|X_{h}^{k}\|.
\end{eqnarray}
\begin{lemma}\label{lemma:high probability bound across all episodes}
Assume Assumptions \ref{ass:lipschitz}-\ref{ass:initial} hold. We have:
\begin{eqnarray*}
    \mathbb{P}\Big(Z\leq \Big(\frac{KM_{p}}{\delta}\Big)^{\frac{1}{p}}\Big)\geq 1-\delta.
\end{eqnarray*}
\end{lemma}
\begin{proof}
\begin{eqnarray*}
    \mathbb{P}\Big(Z \leq \Big(\frac{KM_{p}}{K}\Big)^{\frac{1}{p}}\Big)&\geq& 1-\sum_{k=1}^{K}\mathbb{P}\Big(\sup_{h\in [H]}\|X_{h}^{k}\|\geq \Big(\frac{KM_{p}}{\delta}\Big)^{\frac{1}{p}}\Big)\\
    &\geq&1- K \frac{M_{p}}{\frac{KM_{p}}{\delta}}\\
    &=&1-\delta,
\end{eqnarray*}
where the first inequality holds by the union bound (namely, for a countable set of events $E_1,E_2,...$, we have $\mathbb{P}(\cap_{i}E_i)\geq 1-\sum_{i}\mathbb{P}(E_i^{\complement})$) and the second inequality holds due to Corollary \ref{Theorem: R-estimate}.
\end{proof}

Then we show the proof of \eqref{eq:xi_1} and \eqref{eq:xi_2}. We also claim that \eqref{eq:zeta_1} and \eqref{eq:zeta_2} may be proved in the similar way.

\subsubsection{Proof of \eqref{eq:xi_1}}
\begin{proof}
    With probability at least $1-3\delta$, it holds that, for $x\in \mathbb{R}^{d_{\mathcal{S}}}$,  $h\in [H-1]$ and $k>1$, we have:
\begin{eqnarray}\label{ineq:bar V upper bound}
        |\overline{V}_{h+1}^{k-1}(x)|& \leq & \max \{|V_{h+1,k-1}^{\rm local}(x,S^{\prime})|,|V_{h+1}^{*}(x)|\} \nonumber\\
        &\leq&\widetilde{V}_{h+1}^{0}(S^{\prime})+C_{h}\Big(1+\|x\|^m+\|\tilde{x}(S^{\prime})\|^m\Big)(\|x\|+\|\tilde{x}(S^{\prime})\|) \nonumber\\
        &\leq& \breve{c}_{1}\|x\|^{m+1}+\breve{c}_{2},\nonumber
\end{eqnarray}
where
\[
S^{\prime}=\argmin_{S\in \Gamma_{\mathcal{S}}(\mathcal{P}_{h+1}^{k-1})}
V_{h+1,k-1}^{\rm local}(x,S),
\]
and $\breve{c}_{1},\breve{c}_{2}$ depend only on $m,D,C_{\max},\widetilde{C}_{\max}$. The first inequality holds due to Theorem \ref{thm: True Upper Bounded By Estimates} and the third line of \eqref{V-estimate}. In addition, the second inequality holds due to \eqref{eq:initial value for estimation}, the fact that $\widetilde{V}_{h+1}^{k-1}(S^{\prime})\leq \widetilde{V}_{h+1}^{0}(S^{\prime})$ according to the second line of \eqref{V-estimate}, and the fact that $\|x-\tilde{x}(S^{\prime})\|\leq \|x\|+\|\tilde{x}(S^{\prime})\|$. Finally, the third inequality holds due to the fact that $\|\tilde{x}(S^{\prime})\|\leq \|x\|+D$.

In addition, note that \eqref{ineq:bar V upper bound} also holds for $k=1$, $x\in \mathbb{R}^{d_{\mathcal{S}}}$, and $h\in [H-1]$.

Let $\mathcal{F}_{h,k}=\sigma((X_{h}^{k^{\prime}},A_{h}^{k^{\prime}},r_{h}^{k^{\prime}}),h^{\prime}\leq h,k^{\prime}\leq k)$. We next show that we can bound $(\xi_{h+1}^k)^{2}$ and $\mathbb{E}[(\xi_{h+1}^k)^{2}|\mathcal{F}_{h,k}]$ by polynomials of $Z$.  

To proceed, we will often use the fact that for $n,q\in\mathbb{N}_{+}$:
\begin{eqnarray}\label{eq:fact 1}
    (\sum_{i=1}^{n} a_{i})^{q}\leq n^{q-1}\sum_{i=1}^{n}a_{i}^{q}
\end{eqnarray}
 and
\begin{eqnarray}\label{eq:fact 2}
    \mathbb{E}_{X \sim {T}_{h}(\cdot|X_h^k,A_h^k)}[\|X\|^{q}] &\leq&   (\mathbb{E}_{X \sim {T}_{h}(\cdot|X_h^k,A_h^k)}[\|X\|^{2q}])^{\frac{1}{2}}\nonumber\\
    &\leq & \widetilde{C}(q,d_{\mathcal{S}})(\|\mu_{h}(X_h^k,A_h^k)\|^{q}+\|\Sigma_{h}(X_h^k,A_h^k)\|^{\frac{q}{2}})\nonumber\\
    &\leq & 2\widetilde{C}(q,d_{\mathcal{S}})\eta(X_h^k)^{q},
\end{eqnarray}
where the second inequality holds due to Lemma \ref{lemma:2q moment of multi normal} and the third inequality holds due to \eqref{eq:eta function}.

Then with probability at least $1-3\delta$, the following inequality holds for $h\in [H]$ and $k\in [K]$:
\begin{eqnarray}\label{bound 1}
    (\xi_{h+1}^k)^{2}
    &\leq& \Big(\mathbb{E}_{X \sim {T}_{h}(\cdot|X_h^k,A_h^k)}[|\overline{V}_{h+1}^{k-1}(X)|]+\mathbb{E}_{X \sim {T}_{h}(\cdot|X_h^k,A_h^k)}[|V_{h+1}^{\tilde{\pi}^{k}}(X)|]\nonumber
    \\&&+|\overline{V}_{h+1}^{k-1}(X_{h+1}^{k})|+|V_{h+1}^{\tilde{\pi}^{k}}(X_{h+1}^{k})|\Big)^{2}\nonumber\\
    &\leq & 4\Big((\mathbb{E}_{X \sim {T}_{h}(\cdot|X_h^k,A_h^k)}[|\overline{V}_{h+1}^{k-1}(X)|])^{2}+(\mathbb{E}_{X \sim {T}_{h}(\cdot|X_h^k,A_h^k)}[|V_{h+1}^{\tilde{\pi}^{k}}(X)|])^{2} \nonumber\\
     &&+(\overline{V}_{h+1}^{k-1}(X_{h+1}^{k}))^2+(V_{h+1}^{\tilde{\pi}^{k}}(X_{h+1}^{k}))^{2} \Big) \nonumber\\
    &\leq & \breve{c}_{3} Z^{2m+2}+\breve{c}_{4},
\end{eqnarray}
where $\breve{c}_{3},\breve{c}_{4}$ depends only on $\widetilde{C}_{\max},C_{\max},m,D,d_{\mathcal{S}}$.  The last inequality holds due to \eqref{ineq:bar V upper bound}, \eqref{eq:value function with any policy growth rate}, \eqref{eq:fact 1}, \eqref{eq:fact 2} and the fact that $\|X_h^k\|\leq Z$ for $(h,k)\in [H]\times [K]$.

Similarly, with probability at least $1-3\delta$, the following inequality holds for $h\in [H]$ and $k\in [K]$:
\begin{eqnarray}\label{bound 2}
    &&\mathbb{E}[(\xi_{h+1}^k)^{2}|\mathcal{F}_{h,k}]\nonumber\\
    &\leq& \mathbb{E}_{Y \sim {T}_{h}(\cdot|X_h^k,A_h^k)}\Bigg[\Big(\mathbb{E}_{X \sim {T}_{h}(\cdot|X_h^k,A_h^k)}[|\overline{V}_{h+1}^{k-1}(X)|]+\mathbb{E}_{X \sim {T}_{h}(\cdot|X_h^k,A_h^k)}[|V_{h+1}^{\tilde{\pi}^{k}}(X)|]
    \nonumber\\&& \qquad \qquad \qquad \qquad +|\overline{V}_{h+1}^{k-1}(Y)|+|V_{h+1}^{\tilde{\pi}^{k}}(Y)|\Big)^{2}\Bigg]\nonumber\\
    &\leq & \breve{c}_{5} Z^{2m+2}+\breve{c}_{6},
\end{eqnarray}
where $\breve{c}_{5},\breve{c}_{6}$ depends only on $\widetilde{C}_{\max},C_{\max},m,D,d_{\mathcal{S}}$. 

Define $M_{h+1,k}:=\sum_{h^{\prime}\leq h,k^{\prime}\leq k}\xi_{h+1}^k$. It is clear that $M_{h+1,k}$ is a square integrable martingale. Then by Theorem 2.1 in \cite{bercu2008exponential}, for any $a,b>0$ we have: 
\begin{eqnarray}
    \mathbb{P}\Big(|M_{H+1,K} | \geq a, \langle M\rangle_{H+1,K}+[M]_{H+1,K} \leq b \Big)\leq 2\exp{(-\frac{a^2}{2b})},
\end{eqnarray}
where $[M]_{H+1,K}=\sum_{h\in [H],k\in[K]} (\xi_{h+1}^k)^2, \langle M\rangle_{H+1,K}=\sum_{h\in [H],k\in[K]} \mathbb{E}[(\xi_{h+1}^k)^2|\mathcal{F}_{h,k}]$. 

Therefore, we have for any $a,b,c>0$: 
\begin{eqnarray}
    &&\mathbb{P}(|M_{H+1,K}|\geq a)\nonumber\\
    &\leq &\mathbb{P}(|M_{H+1,K}|\geq a, \langle M\rangle_{H+1,K}+[M]_{H+1,K} \leq b )+  \mathbb{P}(\langle M\rangle_{H+1,K}+[M]_{H+1,K} \geq b )\nonumber \\
    &\leq& 2\exp{(-\frac{a^2}{2b})}+  \mathbb{P}(\langle M\rangle_{H+1,K}+[M]_{H+1,K} \geq b ,Z \leq c)+\mathbb{P}(Z \geq c).
\end{eqnarray}
Let $a=\sqrt{2b\log(\frac{2}{\delta})}$, $ b=2HK(\breve{c}_{3}+\breve{c}_{4}+\breve{c}_{5}+\breve{c}_{6})(c^{2m+2}+1)$,  and $c=\Big(\frac{M_{p}K}{\delta}\Big)^{\frac{1}{p}}$, we get that:
\begin{eqnarray*}
   && \mathbb{P}\Big(\langle M\rangle_{H+1,K}+[M]_{H+1,K} \geq b ,Z \leq c\Big)\leq3\delta, \textit{ and } \mathbb{P}\Big(Z \geq c\Big)\leq \delta,
\end{eqnarray*}
where the first inequality holds since with probability at least $1-3\delta$, for any $h\in [H]$ and $k \in [K]$, it holds that \eqref{bound 1} and \eqref{bound 2}. 

Hence, we conclude that
\begin{eqnarray}
    \mathbb{P}\Bigg(|M_{H+1,K}|\leq 2e^{2}\sqrt{\widetilde{L}_{1}HK\Big(\Big(\frac{M_{p}K}{\delta}\Big)^{\frac{2m+2}{p}} +1\Big)\log\Big(\frac{2}{\delta}\Big)}\,\,\Bigg)\geq 1-5\delta,
\end{eqnarray}
where $\widetilde{L}_{1}=\breve{c}_{3}+\breve{c}_{4}+\breve{c}_{5}+\breve{c}_{6}$. 
\end{proof}

\subsubsection{Proof of \eqref{eq:xi_2}}
\begin{proof}
By similar methods to show \eqref{bound 1} in the proof of \eqref{eq:xi_1}, we can show that
\begin{eqnarray*}\label{bound 3}
    |\xi_{h+1}^k|\leq \breve{c}_{1} Z^{m+1}+\breve{c}_{2},
\end{eqnarray*}
where $\breve{c}_{1},\breve{c}_{2}$ depends only on $\widetilde{C}_{\max},C_{\max},m,D,d_{\mathcal{S}}$. 

Therefore, let $\widetilde{L}_{2}=\breve{c}_{1}+\breve{c}_{2}$,  we have

\begin{eqnarray}
    &&\mathbb{P}\Bigg( \Big|\sum_{h=1}^{H}\sum_{k\in [K]\backslash J_{\rho}^{K}}^{K}
    \xi_{h+1}^{k}\Big|\leq\nonumber\\
    &&\quad e^{2}\widetilde{L}_{2}H
    \Big(\frac{KM_{p}}{\rho^{p}}+\sqrt{2K\log(\frac{1}{\delta}})\Big)
    \Big(\Big(\frac{M_{p}K}{\delta}\Big)^{\frac{m+1}{p}} +1\Big)\Bigg)\nonumber\\
    &\geq & \mathbb{P}\Bigg( \sum_{h=1}^{H}\sum_{k\in [K]\backslash J_{\rho}^{K}}^{K}
    \Big|\xi_{h+1}^{k}\Big|\leq\nonumber\\
    &&\quad e^{2}\widetilde{L}_{2}H
    \Big(\frac{KM_{p}}{\rho^{p}}+\sqrt{2K\log(\frac{1}{\delta}})\Big)
    \Big(\Big(\frac{M_{p}K}{\delta}\Big)^{\frac{m+1}{p}} +1\Big)\Bigg)\nonumber\\
    &\geq & \mathbb{P}\Bigg(H(K-K_{0})(\breve{c}_{1} Z^{m+1}+\breve{c}_{2}) \leq\nonumber\\
    &&\quad e^{2}\widetilde{L}_{2}H
    \Big(\frac{KM_{p}}{\rho^{p}}+\sqrt{2K\log(\frac{1}{\delta}})\Big)
    \Big(\Big(\frac{M_{p}K}{\delta}\Big)^{\frac{m+1}{p}} +1\Big)\Bigg)\nonumber\\
    &\geq & 1-2\delta,
\end{eqnarray}

where the last inequality holds due to Proposition \ref{thm:concentration on size of J0} and Lemma \ref{lemma:high probability bound across all episodes}. 
\end{proof}

\subsection{choice of $c$}\label{app:choice of c}
\begin{remark}[Choice of $c$ in Definition \ref{def:zooming dimension}]\label{remark:zooming dimension}
    In Definition \ref{def:zooming dimension}, if we take $c\geq C_{\mathcal{S},\mathcal{A}}:=\frac{2^{d_{\mathcal{S}}}\Gamma(\frac{d_{\mathcal{S}}+d_{\mathcal{A}}}{2}+1)\bar{a}^{d_{\mathcal{A}}}}{\Gamma(
\frac{d_{\mathcal{S}}}{2}+1)\Gamma(
\frac{d_{\mathcal{A}}}{2}+1)}$, then it holds that $ z_{h,c}\leq d_{\mathcal{S}}+d_{\mathcal{A}}$.

To see this, first note that
\begin{eqnarray}\label{eq:choice of c}
         N_{r}(Z_{h}^{r,\rho})&\leq&N_{r}(\bar{Z})\leq \frac{\Gamma(\frac{d_{\mathcal{S}}+d_{\mathcal{A}}}{2}+1)}{\Gamma(
\frac{d_{\mathcal{S}}}{2}+1)\Gamma(
\frac{d_{\mathcal{A}}}{2}+1)}\left(\frac{\rho+D}{r}\right)^{d_{\mathcal{S}}}\left(\frac{\bar{a}}{r}\right)^{d_{\mathcal{A}}} \nonumber\\
&\leq&C_{\mathcal{S},\mathcal{A}} \frac{\rho^{d_{\mathcal{S}}}}{r^{d_{\mathcal{S}}+d_{\mathcal{A}}}}\leq c \frac{\rho^{d_{\mathcal{S}}}}{r^{d_{\mathcal{S}}+d_{\mathcal{A}}}}.
\end{eqnarray}
Rearrange \eqref{eq:choice of c}, and we get:
\begin{eqnarray*}
    \frac{N_{r}(Z_{h}^{r,\rho})}{\rho^{d_{\mathcal{S}}}}\leq cr^{-(d_{\mathcal{S}}+d_{\mathcal{A}})}. 
\end{eqnarray*}
Hence, by Definition \ref{def:zooming dimension}, we have $z_{h,c}\leq d_{\mathcal{S}}+d_{\mathcal{A}}$.
\end{remark}

In light of Remark \ref{remark:zooming dimension}, we take $c \geq C_{\mathcal{S},\mathcal{A}}$ throughout the rest of the paper. This ensures that the zooming dimension does not exceed the ambient dimension $d_{\mathcal{S}}+d_{\mathcal{A}}$.

\subsection{Technical results modified from \cite{sinclair2023adaptive}}

\subsubsection{Proof of Lemma \ref{lemma:theorem F.3 conclusion}}\label{app:proof-lemma-5.13}

\begin{proof}
We firstly split $\sum_{h}\sum_{k\in J_{\rho}^{K}}{\rm CLIP}\Big(G_h^k(B_h^k)\,\Big|\,\frac{{\rm Gap}_{h}(B_h^k)}{H+1}\Big)$ into two terms.
\begin{eqnarray}\label{eq:decoupled_terms}
    &&\sum_{h}\sum_{k\in J_{\rho}^{K}}{\rm CLIP}\Big(G_h^k(B_h^k)\,\Big|\,\frac{{\rm Gap}_{h}(B_h^k)}{H+1}\Big)\nonumber\\
    &=&\underbrace{\sum_{h}\sum_{k\in J_{\rho}^{K}}\sum_{B_h^k:n_h^{k-1}(B_h^k)>0} {\rm CLIP}\Big(G_h^k(B_h^k)\,\Big|\,\frac{{\rm Gap}_{h}(B_h^k)}{H+1}\Big)}_{(I)}\nonumber\\
    &&+\underbrace{\sum_{h}\sum_{k\in J_{\rho}^{K}}\sum_{B_h^k:n_h^{k-1}(B_h^k)=0} {\rm CLIP}\Big(G_h^k(B_h^k)\,\Big|\,\frac{{\rm Gap}_{h}(B_h^k)}{H+1}\Big)}_{(II)}.
\end{eqnarray}

Then we handle these two terms separately.

\vspace{10pt}

\noindent \underline{Bound for Term (I)}:

For fixed $(h,k)$, if $n_h^{k-1}(B_h^k)>0$, then:
\begin{eqnarray}\label{eq:d48}
    G_{h}^{k}(B_h^k)&=&2\frac{\widehat{C}_{\max}}{\overline{C}_{\max}}\Big(\mbox{\rm R-UCB}_{h}^{k-1}(B_h^k)+ \mbox{\rm T-UCB}_{h}^{k-1}(B_h^k)+{\rm BIAS}(B_h^k)\Big)\nonumber\\
        &&+C_h(1+2(\|\tilde{x}(^{o}B_h^k)\|+D)^m){\rm diam}(B_h^k)\nonumber\\
        &\leq & 2\frac{\widehat{C}_{\max}}{\overline{C}_{\max}}\Big({\rm CONF}_{h}^{k-1}(B_h^k)+g_{2}(\delta,\|\tilde{x}(^{o}B_h^k)\|){\rm CONF}_{h}^{k-1}(B_h^k)\Big)\nonumber\\
        &&+C_h(1+2(\|\tilde{x}(^{o}B_h^k)\|+D)^m){\rm CONF}_{h}^{k-1}(B_h^k)\nonumber\\
        &=& g_{3}(\delta, \|\tilde{x}(^{o}B_h^k)\|)\,\,{\rm CONF}_{h}^{k-1}(B_h^k), 
\end{eqnarray}
where $g_2$ is defined in \eqref{eq:g2 definition} and $g_{3}$ is defined in \eqref{eq:g3 definition}. The first equality holds by the definition of $G_h^k(B_h^k)$ in \eqref{eq:G_h^k definition}. The inequality holds because 
\begin{itemize}
    \item[(a)] $\mbox{\rm R-UCB}_{h}^{k-1}(B_h^k)+ \mbox{\rm T-UCB}_{h}^{k-1}(B_h^k) \leq {\rm CONF}_{h}^{k-1}(B_h^k)$ by \eqref{eq:CONF}, 
    \item[(b)] ${\rm BIAS}(B_h^k)=g_{2}(\delta,\|\tilde{x}(^{o}B_h^k)\|)\,\,{\rm diam}(B_h^k)$ by \eqref{eq:g2 definition},  and 
    \item[(c)] ${\rm diam}(B_h^k)\leq {\rm CONF}_{h}^{k-1}(B_h^k)$ by the Splitting Rule in line \ref{splitting rule} of Algorithm \ref{alg:ML 4}.
\end{itemize}
The last equality holds by \eqref{eq:g3 definition}.

By definition of the ${\rm CLIP}(\cdot|.)$ function in \eqref{eq:clip}, \eqref{eq:d48} implies that 
\begin{eqnarray}\label{eq:clip direct bound}
    {\rm CLIP}\Big(G_h^k(B_h^k)\,\Big|\,\frac{{\rm Gap}_{h}(B_h^k)}{H+1}\Big)&\leq& {\rm CLIP}\Big(g_{3}(\delta, \|\tilde{x}(^{o}B_h^k)\|)\,\,{\rm CONF}_{h}^{k-1}(B_h^k)\,\Big|\,\frac{{\rm Gap}_{h}(B_h^k)}{H+1}\Big)\\
    &=& g_{3}(\delta, \|\tilde{x}(^{o}B_h^k)\|)\,\,{\rm CONF}_{h}^{k-1}(B_h^k)\nonumber\\
    &&\quad \mathbb{I}_{\left\{g_{3}(\delta, \|\tilde{x}(^{o}B_h^k)\|){\rm CONF}_h^{k-1}(B_h^k)\geq \frac{{\rm Gap}_{h}(B_h^k)}{H+1}\right\}}.\nonumber
\end{eqnarray}

Next, we find an upper bound for $\mathbb{I}_{\left\{g_{3}(\delta, \|\tilde{x}(^{o}B_h^k)\|){\rm CONF}_h^{k-1}(B_h^k)\geq \frac{{\rm Gap}_{h}(B_h^k)}{H+1}\right\}}$.

Note that for $(x_1,a_1), (x_2,a_2) \in \mathbb{R}^{d_{\mathcal{S}}}\times \mathcal{A}$, by \eqref{eq:value function local Lipschitz} and \eqref{eq:Q local lipschitz}, we have: 

\begin{eqnarray}\label{eq:GAP inside lip}
&&|{\rm \widetilde{G}ap}_h(x_1,a_1)-{\rm \widetilde{G}ap}_h(x_2,a_2)|\nonumber\\
&\leq& 3{\overline{C}_{\max}(1+\|x_1\|^m+\|x_2\|^m)}
(\|x_1-x_2\|+\|a_1-a_2\|).
\end{eqnarray}

Then by definition in \eqref{eq:gap definition} and \eqref{eq:GAP inside lip}, we have:
\begin{eqnarray}\label{ineq:gap bound 1}
    {\rm \widetilde{G}ap}_{h}({\rm center}(B_h^k))
       \leq {\rm Gap}_{h}(B_h^k)+3{\overline{C}_{\max}(1+2(\|\tilde{x}(^{o}B_h^k)\|+D)^m))}{\rm diam}(B_h^k).
\end{eqnarray}

In addition, we have 

\begin{eqnarray}\label{ineq:l bound}
    &&(H+1)g_{3}(\delta, \|\tilde{x}(^{o}B_h^k)\|)\,\,{\rm CONF}_h^{k-1}(B_h^k)\nonumber\\
    &&\quad+3{\overline{C}_{\max}(1+2(\|\tilde{x}(^{o}B_h^k)\|+D)^m)}
    {\rm diam}(B_h^k)\nonumber\\
     &\leq & 2\Big((H+1)g_{3}(\delta, \|\tilde{x}(^{o}B_h^k)\|)\nonumber\\
     &&\quad+3\overline{C}_{\max}(1+2(\|\tilde{x}(^{o}B_h^k)\|+D)^m)\Big)
     {\rm diam}(B_h^k)\nonumber\\
      &\leq&  \bar{g}(\delta,\tilde{x}(B_h^k))(H+1){\rm diam}(B_h^k),
\end{eqnarray}

where the first inequality holds due to \eqref{eq:CONF splitting bound parent} and the second inequality holds due to \eqref{eq: l(x) def}.

Therefore,
\begin{eqnarray}\label{eq:indicator inequality}
       && \mathbb{I}_{\left\{g_{3}(\delta, \|\tilde{x}(^{o}B_h^k)\|){\rm CONF}_h^{k-1}(B_h^k)\geq \frac{{\rm Gap}_{h}(B_h^k)}{H+1}\right\}}\nonumber\\&
       = & \mathbb{I}_{\left\{(H+1)g_{3}(\delta, \|\tilde{x}(^{o}B_h^k)\|){\rm CONF}_h^{k-1}(B_h^k)\geq {\rm Gap}_{h}(B_h^k)\right\}}\nonumber\\
        &\leq & \mathbb{I}_{\left\{ \bar{g}(\delta,\tilde{x}(B_h^k))(H+1){\rm diam}(B_h^k)\geq {\rm Gap}_{h}(B_h^k)+3{\overline{C}_{\max}(1+2(\|\tilde{x}(^{o}B_h^k)\|+D)^m)}{\rm diam}(B_h^k)\right\}}\nonumber\\
        &\leq & \mathbb{I}_{\left\{ \bar{g}(\delta,\tilde{x}(B_h^k))(H+1){\rm diam}(B_h^k)\geq {\rm \widetilde{G}ap}_{h}({\rm center}(B_h^k))\right\}}\nonumber\\
        &\leq & \mathbb{I}_{\left\{{\rm center}(B_h^k)\in Z_h^{{\rm diam}(B_h^k),\rho}\right\}},
\end{eqnarray}
where the first inequality holds by \eqref{ineq:l bound}, the second inequality holds by \eqref{ineq:gap bound 1} and the last inequality holds by \eqref{eq:def near optimal set}.

Then,

\begin{eqnarray}\label{eq: CLIP I}
     && \sum_{h}\sum_{k\in J_{\rho}^{K}}
     \sum_{B_h^k:n_h^{k-1}(B_h^k)>0}\nonumber\\
     &&\quad\times {\rm CLIP}\Big(G_h^k(B_h^k)\,\Big|\,
     \frac{{\rm Gap}_{h}(B_h^k)}{H+1}\Big)\nonumber\\ 
     &\leq&\sum_{h}\sum_{k\in J_{\rho}^{K}}
     \sum_{B_h^k:n_h^{k-1}(B_h^k)>0}\nonumber\\
     &&\quad\times {\rm CLIP}\Big(g_{3}(\delta, \|\tilde{x}(^{o}B_h^k)\|)
     {\rm CONF}_h^{k-1}(B_h^k)\,\Big|\,\frac{{\rm Gap}_{h}(B_h^k)}{H+1}\Big)\nonumber\\
     &=&  \sum_{h}\sum_{k\in J_{\rho}^{K}}
     \sum_{B_h^k:n_h^{k-1}(B_h^k)>0}
     g_{3}(\delta, \|\tilde{x}(^{o}B_h^k)\|){\rm CONF}_h^{k-1}(B_h^k)\nonumber\\
     &&\quad\times\mathbb{I}_{\{g_{3}(\delta, \|\tilde{x}(^{o}B_h^k)\|)
     {\rm CONF}_h^{k-1}(B_h^k)\geq \frac{{\rm Gap}_{h}(B_h^k)}{H+1}\}}\nonumber\\
     &\leq& \sum_{h}\sum_{k\in J_{\rho}^{K}}
     \sum_{B_h^k:n_h^{k-1}(B_h^k)>0}
     g_{3}(\delta, \|\tilde{x}(^{o}B_h^k)\|){\rm CONF}_h^{k-1}(B_h^k)\nonumber\\
     &&\quad\times\mathbb{I}_{\{{\rm center}(B_h^k)
     \in Z_h^{{\rm diam}(B_h^k),\rho,l}\}}\nonumber\\
     &=&  \sum_{h} \sum_{r\in \mathcal{R}}\sum_{B:{\rm diam}(B)=r}
     \sum_{k:B_h^k=B, n_h^{k-1}(B_h^k)>0}\nonumber\\
     &&\quad\times g_{3}(\delta,\|\tilde{x}(^{o}B)\|){\rm CONF}_h^{k-1}(B)
     \mathbb{I}_{\{{\rm center}(B)\in Z_h^{r,\rho}\}}\nonumber\\
     &=&  \sum_{h} \sum_{r\in \mathcal{R},r< r_{0}}
     \sum_{B:{\rm diam}(B)=r}\sum_{k:B_h^k=B, n_h^{k-1}(B_h^k)>0}\nonumber\\
     &&\quad\times g_{3}(\delta,\|\tilde{x}(^{o}B)\|){\rm CONF}_h^{k-1}(B)
     \mathbb{I}_{\{{\rm center}(B)\in Z_h^{r,\rho}\}} \nonumber\\
     &&+ \sum_{h} \sum_{r\in \mathcal{R}, r \geq r_{0}}
     \sum_{B:{\rm diam}(B)=r}\sum_{k:B_h^k=B, n_h^{k-1}(B_h^k)>0}\nonumber\\
     &&\quad\times g_{3}(\delta,\|\tilde{x}(^{o}B)\|){\rm CONF}_h^{k-1}(B)
     \mathbb{I}_{\{{\rm center}(B)\in Z_h^{r,\rho}\}} \nonumber\\
     &\leq& 2g_{3}(\delta,\rho+D)Kr_{0} + g_{3}(\delta,\rho+D)g_{1}(\delta, \rho+D) \times \nonumber\\
     && \sum_{h}\sum_{r\in \mathcal{R}, r \geq r_{0}}\sum_{B:{\rm diam}(B)=r}\mathbb{I}_{\{{\rm center}(B)\in Z_h^{r,\rho}\}}\sum_{k:B_h^k=B}\frac{1}{\sqrt{n_h^{k-1}(B)}},
\end{eqnarray}

where the first inequality holds by \eqref{eq:clip direct bound}, the second inequality holds by \eqref{eq:indicator inequality}, and the last inequality holds due to \eqref{eq:CONF splitting bound parent}.

To bound the second term above,

\begin{eqnarray}\label{eq:CLIP II}
    &&g_{3}(\delta,\rho+D)g_{1}(\delta, \rho+D)
    \sum_{h}\sum_{r\in \mathcal{R}, r \geq r_{0}}
    \sum_{B:{\rm diam}(B)=r}\nonumber\\
    &&\quad\times\mathbb{I}_{\{{\rm center}(B)\in Z_h^{r,\rho}\}}
    \sum_{k:B_h^k=B}\frac{1}{\sqrt{n_h^{k-1}(B)}}\nonumber\\
    &\leq& g_{3}(\delta,\rho+D)g_{1}(\delta, \rho+D) \times \nonumber\\
    &&\qquad \sum_{h}\sum_{r\in \mathcal{R}, r \geq r_{0}}
    \sum_{B:{\rm diam}(B)=r}\mathbb{I}_{\{{\rm center}(B)\in Z_h^{r,\rho}\}}\nonumber\\
    &&\qquad\times \int_{x=0}^{n_{\max}(B)-n_{\min}(B)}
    \frac{1}{\sqrt{x+n_{\min}(B)}}\,dx\nonumber\\
    &\leq& 2g_{3}(\delta,\rho+D)g_{1}(\delta, \rho+D)
    \sum_{h}\sum_{r\in \mathcal{R}, r \geq r_{0}}
    \sum_{B:{\rm diam}(B)=r}\nonumber\\
    &&\quad\times\mathbb{I}_{\{{\rm center}(B)\in Z_h^{r,\rho}\}}\sqrt{n_{\max}(B)}\nonumber\\
    &\leq& 2g_{3}(\delta,\rho+D)g_{1}(\delta, \rho+D)
    \sum_{h}\sum_{r\in \mathcal{R}, r \geq r_{0}}
    \sum_{B:{\rm diam}(B)=r}\nonumber\\
    &&\quad\times\mathbb{I}_{\{{\rm center}(B)\in Z_h^{r,\rho}\}}
    \frac{g_{1}(\delta, \rho+D)}{r}\nonumber\\
    &\leq& 2g_{3}(\delta,\rho+D)g_{1}(\delta, \rho+D)^{2} \sum_{h}\sum_{r\in \mathcal{R}, r \geq r_{0}}N_{r}(Z_h^{r,\rho})\frac{1}{r},
\end{eqnarray}

where $n_{\max}(B)=(\frac{g_{1}(\delta,\|\tilde{x}(^{o}B)\|}{ {\rm diam}(B)})^{2}, n_{\min}(B)=(\frac{g_{1}(\delta,\|\tilde{x}(^{o}B)\|}{ 2{\rm diam}(B)})^{2}$. The first inequality holds due to the fact that $n_{\min}(B) \leq n_h^{k-1}(B)< n_{\max}(B)$ by \eqref{eq:n min} and \eqref{eq: n max}. The second inequality holds by the fact that $\int_{a}^{b} \frac{1}{\sqrt{y}}\,dy \leq 2\sqrt{b}$ for $b>a>0$. The fourth inequality holds due to \eqref{eq: n max} and the last inequality holds due to the fact that $\sum_{B:{\rm diam}(B)=r}\mathbb{I}_{\{{\rm center}(B)\in Z_h^{r,\rho}\}}\leq N_{r}(Z_h^{r,\rho})$. This fact holds since the distance between the centers of two blocks $B_1$ and $B_2$ with same diameter $r$ is at least $r$.  

\vspace{10pt}

\noindent\underline{Bound for Term (II)}: Next we bound Term (II) in \eqref{eq:decoupled_terms}.

For fixed $(h,k)$, if $n_h^{k-1}(B_h^k)=0$, then ${\rm diam}(B_h^k)=D$.

We now find an upper bound for $\mathbb{I}_{\left\{G_h^k(B_h^k)\geq \frac{{\rm Gap}_{h}(B_h^k)}{H+1}\right\}}$.  Note that by \eqref{eq:G_h^k definition} and \eqref{eq: l(x) def} we have: 
\begin{eqnarray}\label{eq:l bound 2}
    &&(H+1)G_h^k(B_h^k)+3{\overline{C}_{\max}(1+2(\|\tilde{x}(^{o}B_h^k)\|+D)^m)}{\rm diam}(B_h^k)\nonumber\\
      &\leq& \bar{g}(\delta,\tilde{x}(B_h^k))(H+1){\rm diam}(B_h^k).   
\end{eqnarray}
Therefore, 

\begin{eqnarray}\label{eq:indicator bound 2}
         &&\mathbb{I}_{\left\{G_h^k(B_h^k)\geq \frac{{\rm Gap}_{h}(B_h^k)}{H+1}\right\}}\nonumber\\
         &=&\mathbb{I}_{\left\{\bar{g}(\delta,\tilde{x}(B_h^k))(H+1){\rm diam}(B_h^k)
         \geq {\rm Gap}_{h}(B_h^k)+3{\overline{C}_{\max}(1+2(\|\tilde{x}(^{o}B_h^k)\|+D)^m)}
         {\rm diam}(B_h^k)\right\}}\nonumber\\
          &\leq&\mathbb{I}_{\left\{ \bar{g}(\delta,\tilde{x}(B_h^k))(H+1){\rm diam}(B_h^k)\geq {\rm \widetilde{G}ap}_{h}({\rm center}(B_h^k))\right\}}\nonumber\\
        &\leq&\mathbb{I}_{\left\{{\rm center}(B_h^k)\in Z_h^{{\rm diam}(B_h^k),\rho}\right\}},
\end{eqnarray}

where the first inequality holds by \eqref{eq:l bound 2}, the second inequality holds by \eqref{ineq:gap bound 1} and the third inequality holds by \eqref{def:near optimal set}.

Then 
\begin{eqnarray}\label{eq:CLIP III}
     &&\sum_{h}\sum_{k\in J_{\rho}^{K}}\sum_{B_h^k:n_h^{k-1}(B_h^k)=0}{\rm CLIP}\Big(G_h^k(B_h^k)\,\Big|\,\frac{{\rm Gap}_{h}(B_h^k)}{H+1}\Big)\nonumber\\
     &\leq&  \sum_{h}\sum_{k\in J_{\rho}^{K}}\sum_{B_h^k:n_h^{k-1}(B_h^k)=0}  \bar{g}(\delta, \tilde{x}(^{o}B_h^k)){\rm diam}(B_h^k)\mathbb{I}_{\left\{{\rm center}(B_h^k)\in Z_h^{{\rm diam}(B_h^k),\rho}\right\}}\nonumber\\
     &=&  \sum_{h} \sum_{r=D}\sum_{B:{\rm diam}(B)=r}\sum_{k:B_h^k=B, n_h^{k-1}(B_h^k)=0} \bar{g}(\delta,\tilde{x}(^{o}B)){\rm diam}(B) \mathbb{I}_{\{{\rm center}(B)\in Z_h^{r,\rho}\}}\nonumber\\
     &\leq& \bar{g}(\delta,\rho+D)D \sum_{h}\sum_{r=D}\sum_{B:{\rm diam}(B)=r}\mathbb{I}_{\{{\rm center}(B)\in Z_h^{r,\rho}\}} |\{k:B_h^k=B, n_h^{k-1}(B_h^k)=0\}|,\nonumber\\
     &\leq&(d_{\mathcal{S}}+d_{\mathcal{A}})^{\frac{d_{\mathcal{S}}+d_{\mathcal{A}}}{2}}\frac{(\rho+D)^{d_{\mathcal{S}}}(2\bar{a})^{d_{\mathcal{A}}}}{D^{d_{\mathcal{S}}+d_{\mathcal{A}}-1}}\bar{g}(\delta,\rho+D)\sum_{h}\sum_{r=D}\sum_{B:{\rm diam}(B)=r}\mathbb{I}_{\{{\rm center}(B)\in Z_h^{r,\rho}\}} ,\nonumber\\
     &\leq&(d_{\mathcal{S}}+d_{\mathcal{A}})^{\frac{d_{\mathcal{S}}+d_{\mathcal{A}}}{2}}\frac{(\rho+D)^{d_{\mathcal{S}}}(2\bar{a})^{d_{\mathcal{A}}}}{D^{d_{\mathcal{S}}+d_{\mathcal{A}}-2}}\bar{g}(\delta,\rho+D)\sum_{h}\sum_{r=D}N_{r}(Z_h^{r,\rho})\frac{1}{r}\nonumber.
\end{eqnarray}
The first inequality holds due to \eqref{eq:l bound 2} and \eqref{eq:indicator bound 2}. The third inequality holds since $|\{k:B_h^k=B, n_h^{k-1}(B_h^k)=0\}|\leq |\mathcal{P}_h^0|\leq (d_{\mathcal{S}}+d_{\mathcal{A}})^{\frac{d_{\mathcal{S}}+d_{\mathcal{A}}}{2}}\frac{(\rho+D)^{d_{\mathcal{S}}}(2\bar{a})^{d_{\mathcal{A}}}}{D^{d_{\mathcal{S}}+d_{\mathcal{A}}}}$, and the last inequality holds due to the fact that $\sum_{B:{\rm diam}(B)=r}\mathbb{I}_{\{{\rm center}(B)\in Z_h^{r,\rho}\}}\leq N_{r}(Z_h^{r,\rho})$. This fact holds since the distance between the centers of two blocks $B_1$ and $B_2$ with the same diameter $r$ is at least $r$.  

Finally, combining \eqref{eq: CLIP I}, \eqref{eq:CLIP II},  \eqref{eq:CLIP III}, and noting the definition of $g_{4}(.,.) $ in \eqref{eq:g4 definition}, we verify \eqref{eq:CLIP BOUND}.
\end{proof}

\subsection{Proof of Theorem \ref{thm:final coupled for worst case analysis}}\label{app:final coupled for worst case analysis}
\begin{proof}
{Take $\rho=M_{p}^{\frac{1}{p}}K^{\beta}$ and $r_{0}=K^{\gamma}$ in Theorem \ref{thm:first stage regret decomposition}}, we have with probability at least $1-6\delta$:

\begin{eqnarray}
        &&{\rm Regret}(K)\nonumber\\
        &\leq& e^{2}\sum_{h=1}^{H}\sum_{k \in J_{1}}
        {\rm CLIP}\Bigg(G_{h}^{k}(B_h^k)\Bigg|\frac{{\rm Gap}_{h}(B_h^k)}{H+1}\Bigg)\nonumber\\
        &&+2e^{2}\sqrt{\widetilde{L}_{1}HK
        \Big(\Big(\frac{M_{p}K}{\delta}\Big)^{\frac{2m+2}{p}} +1\Big)
        \log\Big(\frac{2}{\delta}\Big)}\nonumber\\
        &&+2K\kappa_{m+1}(\delta,\rho)\nonumber\\
        &&+4\widetilde{C}_1\Big(\widetilde{L}_{3}+\rho^{m+1}
        +e^{2}\widetilde{L}_{2}H\Big(\frac{M_{p}K}{\delta}\Big)^{\frac{m+1}{p}} \Big)\nonumber\\
        &&\quad\times\left(\frac{M_{p}}{\rho^{p}}K
        +\sqrt{2K\log\Big(\frac{1}{\delta}\Big)}\right)\nonumber\\
        &\leq& e^{2} \sum_{h=1}^{H}\Big(2g_{3}(\delta,\rho+D)Kr_{0}\nonumber\\
        &&\quad+g_{4}(\delta,\rho+D)
        \sum_{r \geq r_{0}, r\in \mathcal{R}}N_{r}(Z_{h}^{r,\rho})\frac{1}{r}\Big)\nonumber\\
        &&+2e^{2}\sqrt{\widetilde{L}_{1}HK
        \Big(\Big(\frac{M_{p}K}{\delta}\Big)^{\frac{2m+2}{p}} +1\Big)
        \log\Big(\frac{2}{\delta}\Big)}\nonumber\\
        &&+2K\kappa_{m+1}(\delta,\rho)\nonumber\\
        &&+4\widetilde{C}_1\Big(\widetilde{L}_{3}+\rho^{m+1}
        +e^{2}\widetilde{L}_{2}H\Big(\frac{M_{p}K}{\delta}\Big)^{\frac{m+1}{p}} \Big)\nonumber\\
        &&\quad\times\left(\frac{M_{p}}{\rho^{p}}K
        +\sqrt{2K\log\Big(\frac{1}{\delta}\Big)}\right)\nonumber\\
\end{eqnarray}

where the second inequality is due to \eqref{eq:CLIP BOUND}. 
\end{proof}

\section{Numerical experiments}\label{sec:experiments}

We illustrate the performance of the APL-Diffusion Algorithm with two examples. 

\vspace{-5pt}
\subsection{A one-dimensional example}
We first illustrate the performance using a tractable one-dimensional problem. Let us take the state space as $\mathcal{S}=\mathbb{R}$ and the action space as $[0,10]$. 

 \paragraph{Set-up} The experiment set-up is specified as follows. 
\begin{itemize}
    \item  Dynamics and reward: for $h\in [H-1]$, $\mu_h(x,a)=0.05-0.1x+0.01a$, $\sigma_h(x,a)=0.1$, $X_{1}=4$, $
R_h(x,a)\sim \NN((x-a)^2,0.01)$. 

 \item Model parameters: $H=10$, $K=2000$, $\rho=10$, $\forall h\in [H], \widetilde{C}_h=5, D=10\sqrt{2}, \Delta=1$.

 \item Initialization: For any
 $h\in [H], k\in [K]$, and $B \in\mathcal{P}_h^0$, we set
\begin{eqnarray}
    &&\mathcal{P}_h^0=\{[0,10]\times [0,10],[10,0]\times [0,10]\}, \,\,\overline{Q}_h^0(\cdot)= 1837.1, \,\,\overline{Q}_h^k(\bar{Z}^{\complement})=-505, \nonumber\\
    && \widetilde{V}_h^0(S)=1837.1,\,\, S=\Gamma_{\mathcal{S}}(B), \,\,\overline{V}_h^0(x) =5+5\|x\|^2 \mbox{ for } x\in \mathbb{R}^{d_{\mathcal{S}}}\nonumber.
\end{eqnarray}
\end{itemize}

 \paragraph{Adaptive partition and convergence} In Figure \ref{fig:Adaptive Discretization}, our algorithm adaptively refines the partition granularity in regions where the underlying $Q_h^{*}$ values are high (with high confidence). Notably, the ground truth optimal action $a^*$ which is equal to 10 with high probability, unknown to the algorithm, falls within these finely partitioned regions, highlighting the algorithm’s effectiveness and superior performance in efficient discretization. In addition, Figure \ref{fig:Reward and Log-Log}-(a) shows that the estimated $V^{\widetilde{\pi}}$ rapidly converge to the optimal level, indicating a fast convergence rate of the algorithm.

 \paragraph{Regret order} In Figure \ref{fig:Reward and Log-Log}-(b),we present the log-log plot of cumulative regret versus episode index, focusing on the regime where performance has stabilized. By fitting a linear regression model to the data, we estimate the regret order based on the slope of the fitted linear line. The estimated slope is 0.69 which is smaller than the worst case regret order of value $\frac{1+d_{\mathcal{S}}+d_{\mathcal{A}}}{2+d_{\mathcal{S}}+d_{\mathcal{A}}}=\frac{3}{4}$.

\begin{figure}[H]
        \centering
        \includegraphics[height=2.1in]{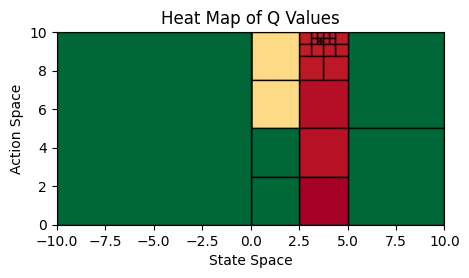}
        \caption{Demonstration of the adaptive partition from the APL-Diffusion algorithm for $\mathcal{P}_9^{2000}$.}
    \label{fig:Adaptive Discretization}
\end{figure}

\begin{figure}[H]
    \centering
    \begin{subfigure}[t]{0.48\textwidth}
        \centering
        \includegraphics[width=\linewidth,height=2.1in,keepaspectratio]{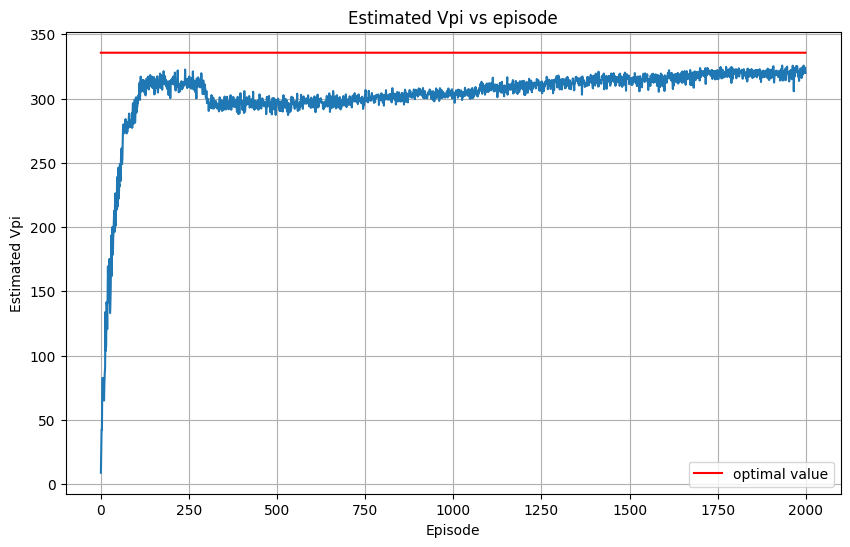}
        \caption{Estimated $V^{\widetilde{\pi}}$ (per episode) throughout training.}
    \end{subfigure}\hfill
    \begin{subfigure}[t]{0.48\textwidth}
        \centering
        \includegraphics[width=\linewidth,height=2.0in,keepaspectratio]{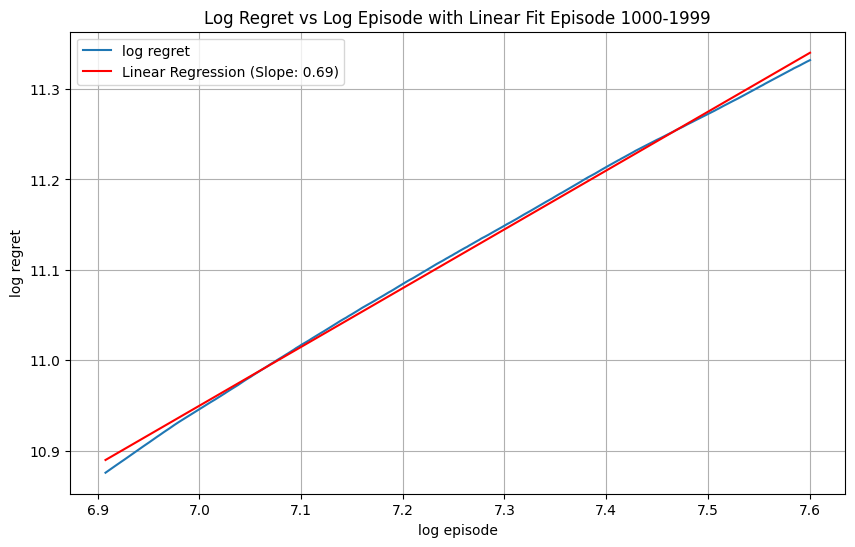}
        \caption{Estimating regret order via linear regression: log(cumulative regret) with respect to log(episode).}
    \end{subfigure}
    \caption{Algorithm performance.}
    \label{fig:Reward and Log-Log}
\end{figure}

\subsection{Mean-variance portfolio optimization}
We next evaluate the performance of the APL-Diffusion Algorithm in the context of mean-variance portfolio optimization with multiple assets. In this setting, the agent learns to determine the optimal allocation of wealth across a basket of securities, balancing expected return against portfolio variance.

We consider a market with $n$ assets. One of the assets is a risk-free asset with interest rate $r_0>0$. For $h\in [H-1]$, the price follows:
\begin{eqnarray*}
Y_{h+1}-Y_{h}=r_{0} Y_h\Delta,
\end{eqnarray*}
with initial condition $Y_{1}=y>0.$

The other five assets are stocks whose price processes follow, for $h\in [H-1]$,
\begin{eqnarray*}Z^i_{h+1}-Z^i_{h}=b^i Z_h^i\Delta+\sigma^i Z_h^iB_{h}^i\sqrt{\Delta},
\end{eqnarray*}
with initial condition $Z^i_1=z^i>0$. Here, $b^i>r_0$ is the appreciation rate  and $\sigma^i>0$ is the volatility of the stock $i$ ($i=1,\cdots,n-1$). 

Consider an investor who invests $a_h^i$ proportion of the wealth to stock $Z^i$ at time $h$, with the remaining proportion $1-\sum_{i=1}^{n-1}a_h^i$ to the risk-free asset, 
then the wealth process follows, for $h\in [H-1]$,
\begin{eqnarray*}
 X_{h+1}-X_{h}=\left(r_0X_h+\sum_{i=1}^{n-1}(b^i-r_0)X_ha_h^i\right) \Delta+\sum_{i=1}^{n-1}\sigma^iX_ha^i_hB_h^i\sqrt{\Delta}, 
\end{eqnarray*}
with initial condition $X_1=x_1>0.$
Here we restrict that $0\leq a_h^i\leq 1, \sum_{i=1}^{n-1}a_h^i\leq 1$.

The reward function is set as
\begin{eqnarray*}
    R_h(x,a)=\delta_0, \mbox{ for } h\in[H-1], \mbox{ and } R_{H}(x,a)=\delta_{(\nu-x)x}.
\end{eqnarray*} 
\begin{remark}
 It is worth emphasizing that in this experiment setting, the volatility can become arbitrarily small, and the drift and volatility coefficients may fail to be Lipschitz continuous  with respect to the action variable. These conditions fall outside the scope of Assumptions \ref{ass:lipschitz}, which are required for our theoretical regret guarantees. However, empirical results demonstrate that the APL-Diffusion Algorithm maintains strong performance despite the violation of these assumptions. This suggests that the algorithm exhibits robustness and practical effectiveness beyond the confines of the theoretical framework.
\end{remark}

\paragraph{Set-up} We specify the parameters governing the system dynamics and reward function, along with other model configurations and initialization settings, as follows.
\begin{itemize}
\item We take $n=6$ in this example with $5$ risky assets and $1$ risk-free asset.
    \item     Dynamics and reward: 
    $r_0=0.05, b^i=0.15, \sigma^i =0.2, \nu=10, X_1=2, (i=1,\cdots,5).$
   \item Model parameters: $H=30$, $K=2000$, $\rho=10$, $\forall h\in [H], \widetilde{C}_h=1$,  $\Delta=\frac{1}{52}$. 
\item  Initialization:\footnote{ Note that in this application, the action domain is not a hypercube as assumed in the earlier section. Consequently, both the initialization and block-splitting procedures are modified accordingly. We define the initial partition as $\mathcal{P}_{h}^0=\{[0,\rho]\times \mathcal{A}, [-\rho,0]\times \mathcal{A}\}$ and initialize the estimators as  $\overline{Q}_h^0(\cdot)= \widetilde{C}_h(1+\rho^{m+1}) $ and $\widetilde{V}_h^0(.)=\widetilde{C}_h(1+\rho^{m+1})$. For $\overline{Q}_h^k(\bar{Z}^{\complement})$ and $\overline{V}_h^0$, we adopt the same values as in \eqref{eq:initial value for estimation}. When splitting a block, we divide the corresponding one-dimensional state space into two halves, and partition the five-dimensional isosceles right simplex action space into thirty-two isosceles right simplex of equal size.}
$\forall h\in [H], \forall k\in [K], B \in\mathcal{P}_h^0$,

\begin{eqnarray*}
      \mathcal{P}_{h}^0&=&\{[0,10]\times \mathcal{A}, [-10,0]\times \mathcal{A}\},\\
      &&\mbox{where }\mathcal{A}=\left\{a:a_i\geq 0,
      \sum_{i=1}^{5}a_i\leq 1, i=1,2,3,4,5\right\},\\
     \overline{Q}_h^0(\cdot)&=& 101,\qquad
     \overline{Q}_h^k(\bar{Z}^{\complement})=-101,\\
      \widetilde{V}_h^0(S)&=&101,\qquad S=\Gamma_{\mathcal{S}}(B),\\
      \overline{V}_h^0(x)&=&\|x\|^{2}+101
      \mbox{ for } x\in \mathbb{R}^{d_{\mathcal{S}}}.
\end{eqnarray*}

\end{itemize}

\begin{figure}[H]
    \centering
    \begin{subfigure}[t]{0.48\textwidth}
        \centering
        \includegraphics[width=\linewidth,height=2.0in,keepaspectratio]{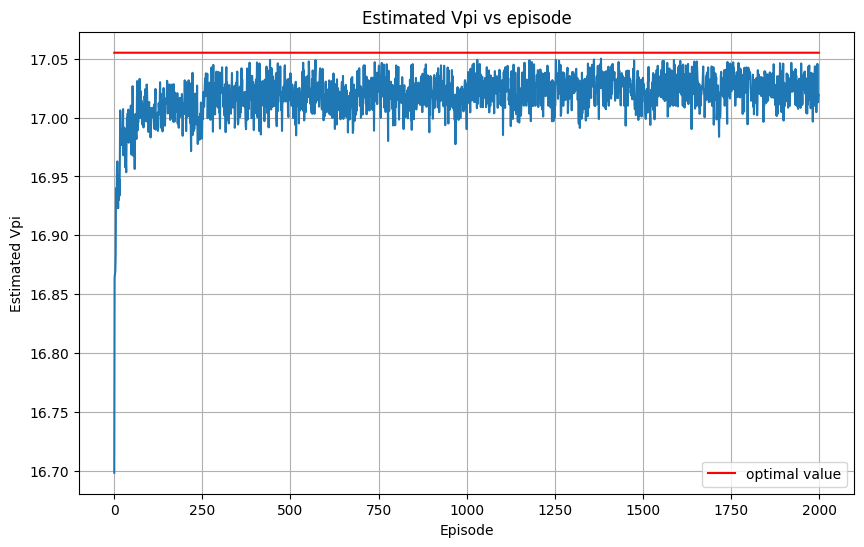}
        \caption{Estimated $V^{\widetilde{\pi}}$ (per episode) throughout training.}
    \end{subfigure}\hfill
    \begin{subfigure}[t]{0.48\textwidth}
        \centering
        \includegraphics[width=\linewidth,height=2.0in,keepaspectratio]{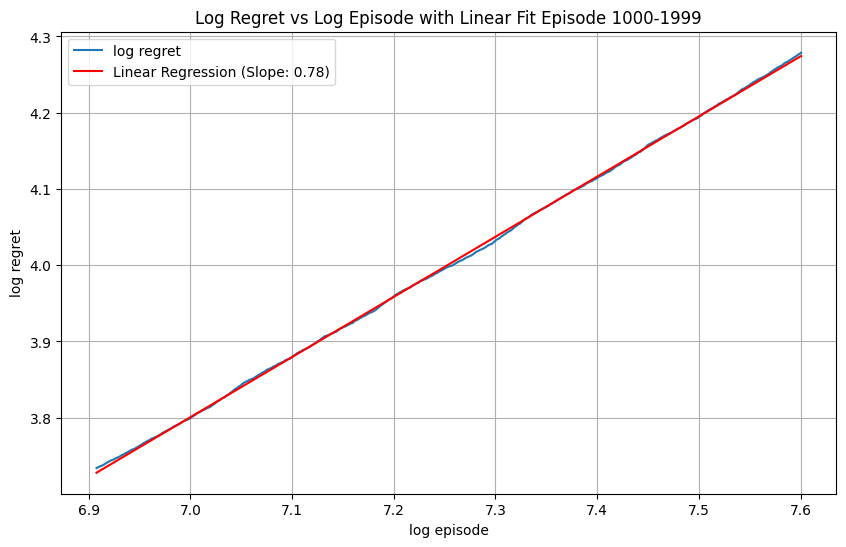}
        \caption{Estimating regret order via linear regression: log(cumulative regret) with respect to log(episode).}
    \end{subfigure}
    \caption{Algorithm performance.}
    \label{fig:1-Reward and Log-Log}
\end{figure}

\paragraph{Reward convergence and regret order} In Figure \ref{fig:1-Reward and Log-Log}-(a), we see the rapid convergence of estimated $V^{\widetilde{\pi}}$ towards the optimal value. In Figure \ref{fig:1-Reward and Log-Log}-(b), we present the log-log plot of cumulative regret versus episode index, focusing on the regime where performance has stabilized. By fitting a linear regression model to the data, we estimate the regret order based on the slope of the fitted linear line. The estimated slope is 0.78, which is lower than the worst-case theoretical regret bound with value $\frac{1 + d_{\mathcal{S}} + d_{\mathcal{A}}}{2 + d_{\mathcal{S}} + d_{\mathcal{A}}} = \frac{7}{8}$. 
This indicates an improved empirical performance relative to the worst-case scenario guarantee.

\ifnum\supplementonlyflag=1\relax
\begingroup
\setlength{\columnsep}{14pt}
\nocite{\originalreferencekeys}
\bibliographystyle{siamplain}
\bibliography{references}
\endgroup
\fi

\end{document}